\documentclass[letterpaper, 10 pt]{article}
\usepackage{style}
\usepackage{cancel}
\usepackage{blindtext}
\newcommand{\E}{\mathbb{E}}

\usepackage{amsmath,amssymb}
\usepackage{hyperref}
\usepackage{amsthm}
\usepackage{arydshln}
\usepackage[ruled,algo2e,noline]{algorithm2e}
\usepackage{xcolor}
\usepackage{threeparttable}
\usepackage{tikz}
\usepackage{pgfplots}
\pgfplotsset{compat=1.15}
\usetikzlibrary{shapes.arrows}
\usetikzlibrary{shapes}
\usetikzlibrary{decorations}
\usetikzlibrary{backgrounds}
\usetikzlibrary{calc}
\usetikzlibrary{arrows}
\usetikzlibrary{automata,positioning}
\usetikzlibrary{decorations.pathreplacing,calligraphy}

\usepackage{adjustbox}
\usepackage{bm}
\usepackage{gincltex}
\usepackage{caption}
\usepackage{subcaption}
\usepgfplotslibrary{fillbetween}
\usepgfplotslibrary{groupplots}
\usetikzlibrary{plotmarks}
\usetikzlibrary{patterns}
\pgfplotsset{
tick label style={font=\footnotesize},
label style={font=\footnotesize},
legend style={font=\footnotesize},
}
\newcommand{\Mod}[1]{\ \mathrm{mod}\ #1}

\theoremstyle{plain}
\newtheorem{altassumption}{Assumption}[section]



\allowdisplaybreaks

\usepackage{multirow}
\RequirePackage{tikz}
\usepackage{pifont}

\title{\LARGE \bf 
Stochastic Approximation with Unbounded Markovian Noise: A General-Purpose Theorem
}

\author{
{\normalsize Shaan Ul Haque}\footnote{H. Milton Stewart School of Industrial \& Systems Engineering, Georgia Institute of Technology, Atlanta, GA, 30332, USA, {\tt\small \{\href{mailto:shaque49@gatech.edu}{shaque49}, \href{mailto:siva.theja@gatech.edu}{siva.theja}\}@gatech.edu}} $\,$\and
{\normalsize Siva Theja Maguluri$^*$}
}
\date{}

\begin{document}

\maketitle

\setlength{\abovedisplayskip}{5pt}
\setlength{\belowdisplayskip}{5pt}

\begin{abstract}
Motivated by engineering applications such as resource allocation in networks and inventory systems, we consider average-reward Reinforcement Learning with unbounded state space and reward function. Recent work \cite{murthy2024} studied this problem in the actor-critic framework and established finite sample bounds assuming access to a critic with certain error guarantees. We complement their work by studying Temporal Difference (TD) learning with linear function approximation and establishing finite-time bounds with the optimal sample complexity. These results are obtained using the following general-purpose theorem for non-linear Stochastic Approximation (SA).

Suppose that one constructs a Lyapunov function for a non-linear SA with certain drift condition. Then, our theorem establishes finite-time bounds when this SA is driven by unbounded Markovian noise under suitable conditions. It serves as a black box tool to generalize sample guarantees on SA from i.i.d. or martingale difference case to potentially unbounded Markovian noise. The generality and the mild assumptions of the setup enables broad applicability of our theorem. We illustrate its power by studying three more systems: (i) We analyze the least mean squares algorithm to solve regression problem for generalized linear models with unbounded and uncountable Markovian noise. (ii) We improve upon the finite-time bounds of Q-learning in \cite{chen2023} by tightening the error bounds and also allowing for a larger class of behavior policies. (iii) We establish the first ever finite-time bounds for distributed stochastic optimization of high-dimensional smooth strongly convex function using cyclic block coordinate descent.
\end{abstract}

\section{Introduction}
Reinforcement Learning (RL) is an important paradigm in machine learning that provides a powerful framework for learning optimal decision-making strategies in uncertain environments \cite{sutton2018, szepesvari2022algorithms}. Since its inception, it has been employed in a variety of practical problems such as health care \cite{dann2019}, robotics \cite{kober2013reinforcement}, autonomous vehicles \cite{aradi2020survey}, and stochastic networks \cite{bai2019}. This remarkable success has led to an extensive study of its convergence behavior both asymptotically \cite{bertsekas1996, tsitsiklis1994, sutton1988learning} and in finite-time \cite{beck2012, bhandari2018, srikant2019, qu2020, chandak2022concentration, chen2023, zhang2024constant}.

The underlying problem structure in RL is typically modeled by a Markov Decision Process (MDP) \cite{puterman2014markov} whose transition dynamics are unknown. 
Several real-world problems such as inventory management systems or queueing models of resource allocation in stochastic networks involve infinite state spaces, and moreover 
rewards or costs usually go to infinity with the state. %
Despite these challenges, RL algorithms have shown promising empirical results in these extreme settings \cite{bai2019, cuartas2023hybrid, wei2024sample, bharti2020reinforcement}. In contrast, there is little analytical understanding of their performance in the unbounded setting. In particular, their finite time/sample performance is not well understood. 
Most of the literature focusing on the finite-time analysis of RL algorithms either assumes finite state space for the underlying MDP \cite{chen2023, sajad_ac, qiu_ac, chen_ac} or bounded rewards \cite{wu2020finite, yang_ac, wang2017finite}. Furthermore, these assumptions are crucial to their analysis, and thus, their results cannot be easily extended.

One of the widely adopted approaches to find the optimal policy is the actor-critic (AC) framework \cite{barto1983neuronlike}. In this method, the actor improves the current policy by updating it in a direction that maximizes the expected long-term rewards, while the critic evaluates the performance of the policy based on the data samples from the MDP. A recent prior work that analytically studied infinite state MDPs in this context is \cite{murthy2024}, where the authors focus on the actor phase and established finite-time convergence bounds of policy optimization algorithms assuming that the critic evaluates a given policy with certain error guarantees. In this paper, we complement their work by providing finite sample guarantees of such a critic. In particular, we
analyze Temporal Difference (TD) learning, a popular algorithm for policy evaluation in critic, and establish finite-time bounds on the mean square error. 

The main contributions of the paper are as follows.

\noindent\textbf{Finite-Time Convergence Guarantees for SA with Unbounded Markov noise:} TD learning is based on using SA to solve the underlying Bellman equation of the MDP. The aforementioned results on TD learning is obtained by studying a general class of non-linear SA corrupted by unbounded Markovian noise, and establishing the following general-purpose result. 
\begin{itheorem}
    Consider a nonlinear SA, and suppose that a Lyapunov function satisfying certain drift condition is constructed in the setting when the noise is i.i.d. or martingale difference. Then, we establish finite sample bounds when the same SA is driven by Markovian noise with unbounded state space under appropriate assumptions. 
\end{itheorem}
In other words, we decouple the challenge of handling Markovian noise from the issue of analyzing the SA itself.
Our result complements the existing literature by enabling one to generalize any SA result to the case of unbounded Markovian noise. Therefore, we believe that this powerful result is of independent interest due to its applicability in a wide variety of settings.

\noindent\textbf{Methodological Contribution:} The key technique that enables us to establish these results is the use of the solution of the Poisson equation to analyze Markov noise. Recent works control Markov noise by exploiting the geometric mixing properties of Markov chains \cite{bhandari2018, srikant2019, qu2020, mou2021, xu2021sample, khodadadian2022, chen2023}. However, it is unclear if this approach enables one to analyze unbounded Markovian noise. We instead adopt the use of Poisson equation, which has been used to study asymptotic convergence and statistics of SA \cite{harold1997, benveniste2012, borkar2024, meyn2024, gast2024}. Although this approach has also been recently used to study linear SA under bounded Markovian noise in \cite{kaledin2020finite, haque2023tight, agrawal2024markov}, we use it to obtain finite sample bounds for nonlinear SA under unbounded noise settings. Compared to the mixing-time approach, this approach is not only more elegant but also has the added advantage of giving tighter bounds (in terms of log factors) and allows for a larger class of Markov chains (such as periodic chains). The next three  contributions focus on exploiting these improvements. 

\noindent \textbf{Performance of TD-Learning in Unbounded State Space and Rewards:} The problem of policy evaluation in MDPs corresponds to the setting of countable state space but unbounded noise, where both the feature vectors and rewards in the algorithm can grow arbitrarily large. We analyze average-reward TD$(\lambda)$ with linear function approximation (LFA) under asynchronous updates, a popular algorithm for policy evaluation in RL. We establish the first known finite-time convergence bounds for this setting, and show an optimal  $\mathcal{O}(1/k)$ convergence rate 
under appropriate choice of step sizes. Due to the challenges in the average-reward setting, to the best of our knowledge, even the asymptotic convergence has not been formally established in the literature. By a careful projection of the iterates to an appropriate subspace, we also establish its almost-sure (a.s.) convergence.

\noindent\textbf{Generalized Linear Model with Markovian Data:} The general setup of our theorem also allows us to consider settings beyond RL. In this instance, we address the fully general case of analyzing unbounded and uncountable Markovian noise in the algorithm. In particular, we consider the regression problem for a generalized linear model (GLM) where the Markovian noise comes from an auto-regressive process driven by an i.i.d. sequence of noise satisfying a finite fourth-moment condition. In this setup, we provide finite-time bounds for the least mean squares algorithm, a well-known method for regression tasks. Two closely related works concerning GLMs are \cite{kotsalis2022simple} and \cite{nagaraj2020least} which restrict the i.i.d. noise in the auto-regressive process to either have bounded support or follow Gaussian distribution, respectively. Notably, our analysis significantly generalizes these settings by only requiring a finite fourth-moment for the noise.

\noindent\textbf{Performance of $Q$-learning Algorithm:} As an illustrative application in the case of finite-state Markovian noise, we consider $Q$-learning in the discounted setting. Using our black box, we immediately obtain finite-sample bounds for $Q$-learning using the Lyapunov function constructed in \cite{chen2023}. Our result improves \cite{chen2023} by (i) shaving off additional $\mathcal{O}(\log(k))$, and $\mathcal{O}\left(\log\left(1/1-\gamma\right)\right)$ factors in the convergence bounds and (ii) allowing for a larger class of behavior policies, including those that may not have geometric mixing or lead to periodic behavior.
        
\noindent\textbf{Performance of Stochastic Cyclic Block Coordinate Descent:} In this final application, we study the stochastic optimization of a high-dimensional smooth strongly convex function, where one is only allowed to update a subset of components at each time. This is commonly done using a variant of stochastic gradient descent called cyclic block coordinate descent (CBCD). While other versions of block coordinate descent were studied in the literature \cite{nesterov2012, diakonikolas2018,lan2020first}, finite sample bounds of CBCD in the stochastic setting were not known. We provide a new perspective to handle the cyclic nature of updates by viewing each block as the states of a periodic Markov chain. This outlook in conjunction with our black box immediately gives optimal $\mathcal{O}\left(1/k\right)$ convergence rate.

\section{Related Work}

\noindent\textbf{Asymptotic Analysis of Stochastic Approximation:} SA was first proposed by \cite{robbins1951} as a family of iterative algorithms to find the roots of an operator and has been extensively studied since then.
Asymptotic convergence of SA was studied in 
\cite{borkar2008, benveniste2012, Kushner1997StochasticAA}. 
More recent work including \cite{borkar2024, meyn2024, gast2024}  studies SA with unbounded Markovian noise using the Poisson's equation, as we do in this paper. 
However, their focus is on establishing a central limit theorem, i.e., an asymptotic result of the form $(x_k-x^*)/\sqrt{\alpha_k}\stackrel{d}{\to} \mathcal{N}(0, \Sigma)$, for appropriate choice of $\Sigma$.
In contrast, the focus of our work is on establishing a finite-time bound. 
Another line of work,  
inspired by off-policy algorithms in RL studies the asymptotic convergence of SA under more general setup where the solution to Poisson's equation may not exist. This approach, as seen in works by \cite{yu2012, yu2017, yu2018, liu2025} only assumes the ergodic theorem for Markov chains, which is arguably the most general framework for controlling the noise in the algorithm \cite{Kushner1997StochasticAA}. 
Note that even when one has a finite-state MDP, 
off-policy RL algorithms such as Least Squares TD, Emphatic TD, and Gradient TD$(\lambda)$ lead to SA with unbounded Markovian noise due to the product of importance sampling ratios. In some cases, such noise can even have infinite variance \cite{glynn1989importance}. 
While we make a more restrictive assumption on the existence and moments of the solution of the Poisson's equation, we obtain finite-time mean-square bounds. 

\noindent\textbf{Finite Time Bounds for Stochastic Approximation:} 
Finite time analysis has gained significant attraction in recent works such as \cite{chen2020finite, srikant2019, chen2023, mou2021}. In particular, these works demonstrate that finite-time bounds on general SA algorithms immediately imply performance guarantees of a large class of RL algorithms including V-Trace, $Q$-learning, n-step TD, etc. We contribute to this line of work by providing a general-purpose theorem for SA with unbounded Markovian noise which furnishes finite-sample bounds on various RL algorithms for infinite state MDPs, and we illustrate such use in the context of TD learning.

\noindent\textbf{Reinforcement Learning:} RL has been extensively studied in the literature, starting with 
asymptotic convergence which was established in \cite{tsitsiklis1994, tsitsiklis1997, bertsekas1996}. Off late, there has been a growing interest in obtaining finite sample complexity of these algorithms as established in \cite{beck2012, bhandari2018, srikant2019, qu2020, li2020, pananjady2020instance, chen2023, zhang2024constant}. AC algorithms were initially proposed and studied in \cite{barto1983neuronlike, konda1999actor} with their finite time performance analyzed in \cite{kumar2023samplecomplexityactorcriticmethod, qiu_ac, wang2019neuralpolicygradientmethods, chen_ac}. However, these finite time studies are primarily focused on finite state space settings. 
Some notable exceptions include recent works such as \cite{shah2020stable, murthy2024}. In \cite{shah2020stable}, the authors focus on designing RL policies that ensure stable behavior in queueing systems without emphasizing optimality. On the other hand, \cite{murthy2024} establishes finite time bounds for policy optimization using natural policy gradient in infinite state settings under an oracular critic having guaranteed error margins. We focus on constructing such a critic based on TD learning and characterizing its performance.

\noindent\textbf{TD Learning:} TD Learning is one of the most common algorithms for the critic phase, i.e., policy evaluation, and has been extensively studied both in discounted and average reward settings. The asymptotic behavior of TD learning in these regimes was characterized in \cite{tsitsiklis1997, tsitsiklis1999}.
Finite-sample complexity of TD has been established in \cite{bhandari2018, srikant2019, pananjady2020instance} in the discounted reward setting and in \cite{zhang2021} in the average-reward setting. However, most of the prior work on finite-sample guarantees considers only finite state MDPs. Motivated by applications in engineering systems, we study infinite state MDPs in the average reward setting, and establish finite sample guarantees.

\noindent\textbf{Generalized Linear Model:} GLMs have been extensively studied in the literature, particularly due to their relevance in signal estimation problem \cite{kotsalis2022simple}. Several works have considered their i.i.d. variants in various specialized cases \cite{juditsky2020statistical, nagaraj2020least, kotsalis2023simple}. Moreover, a closely related problem of filter design in signal processing$-$whose asymptotics were studied in \cite{benveniste2012, Kushner1997StochasticAA}$-$can also be modeled as a special case of GLMs. While existing techniques for finite-time bounds typically assume either i.i.d. or bounded Markovian noise, we extend the theory to unbounded Markovian noise with uncountable state space.

\noindent\textbf{Block Coordinate Descent:} Block coordinate descent (BCD) methods have been widely explored due to their effectiveness in large-scale distributed optimization \cite{richtarik2015, richtarik2016} for machine learning, such as in  L1-regularized least squares (LASSO) \cite{fu1998, sardy2000} and support vector machines (SVMs) \cite{joachims1998, chang2011, chou2020dual}. While a substantial number of studies have investigated the Randomized and Greedy variants of BCD \cite{nesterov2012, nutini2015, nesterov2017, diakonikolas2018, lan2020first}, the literature on CBCD is not as rich. Some of the works that have analyzed it in deterministic settings include \cite{beck2013, gurbuzbalaban2017cyclic, Diakonikolas2023} but to the best of our knowledge, no prior work has explored the stochastic version of CBCD.

\noindent\textbf{Poisson Equation for Markov Chains:} 
Recent works on finite sample bounds of SA such as \cite{bhandari2018, srikant2019, mou2021, qu2020, xu2021sample,chen2023}  have exploited geometric mixing of the underlying Markov chain. It is unclear if this approach generalizes to the case of unbounded setting. In this paper, we adopt the use of Poisson equation to analyze Markov noise which has been extensively used for this purpose in classical work on asymptotic convergence of SA, such as \cite{benveniste2012, harold1997}, and also in other domains such as queueing theory in \cite{grosof2023resetmarctechniquesapplication, falin1999heavy}. More recently, while this approach has recently been used to study linear SA in \cite{chandak2022concentration, kaledin2020finite, agrawal2024markov, haque2023tight}, their analysis is restricted to finite state space. 

\section{Problem Setting and Main Result}\label{sec:main_thm}
Consider a non-linear operator $\Bar{F}:\mathbb{R}^d\to \mathbb{R}^d$. Our objective is to find the solution $x^*$ to the following equation:
\begin{align}\label{eq:main_eq}
    \bar{F}(x)=\E_{Y\sim \mu}[F(x, Y)]=0,
\end{align}
where $Y$ represents random noise sampled from a Markov chain with a unique stationary distribution $\mu$, and $F$ is a general non-linear operator. The state space of the Markov chain is denoted by $\mathcal{Y}$.

Suppose $\bar{F}(\cdot)$ is known, then Eq. \eqref{eq:main_eq} can be solved using the simple fixed-point iteration $x_{k+1}=\bar{F}(x_k)$. The convergence of this iteration is guaranteed if one can construct a potential function—also known as Lyapunov function in stochastic approximation theory— that strictly decreases over time. However, when the distribution $\mu$ is unknown, and thus $\bar{F}(x)$ is unknown, we consider solving Eq. \eqref{eq:main_eq} using the stochastic approximation iteration proposed as follows.

Let $\{Y_k\}_{k\geq 0}$ be a Markov process with stationary distribution $\mu$. Then, the algorithm iteratively updates the estimate $x_k$ by:
\begin{align}\label{eq:main_rec}
    x_{k+1}=\Pi_\mathcal{X}\left(x_k+\alpha_k(F(x_k, Y_k)+M_k)\right),
\end{align}
where $\{\alpha_k\}_{k\geq 0}$ is a sequence of step-sizes, $\{M_k\}_{k\geq 0}$ is a random process representing the additive external noise, and $\Pi_{\mathcal{X}}(\cdot)$ is $\ell_2$-norm projection of the iterates to set $\mathcal{X}$. The projection on the set $\mathcal{X}$ is included for generality, where $\mathcal{X}$ can be either a compact set or the entire space $\mathbb{R}^d$, depending on the context. We emphasize the importance of projection operator $\Pi_{\mathcal{X}}$ to get meaningful mean square bounds here. In a recent study \cite{borkar2024}, the authors constructed an SA with unbounded noise that operates without any projection step. It is shown that such an algorithm will converge to the stationary point a.s., however, the mean square error diverges (Proposition 4, Section 3.3, \cite{borkar2024}). Thus, projecting the iterates to a bounded set is not a proof artifact, but rather a technical necessity. We present this example in Section \ref{sec:counterexample} and also provide a new proof for the divergence of the mean square error.

We begin by outlining the set of assumptions for Algorithm \ref{eq:main_rec}. These assumptions are motivated by practical applications of SA algorithms, such as those in RL and optimization algorithms, which will be studied in Sections \ref{sec:RL} and \ref{sec:opt}. Let $\|\cdot\|_c$ be an arbitrary norm in $\mathbb{R}^d$.

\begin{assumption}\label{assump:linear_F}
    There exist functions $A_1(y), B_1(y):\mathcal{Y}\to [0,\infty)$ such that for all $x\in \mathcal{X}$ and $y\in\mathcal{Y}$ the operator $F$ satisfies the following:
    \begin{align*}
        \|F(x, y)\|_c\leq A_1(y)\|x-x^*\|_c+B_1(y).   
    \end{align*}
\end{assumption}
\begin{remark}
    Prior works such as \cite{srikant2019, mou2021, chen2023} assumed that the state space of the Markov chain is bounded, thus could replace the functions $A_1(y)$ and $B_1(y)$ by their upper bounds. However, in contrast, we consider the case of unbounded state space where these functions can possibly be unbounded as well.
\end{remark}

Next, we state the assumptions about the Markov process. Let $P:\mathcal{Y}\times\mathcal{Y}\to [0,1]$ be the transition kernel, and let us denote the one-step expectation of any measurable function $G$ conditioned on $z\in \mathcal{Y}$ as $\E_{z}[G(Y_1)]=\int_{\mathcal{Y}}G(y)P(z, dy)$. Note that if $\mathcal{Y}$ is countable, then $\E_{z}[G(Y_1)]=\sum_{j\in \mathcal{Y}}G(j)P(j|z)$.
\begin{assumption}\label{assump:Poisson_eq}
 We assume the following properties on the Markov process:
 \vspace{-1mm}
    \begin{enumerate}[(a)]
        \item The Markov process has a unique stationary distribution denoted by $\mu$. Moreover, $\E_{Y\sim\mu}[F(x,Y)]$ exists for all $x\in \mathbb{R}^d$ and is denoted by $\bar{F}(x)$.
        \item There exists a function $V_x(z)$ for all $z\in\mathcal{Y}$ and $x\in \mathbb{R}^d$ which satisfies the Poisson equation:
        \begin{align}\label{eq:Poisson_eq}
            V_x(z)=F(x, z)+\E_{z}[V_x(Y_1)]-\bar{F}(x).
        \end{align}
        \item There exist functions $A_2(y), B_2 (y):\mathcal{Y}\to [0,\infty)$ such that for all $x_1,x_2\in \mathcal{X}$ and $y\in\mathcal{Y}$, the solution $V_x$ satisfies the following:
        \begin{equation}\label{assump:value_fn-prop}
            \begin{split}
                \|V_{x_{2}}(y)-V_{x_{1}}(y)\|_c\leq A_2(y)\|x_2-x_1\|_c,~~\|V_{x^*}(y)\|_c\leq B_2(y).
            \end{split}
        \end{equation}
        \item Let $\{Y_k\}_{k\geq 0}$ be a sample path starting with an arbitrary initial state $Y_0=y_0$, then for all $k\geq 0$ we have the following:
        \begin{enumerate}[(1)]\label{assump:moments_bound} 
            \item Starting from any initial state $y_0$, second moment of the functions $A_1(\cdot), A_2(\cdot), B_1(\cdot)$, and $B_2(\cdot)$ are finite and given as follows:
            \begin{align*}
                \max\{\E_{y_0}[A_1^2(Y_k)], 1\}\leq \hat{A}_1^2(y_0),~&\max\{\E_{y_0}[A_2^2(Y_k)], 1\}\leq \hat{A}_2^2(y_0),\\
                \hspace{-1mm}\E_{y_0}[B_1^2(Y_k)]\leq \hat{B}_1^2(y_0),~&\E_{y_0}[B_2^2(Y_k)]\leq \hat{B}_2^2(y_0),
            \end{align*}
            where $\E_{y_0}[\cdot]=\E[\cdot|Y_0=y_0]$.
            \item If the state space is bounded then we denote the upper bound on these functions as follows:
            \begin{align*}
                \max\{|\max_{y\in \mathcal{Y}}A_1(y)|, 1\}=A_1,~&\max\{|\max_{y\in \mathcal{Y}}A_2(y)|, 1\}=A_2,\\
                |\max_{y\in \mathcal{Y}}B_1(y)|=B_1,~&|\max_{y\in \mathcal{Y}}B_2(y)|=B_2.
            \end{align*}
        \end{enumerate}   
    \end{enumerate}
    
\end{assumption}
\begin{remark}\label{rem:Poisson_eq}
    This set of assumptions is inspired by the asymptotic analysis of SA in \cite{benveniste2012}. Implicit in them is the fact that the Markov process exhibits a certain degree of stability. It is important to note that these assumptions are always satisfied for bounded state space Markov chains under fairly general conditions. Additionally, they also hold in many scenarios of practical interest when the state space $\mathcal{Y}$ is unbounded, as will be demonstrated in Section \ref{sec:RL}.
\end{remark}
\begin{remark}
    There is a parallel line of research that studies the a.s. convergence of SA under even more general setting where the Poisson's equation may not have a solution 
    \cite{yu2012, yu2017, yu2018, liu2025}. This line of work 
    relies on the ergodic theorem  for Markov chains to get a handle on the noise, and focuses on asymptotic convergence. 
\end{remark}

Let $\{\mathcal{F}_k\}$ be a set of increasing families of $\sigma$-fields, where $\mathcal{F}_k=\sigma\{x_0, Y_0, M_0, \dots, x_{k-1}, Y_{k-1}, M_{k-1}, Y_k\}$.
\begin{assumption}\label{assump:linear_M}
    Let $A_3, B_3\geq 0$. Then, process $\{M_k\}_{k\geq 0}$ satisfies the following conditions: (a) $\E[M_k|\mathcal{F}_k]=0$ for all $k\geq 0$, (b) $\|M_k\|_c\leq A_3\|x_k-x^*\|_c+B_3$.
\end{assumption}
\begin{remark}\label{rem:linear_M}
    Assumption \ref{assump:linear_M} implies that $\{M_k\}_{k\geq 0}$ forms a martingale difference sequence with respect to the filtration $\mathcal{F}_k$, and its growth is at most linear with respect to the iterate $x_k$.
\end{remark}

Let $\|\cdot\|_s$ be a norm in $\mathbb{R}^d$. To study the convergence behavior of Eq. \ref{eq:main_eq}, we assume the existence of a smooth Lyapunov function with respect to $\|\cdot\|_s$ that has negative drift with respect to the iterates $x_k$. More concretely, the Lyapunov function satisfies the following assumption.
\begin{assumption}\label{assump:Lyapunov_fn}
    Given a Lyapunov function $\Phi(x)$, there exists constants $\eta, L_s, l, u>0$, such that we have  
        \begin{align}
            &\langle \nabla \Phi(x-x^*), \bar{F}(x)\rangle\leq -\eta \Phi(x-x^*),\label{assump:neg_drift}\\
            &\Phi(y)\leq \Phi(x)+\langle\nabla\Phi(x), y-x\rangle+\frac{L_s}{2}\|x-y\|_s^2,\label{assump:smoothness}\\
            &l\Phi(x)\leq \|x\|_c^2\leq u\Phi(x),\label{assump:equivalence_rel}\\
            &\Phi(\Pi_{\mathcal{X}}(x)-x^*)\leq\Phi(x-x^*),\label{assump:non-expansive}
        \end{align}
\end{assumption}
where Eq. \eqref{assump:neg_drift} is the negative drift condition, Eq. \eqref{assump:smoothness} is smoothness with respect to $\|\cdot\|_s$, Eq. \eqref{assump:equivalence_rel} is equivalence relation with the norm $\|\cdot\|_c$, and Eq. \eqref{assump:non-expansive} is nonexpansivity of the Lyapunov function $\Phi(x)$. Note that we allow $\|\cdot\|_c$ and $\|\cdot\|_s$ to be two different norms for generality. Often, one can fine-tune the $s$-norm to get tighter bounds.

\begin{assumption}\label{assump:step-size}
    Finally, we assume that the step-size sequence is of the following form:
    \begin{align*}
        \alpha_k=\frac{\alpha}{(k+K)^\xi},
    \end{align*}
    where $\alpha>0$, $K\geq 2$, and $\xi\in [0,1]$.
\end{assumption}

\subsection{Unbounded State Space}
We will now present finite bounds for the two most popular choices of step size, which are common in practice. Let $\mathcal{X}$ denote an $\ell_2$-ball of sufficiently large radius chosen such that $x^*\in \mathcal{X}$. Then, the resulting mean-squared error is as follows.

\begin{theorem}\label{thm:main_thm;bounded}
   Suppose that we run the Markov chain with initial state $y_0$. When the state space $\mathcal{Y}$ is unbounded and the set $\mathcal{X}$ is an $\ell_2$-ball, then under the Assumptions \ref{assump:linear_F}-\ref{assump:step-size}, $\{x_k\}_{k\geq 0}$ in the iterations \eqref{eq:main_eq} satisfy the following:
   \begin{enumerate}[(a)]\label{eq:main_thm;bounded}
       \item When $\alpha_k\equiv\alpha\leq 1$, then for all $k\geq 0$:
       \begin{align*}
           \E[\|x_{k+1}-x^*\|_c^2]&\leq \varphi_0\exp\left(-\eta\alpha k\right)+3\varphi_1\hat{C}(y_0)\alpha+\frac{6\varphi_1\hat{C}(y_0)\alpha}{\eta}.
       \end{align*}
       \item When $\xi=1$, $\alpha> \frac{1}{\eta}$ and $K\geq \max\{\alpha, 2\}$, then for all $k\geq 0$:
       \begin{align*}
           \E[\|x_{k+1}-x^*\|_c^2]&\leq \varphi_0\left(\frac{K}{k+K}\right)^{\eta\alpha}+\frac{\varphi_1\hat{C}(y_0)\alpha}{k+K}+\frac{4(6+4\eta)\varphi_1\hat{C}(y_0)e\alpha^2}{\left(\eta\alpha-1\right)(k+K)}.
       \end{align*}
   \end{enumerate}
\end{theorem}
The rate of convergence under other choices of step-size and the constants $\{\varphi_i\}_i$ and $\hat{C}(y_0)$ are defined in Appendix \ref{appendix:proof}.
\begin{remark}
    In part (a), the error never converges to 0 due to noise variance, however, the expected error of the iterates converges to a ball around $x^*$ at an exponential rate. In part (b), using a decreasing step size with appropriately chosen $\alpha$, we obtain the $\mathcal{O}\left(1/k\right)$ convergence rate, which leads to the sample complexity of $\mathcal{O}\left(1/\epsilon^2\right)$ to achieve $\E[\|x_{k+1}-x^*\|_c]\leq \epsilon$.
\end{remark}
\begin{remark}
    Note that $\Pi_{\mathcal{X}}(x)=\argmin_{x'\in \mathcal{X}}\|x-x'\|_2$. Thus, from a computational standpoint, this only involves rescaling the iterates, as the projection operator $\Pi_{\mathcal{X}}(x)$ reduces to $\frac{\|x\|_2}{\text{radius}(\mathcal{X})}x$ if $\|x\|_2\geq \text{radius}(\mathcal{X})$ and is $x$ otherwise. 
\end{remark}
 We introduce the projection onto the ball  $\mathcal{X}$ for analytical tractability \footnote{One way to bypass projection is if one can show by other means that the iterates remain bounded, such as in discounted bounded rewards settings (Chapter 1, \cite{bertsekas2011dynamic})}. 
 The interplay of the unbounded state space $\mathcal{Y}$ and the iterate space $\mathbb{R}^d$ makes the analysis significantly challenging. Prior works assume that the set $\mathcal{Y}$ is bounded and thus do not need projection. In contexts like queueing systems, truncating the state space would change the stationary distribution, thereby altering the optimal policy. However, projecting the iterates is a more realistic solution in such cases, as it does not change the solution provided that the projecting set is taken to be large enough. Nonetheless, it is worth noting that even after projection, handling the noise is substantially challenging, and no previous work handles this.

\subsection{Bounded State Space}\label{sec:thm:finite}
Now we state the sample complexity when $\mathcal{Y}$ is bounded and $\mathcal{X}\equiv\mathbb{R}^d$ which implies no projection is required. 
\begin{theorem}\label{thm:main_thm;finite} 
When the state space $\mathcal{Y}$ is bounded and the set $\mathcal{X}\equiv\mathbb{R}^d$, then under the Assumptions \ref{assump:linear_F}-\ref{assump:step-size}, $\{x_k\}_{k\geq 0}$ in the iterations \eqref{eq:main_eq} satisfy the following:
\begin{enumerate}[(a)]\label{eq:main_thm;finite}
       \item When $\alpha_k\equiv\alpha\leq \min\left\{1, \frac{\eta}{A(5+2\eta)\varrho_1}\right\}$, then for all $k\geq 0$:
       \begin{align*}
           \E[\|x_{k+1}-x^*\|_c^2]&\leq \varrho_0\exp\left(\frac{-\eta\alpha k}{2}\right)+18B\varrho_1\alpha+\frac{40B\varrho_1\alpha}{\eta}.
       \end{align*}
       \item When $\xi=1$, $\alpha> \frac{2}{\eta}$ and $K\geq \max\{A\alpha(5\alpha+8)\varrho_1, 2\}$, then for all $k\geq 0$:
       \begin{align*}
           \E[\|x_{k+1}-x^*\|_c^2]&\leq \varrho_0\left(\frac{K}{k+K}\right)^{\frac{\eta\alpha}{2}}+\frac{2B\varrho_1\alpha}{k+K}+\frac{8B(5+4\eta)\varrho_1e\alpha^2}{\left(\frac{\eta\alpha}{2}-1\right)\left(k+K\right)}.
       \end{align*}
   \end{enumerate}
\end{theorem}
Please refer to Appendix \ref{appendix:proof} for error bounds under other choices of step-sizes and the constants $\{\varrho_i\}_i$, $A$, and $B$.
\begin{remark}
    Compared to the existing literature with finite state Markov chains, such as \cite{chen2023}, we do not need any mixing property of the Markov chain to establish the bounds. We instead assume that the solution to the Poisson equation exists, which is true in finite state Markov chains even when there is no mixing (such as under periodic behavior). This also has an additional benefit of eliminating the poly-logarithmic factors from the bounds. 
\end{remark}
\vspace{-1mm}

\subsection{SA with almost sure convergence but diverging mean square error}\label{sec:counterexample}
In this section, we argue for the necessity of the projection operator $\Pi_{\mathcal{X}}$ for the unbounded state space setting. Specifically, we will construct a counterexample where, without the projection, the iterations converge to $x^*$ almost surely but the mean square error diverges. We again highlight that although the example is taken from \cite{borkar2024}, we use a different proof technique to show divergence. In particular, instead of using a large deviation principle, we present a proof that is based on elementary arguments. 

Consider a uniformized DTMC $\{Q_k\}_{k\geq 0}$ constructed from an $M/M/1$ queue. The DTMC is defined on the state space $\mathbb{\mathcal{Y}}=\{0, 1, 2,\dots\}$. Then, the queue length $Q_k$ evolves according to the following relation:
\begin{align}\label{eq:queue_length}
    Q_{k+1}=\max(0, Q_k+D_{k+1}).
\end{align}
where $\{D_k\}_{k\geq 1}$ is a sequence of i.i.d. random variables satisfying $\mathbb{P}(D_k=1)=\lambda$ and $\mathbb{P}(D_k=-1)=1-\lambda=\nu$. For simplicity, we will assume that $Q_0=0$, i.e., the queue was started with no jobs in the system. It is well known that if $\lambda<\nu$ then DTMC $\{Q_k\}_{k\geq 0}$ is stable and admits a unique stationary distribution which we denote by $\mu$. Let $\E_{Q_k\sim \mu}[Q_k]=\bar{q}<\infty$ and define the following one-dimensional linear stochastic approximation:
\begin{align}\label{eq:counterex_iteration}
    x_{k+1}=x_k+\alpha_k\left(\left(Q_{k+1}-\bar{q}-1\right)x_k+W_k\right)
\end{align}
where $W_k\sim \mathcal{N}(0,1)$ is a i.i.d. sequence of random variables that are also independent of the $\{Q_k\}_{k\geq 0}$ process and the step-size $\alpha_k = \alpha/(k+1)^\xi$ with $\alpha\geq 1$ and $\xi\in [0, 1]$. 

Note that the Markovian noise in the algorithm is given by $Y_k\equiv Q_{k+1}$. Thus, the noisy operator $F(x_k, Q_{k+1}) = \left(Q_{k+1}-\bar{q}-1\right)x_k$ and the stationary expectation is $\bar{F}(x)=-x$. This implies that the root of equation $\bar{F}(x)=0$ is simply $x^*=0$. Consider $\Phi(x)=x^2$ as the Lyapunov function. Then, it is straightforward to verify that all Assumptions \ref{assump:linear_F}-\ref{assump:step-size} are satisfied for the SA in Eq. \eqref{eq:counterex_iteration}. We now present the following theorem to formalize the properties of the counterexample \eqref{eq:counterex_iteration}. The proof is provided in Appendix \ref{appendix:counterexample}.

\begin{theorem}\label{thm:div_ex}
    The SA recursion defined in Eq. \eqref{eq:counterex_iteration} satisfies the following:
    \begin{enumerate}[(a)]
        \item For any $\lambda$ such that $\lambda<0.5$, we have $\lim_{k\to \infty} x_k = 0$ a.s.
        \item If $\lambda\in (\exp(-\log(4)\log(3.4)), 0.5)$, then $\lim_{k\to \infty}\E[x_k^2]=\infty$.
    \end{enumerate}
\end{theorem}
\begin{remark}
    The two results presented above seem surprising since the mean square error diverges even when  the iterates converge almost surely. 
    However, such seemingly contradictory behavior of mean-square and almost sure convergence is well-known in the literature. A classical example related to the gambling and St. Petersburg paradox involves a product of random variables that exhibits similar behavior, and we present it in Appendix \ref{appendix:gamblers_ruin}.
\end{remark}

A natural question for future research is to obtain finite-time bounds that reconcile both the results in the above theorem. This can be done by establishing a concentration result of the following form. Suppose that we run the SA without projection, then for any $\delta>0$, the error satisfies a bound of the following form 
\begin{align*}
    P(\|x_k-x^*\|_c< f(k, \delta))\geq 1-\delta.
\end{align*}
To reconcile both the results in Theorem \ref{thm:div_ex}, the function $f(k,\delta)$ has the following property. There exists a threshold $\delta_0(k)$ such that,
\begin{enumerate}[(i)]
    \item $\delta_0(k)$ is summable, i.e., $\sum_{k=0}^\infty \delta_0(k)<\infty$.
    \item For any sequence $\delta_1(k)$ that is $\omega(\delta_0(k))$, $f(\cdot, \cdot)$ satisfies $\lim_{k\to \infty} f(k, \delta_1(k))= 0$, and
    \item For any sequence $\delta_2(k)$ that is $o(\delta_0(k))$, $f(\cdot, \cdot)$ satisfies $\lim_{k\to \infty} f(k, \delta_2(k))= \infty$.
\end{enumerate} 
In other words, we get a concentration result, where if our high probability guarantee is smaller than  $(1-\delta_0(k) )$, then the bound on the error $f(k, \delta)$ goes to zero. But if we are  more stringent and seek a bound with a larger probability, then the error diverges. Nevertheless, the summability of $\delta_0(k)$ ensures the a.s. convergence of $x_k$ to $x^*$ owing to Borel-Cantelli Lemma. We provide explicit expressions for $\delta_0(k)$ and $f(k, \delta)$ in the context of St. Petersburg paradox in Appendix \ref{appendix:gamblers_ruin}.

As guaranteed by Theorem \ref{thm:main_thm;finite}, if the Markovian noise is sampled from a finite or bounded state space Markov chain then divergence related issues do not exist. However, more caution is required to avoid divergence of iterates when the noise is unbounded. 

To elaborate more on this, let us consider a simple case where the multiplicative noise $\{Q_k\}_{k\geq 0}$ in Eq. \eqref{eq:counterex_iteration} is i.i.d. in nature. In this setting, independence of $\{Q_k\}_{k\geq 0}$ and $x_{k}$ helps us to establish the following one step recursion:
\begin{align*}
    \E[x_{k+1}^2] &= \E[x_k^2(1+\alpha_k(Q_{k+1}-\bar{q}-1))^2]+\alpha_k^2\E[W_k^2]\\
    &= (1-\alpha_k)^2\E[x_k^2]+\alpha_k^2\\
    &\leq (1-\alpha_k)\E[x_k^2]+\alpha_k^2
\end{align*}
which converges to $0$ as $k\to \infty$ (cf. Corollary 2.1.1 and Corollary 2.1.2 in \cite{chen2020finite}). Clearly, the above steps can also be combined with the Lyapunov function technique to non-linear multi-dimensional SA. Thus, the Markovian structure of the noise plays a crucial role in determining the boundedness or potential divergence of the SA iterates. In Appendix \ref{sec:simpler_method}, we discuss the key role set $\mathcal{S}_1$ plays in proving divergence by comparing with another possible method.

\begin{figure}
    \centering
        \begin{tikzpicture}
    \def\kone{4}
    \def\kzero{1}
    \draw[->,thick, black] (6,3.3) -- (6.6,3.3)
    node[pos=1, anchor=west, color=black] {\small{Upper Bound on $Q_l$}};

    \draw[->,thick, black] (6,2.15) -- (6.6,2.15)
    node[pos=1, anchor=west, color=black] {\small{A sample path for $Q_l\in \mathcal{S}_1$}};

    \draw[->,thick, black] (6,1.3) -- (6.6,1.3)
    node[pos=1, anchor=west, color=black] {\small{Lower Bound on $Q_l$}};

    \begin{axis}[
        legend style={
        at={(0.55,1)},         
        anchor=north,            
        legend columns=2,         
        cells={anchor=west},
        font=\small,                 
        column sep=2pt,              
    },
        xlabel={$l$},
        ylabel={$Q_l$},
        xmin=-1, xmax=8,
        ymin=-1, ymax=4,
        axis lines=middle,
        xtick={0, \kzero, \kone, 2*\kone-\kzero+0.1},
        xticklabels={$0$, $k_0$, $k_1$, $2k_1-k_0$},
        ytick={0, \kzero/2},
        yticklabels={$0$, $\bar{q}$},
        xlabel style={anchor=north east, at={(axis description cs:1.05,0.1)}},
        ylabel style={anchor=north east, at={(axis description cs:0.1,1.1)}},
        axis equal image,
    ]
        \addlegendimage{area legend, fill=blue, draw=blue, fill opacity=0.3, mark=none}
        \addlegendentry{$\mathcal{S}_1$}
        
        \addplot[
            color=blue,      
            fill=blue,       
            fill opacity=0.3,
            draw=none       
        ] coordinates {
            (\kone, \kone/2)          
            (2*\kone-0.5, \kone-0.25)        
            (2*\kone-0.5, \kzero/2-0.2)
            (\kone, \kone/2)
        };

        \addplot[
            color=black,  
            thick,       
            mark size=2pt
        ] coordinates {
            (0, 0) 
            (\kone, \kone/2) 
            (\kone+0.4, \kone/2-0.2)
            (\kone+0.8, \kone/2)
            (\kone+1.2, \kone/2+0.2)
            (\kone+1.6, \kone/2+0.4)
            (\kone+2, \kone/2+0.2)
            (\kone+2.4, \kone/2)
            (\kone+2.8, \kone/2-0.2)
            (\kone+3.2, \kone/2)
            (\kone+3.6, \kone/2+0.2)
        };

        \addplot[
            color=black,  
            thick,  
            dashed,
            mark size=2pt
        ] coordinates {
            (\kone, \kone/2) 
            (2*\kone - 0.5, \kzero/2-0.2)
        };
        \node[rotate=-28] at (5.4,0.95) {\small{Slope $=-1$}};

        \addplot[
            color=black,  
            dashed,       
            thick,
            mark size=2pt
        ] coordinates {
            (0, \kzero/2) 
            (2*\kone-0.5, \kzero/2) 
        };

        \addplot[
            color=black,  
            dashed,       
            thick,
            mark size=2pt
        ] coordinates {
            (0, 0) 
            (2*\kone-0.5, \kone-0.25) 
        };
        \node[rotate=28] at (5.4,3) {\small{Slope $=1$}};
    \end{axis}
\end{tikzpicture}
    \caption{$\mathcal{S}_1=\{\{Q_l\}_{l\geq 0}; Q_0=0, D_l=1~\forall 0\leq l\leq k_1\}$. For illustrative purposes, we represent $Q_l$ as a continuous piecewise linear function by linearly interpolating queue lengths between the time instants.}
    \label{fig:bad sets}
\end{figure}

Finally, we end this discussion by giving another intuitive perspective based on Figure \ref{fig:bad sets} highlighting the critical role Markovian noise plays in divergence. For a fair comparison, suppose that $\{Q_l\}_{l\geq 0}$ are independent and have geometric distribution (same as the stationary distribution of the $M/M/1$ queue). Recall from the previous analysis that we need the queue length $Q_l=\mathcal{O}(l)$ for a sustained time frame $[k_0, k]$. In the i.i.d. setting, $P(Q_l=\mathcal{O}(l))=(1-p)^{\mathcal{O}(l)}p$ which leads to the joint probability $P(Q_l=\mathcal{O}(l),k_0\leq l\leq k)=\prod_{l=k_0}^{k}(1-p)^{\mathcal{O}(l)}p=(1-p)^{\mathcal{O}(k^2)}p^{\mathcal{O}(k)}$. Observe that the exponent here has quadratic growth, making the decay rate too rapid that cannot be overpowered by any exponential growth of the random variable and therefore ensuring that the algorithm remains stable. On the other hand, when $\{Q_l\}_{l\geq 0}$ is Markovian, it retains the memory of the past states. Although the probability of climbing up to a large queue length is still small, the probability that the sample path will meander around this large value for a substantial time is relatively higher. This restricts the abrupt drop in the probability of such a sample path as seen in the i.i.d. case, allowing the exponential growth of the multiplicand to dominate and ultimately cause divergence.

\section{Policy Evaluation for Infinite State Space MDPs}\label{sec:RL}
In this section, we consider the infinite-horizon MDP which is specified by the tuple $(\mathcal{S}, \mathcal{A}, \mathcal{R}, P)$. Here, $\mathcal{S}$ is the state space which may be countably infinite, $\mathcal{A}$ is the finite action space, $\mathcal{R}:\mathcal{S}\times \mathcal{A}\to \mathbb{R}$ is the reward function, and $P:\mathcal{S}\times\mathcal{S}\times \mathcal{A}\to [0,1]$ is the transition kernel. At each time step $k=0,1,2,\dots$, the agent in state $S_k\in \mathcal{S}$ selects an action $A_k\in \mathcal{A}$ sampled from a policy $\pi(\cdot|S_k)$, receives a reward $\mathcal{R}(S_k, A_k)$, and transitions to the next state $S_{k+1}$ sampled from $P(\cdot|S_k, A_k)$.

Consider the problem of evaluating the performance of a policy $\pi$ from data generated by applying $\pi$ to the MDP. Let us denote $P_{\pi}(s'|s)=\sum_{a\in S} P(s'|s,a)\pi(a|s)$ as the transition probabilities, and $\mathcal{R}_\pi(s)=\sum_{a\in \mathcal{A}}\mathcal{R}(s,a)\pi(a|s)$ as the average reward for each state. For clarity, we will drop the subscript $\pi$ from the notation wherever it is evident. We assume that the policy $\pi$ has the following property:
\begin{assumption}\label{assump:dist_td}
    The Markov chain generated by policy $\pi$ is irreducible, aperiodic and has a unique stationary distribution given by $\mu$. Furthermore, the rewards have a finite fourth moment under $\mu$.
    \begin{align*}
        \E_{\mu}[\mathcal{R}^4(S_k,A_k)]=\hat{r}^2<\infty,
    \end{align*}
\end{assumption}
where $\E_{\mu}[\cdot]$ denotes the stationary expectation. Let $\bar{r}=\sum_{s\in \mathcal{S}}\mathcal{R}_{\pi}(s)$. Then, Assumption \ref{assump:dist_td} is sufficient for the existence of a differential value function $V^*: \mathcal{S}\to \mathbb{R}$ that satisfies the Bellman/Poisson equation $\mathcal{B}_{\pi}V^*=V^*$ \cite{derman1967}. Here  $\mathcal{B}$ is defined as:
\begin{align}\label{eq:avg_reward_bellman_eq}
    \mathcal{B}_{\pi}V(s)=\mathcal{R}_{\pi}(s)+\sum_{s'\in \mathcal{S}}P(s'|s)V(s')-\bar{r},~\forall s\in \mathcal{S}.
\end{align}
It should be noted that the average-reward Bellman equation has a unique solution only up to the additive constant. Thus, if there exists a $c\in \mathbb{R}$ such that $V'(s)=V^*(s)+c$, for all $s\in \mathcal{S}$, then $V'$ is also a solution to the Bellman equation.  

Since the state space is infinite, directly estimating $V^*$ from data samples is intractable. Hence, we will use linear function approximation (LFA) to approximate $V^*$. Denote $\psi(s)=(\psi_1(s), \psi_2(s), \dots, \psi_d(s))^T\in \mathbb{R}^d$ as the feature vector for state $s$. Consider an arbitrary indexing of the state space $\mathcal{S}=\{s_1, s_2, s_3,\dots\}$ such that $P(s_{i+1}|s_i)>0$ for all $i\geq 1$. For notational ease, we will concatenate the feature vectors $\psi(s_i)$ into a feature matrix $\Psi$ which has $d$ columns and infinite rows such that the $i$-th row of $\Psi$ corresponds to $s_i$. Let $\Lambda$ be a diagonal matrix (infinite dimensional) with diagonal entries as $\mu(s)$. Then, we have the following assumption on $\Psi$.
\begin{assumption}\label{assump:feature_mat}
    \begin{enumerate}[(a)]
        \item The columns of matrix $\Psi$ are linearly independent. More explicitly,
        \begin{align*}
            \sum_{j=1}^da_j\psi_j(s)=0,~\forall s\in \mathcal{S}\implies a_j=0,~1\leq j\leq d.
        \end{align*}
        \item The columns $\psi_j$ of $\Psi$ satisfy:
        \begin{align*}
            \|\psi_j\|_{\Lambda}^2:=\sum_{i=1}^\infty\mu(s_i)\psi_j^2(s_i)\leq \hat{\psi}^2 <\infty,~ 1\leq j\leq d.
        \end{align*}
    \end{enumerate}
\end{assumption}
\vspace{-2mm}
\begin{remark}
    The assumption of linear independence holds without loss of generality. If any columns of the matrix are linearly dependent, they can be removed without affecting the approximation accuracy. Part (b) states that the $\mu$-weighted $\ell_2$ norm of columns of $\Psi$ are bounded. Note that due to Assumption \ref{assump:dist_td}, $\mu(s)>0$ for all $s\in \mathcal{S}$, hence $\|\cdot\|_{\Lambda}$ is a valid norm. 
\end{remark}

\begin{assumption}\label{assump:stable_markov}
    For any initial state $s_0\in \mathcal{S}$ and $m\geq0$, there exist functions $f_1, f_2, f_3:\mathcal{S}\to [0,\infty)$ and a constant $\rho\in (0,1)$ satisfying the following:
    \begin{align*}
        |\E_{s_0}[\mathcal{R}(S_k, A_k)]-\E_{\mu}[\mathcal{R}_{\pi}(\tilde{S}_k)]|&\leq \rho^kf_1(s_0),\\
        \|\E_{s_0}[\psi(S_k)]-\E_{\mu}[\psi(\tilde{S}_k)]\|_2&\leq \rho^kf_1(s_0),\\
        \|\E_{s_0}[\psi(S_k)\psi(S_{k+m})^T]-\E_{\mu}[\psi(\tilde{S}_k)\psi(\tilde{S}_{k+m})^T]\|_2&\leq \rho^kf_1(s_0),\\
        \|\E_{s_0}[\psi(S_k)\mathcal{R}(S_{k+m}, A_{k+m})]-\E_{\mu}[\psi(\tilde{S}_k)\mathcal{R}_{\pi}(\tilde{S}_{k+m})]\|_2&\leq \rho^kf_1(s_0),
    \end{align*}
    where for all $k\geq 0$, $f_1(\cdot)$ satisfies: $ \E_{s_0}[f^4_1(S_k)]\leq f_2(s_0)$. Furthermore, for all $k\geq 0$, we have
    \begin{align*}
        \E_{s_0}[\|\psi(S_k)\|_2^4]\leq f_3(s_0);~~\E_{s_0}[\mathcal{R}^4(S_k, A_k)]\leq f_3(s_0).
    \end{align*}
\end{assumption}

\begin{remark}
    Note that these assumptions are always true for finite state space. For infinite-state space, they quantify the stability of the Markov chain. Overall, while these assumptions are technical, they are 
    fairly mild in many practical applications of interest. For example, in stable queueing systems with downward drift, the stationary distribution is light-tailed and decreases geometrically fast with queue length. Moreover, one can show rapid convergence of $P^m(\cdot|s_0)\to \mu$ for these Markov chains \cite{stamoulis1990settling, meyn1994, tweedie1996}. Thus, for rewards and feature vectors with polynomial growth with state, these assumptions are readily satisfied. 
\end{remark}
\begin{remark}
    Although, these assumptions imply certain level of mixing in the Markov chain, it remains unclear if one can use the techniques in the existing literature for this setting. This is because these works use the finiteness of the noise to control the growth rate of the iterate by picking small enough step-size. However, this is not possible in the case of unbounded space since the noise can be arbitrarily large.
\end{remark}

\subsection{Average-Reward TD\texorpdfstring{$(\lambda)$} \ }
We now define a modified version of the Bellman operator which is essential for understanding TD$(\lambda)$. For any $\lambda\in [0,1)$, define $\mathcal{B}_{\pi}^{(\lambda)}(V)=(1-\lambda)\left(\sum_{m=0}^\infty\lambda^m\mathcal{B}_{\pi}^m(V)\right)$, where $\mathcal{B}_{\pi}^m(\cdot)$ is the $m$-step Bellman operator. It is easy to verify that $\mathcal{B}_{\pi}^{(\lambda)}$ and $\mathcal{B}_{\pi}$ have the same set of fixed points. Since we are restricting our search for the value function in the subspace spanned by $\Psi$, we instead solve the corresponding $\lambda$-weighted projected Bellman equation:
\begin{align}\label{eq:proj_td}
    \Psi\theta=\Pi_{\Lambda,\Psi}\left(\mathcal{B}_{\pi}^{(\lambda)}\Psi\theta\right),
\end{align}
where $\theta\in \mathbb{R}^d$ is the parameter variable and $\Pi_{\Lambda,\Psi}=\Psi(\Psi^T\Lambda\Psi)^{-1}\Psi^T\Lambda$ ($(\Psi^T\Lambda\Psi)^{-1}$ is a $d\times d$ matrix which well defined due to Assumption \ref{assump:feature_mat}) is the projection operator onto the column space of $\Psi$ with respect to $\|\cdot\|_{\Lambda}$. Let $E_{\Psi}$ be a subspace defined as follows:
\begin{align*}
    E_{\Psi}&=\text{span}\{\theta|\psi(s_i)^T\theta=1, \forall i\} \\
    &=\begin{cases}
        \{c\theta_e|c\in\mathbb{R}\},  &\hspace{-2mm}\text{if } \exists\theta_e\in \mathbb{R}^d \text{ and } \psi(s_i)^T\theta_e=1, \forall i\\
        \{0\}, &\text{otherwise}\\
    \end{cases}
\end{align*}
Observe that if $E_{\Psi}\neq \{0\}$, then Eq. \ref{eq:proj_td} has infinitely many solutions of the form $\{\theta^*+c\theta_e|c\in\mathbb{R}\}$, where $\theta^*$ is the solution in the orthogonal complement of $E_{\Psi}$ denoted by $E_{\Psi}^{\perp}$. As shown in Part (b) of Proposition \ref{prop:td_prop}, this solution is unique in the subspace $E_{\Psi}^{\perp}$. Let $\Pi_{2, E^{\perp}_{\Psi}}(\theta)=\theta-\theta_e\langle\theta, \theta_e\rangle/\|\theta_e\|_2^2$, which is the projection operator onto the subspace $E^{\perp}_{\Psi}$\footnote{This should not be confused with $\Pi_{\mathcal{X}}(\cdot)$ which is projection onto the compact set $\mathcal{X}$.} and $\mathcal{X}$ be an $\ell_2$-ball in $\mathbb{R}^{d+1}$ such that $[\bar{r},\theta^{*T}]^T\in \mathcal{X}$. Then, we have Algorithm \ref{algo:TD} to estimate $\bar{r}$ and $\theta^*$.

\begin{algorithm2e}[!ht]
     \SetKwInOut{Input}{Input}
     \SetKwInOut{Output}{Output}
     \Input{$\lambda \in [0,1)$,  $c_\alpha > 0$, basis functions $\{\psi_i\}^d_{i=1}$, step-size sequence $\{\alpha_k\}_{k \geq 0}$.}
     Initialize $z_{-1} = 0$, $\bar{r}_0\in \mathbb{R}$ and $\theta_0\in E^{\perp}_{\Psi}$ arbitrarily.\\
     \For{$k=0,1,\ldots$}{
     	 Observe $(S_k,\mathcal{R}(S_k, A_k),S_{k+1})$\\
         $\delta_k = \mathcal{R}(S_k, A_k) - \bar{r}_k + \psi(S_{k+1})^T \theta_k - \psi(S_k)^T \theta_k$\\
         $z_k = \lambda z_{k-1} + \psi(S_k)$ \\
         $\tilde{r}_{k+1} = \bar{r}_k + c_\alpha \alpha_k (\mathcal{R}(S_k, A_k) - \bar{r}_k)$\\
         $\tilde{\theta}_{k+1} = \theta_k + \alpha_k \delta_k \Pi_{2, E^{\perp}_{\Psi}}z_k$\\
         $[\bar{r}_{k+1}, \theta_{k+1}^T] = \Pi_{\mathcal{X}}([\tilde{r}_{k+1}, \tilde{\theta}_{k+1}^T])$
        }
\caption{Average-reward TD($\lambda$) with LFA} 
\label{algo:TD}
\end{algorithm2e}
The algorithm is essentially equivalent to TD$(\lambda)$ for average-reward setting, however one difference lies in the update of the parameter $\theta_k$. Specifically, we project the eligibility trace vector $z_k$ onto $E^{\perp}_{\Psi}$ which restricts the iterates $\{\theta_k\}_{k\geq 0}$ to the subspace $E^{\perp}_{\Psi}$. This ensures that the convergent point of the algorithm unique. Finally, we use projection onto $\mathcal{X}$ to control the growth of the concatenated vector $[\bar{r}_{k+1}, \theta_{k+1}^T]$. 

\subsubsection{Properties of TD\texorpdfstring{$(\lambda)$}\ \ algorithm}
To transform Algorithm \ref{algo:TD} in the form of iteration \eqref{eq:main_rec}, we construct a process $\mathcal{M}_{Y}=\{Y_k=(S_k, A_k, S_{k+1}, z_k)\}_{k\geq 0}$ taking values in the space $\mathcal{Y}:=\mathcal{S}\times \mathcal{A}\times \mathcal{S}\times \mathbb{R}^d$. It is easy to verify that $Y_k$ forms a Markov process in the continuous unbounded state space\footnote{We would like to highlight the distinction between infinite and unbounded state space here. In the analysis of TD$(\lambda)$ for finite Markov chains, the process $Y_k$ also has infinite state space, but it is \textit{bounded}.}. Let us define $x_k:=[\bar{r}_k, \theta_k^T]^T$, then the iterations can be compactly written as:
\begin{align}\label{eq:linear_SA}
    x_{k+1}&=\Pi_{\mathcal{X}}\left(x_k+\alpha_k F(x_k, Y_k)\right),
\end{align}
where $F(x_k, Y_k)=T(Y_k)x_k+b(Y_k)$ and 
$$T(Y_k)=\begin{bmatrix}
        -c_\alpha & 0\\
        -\Pi_{2, E^{\perp}_{\Psi}}z_k & \Pi_{2, E^{\perp}_{\Psi}}z_k\left(\psi(S_{k+1})^T-\psi(S_k)^T\right)
    \end{bmatrix};~~b(Y_k)=\begin{bmatrix}
            c_\alpha\mathcal{R}(S_k, A_k)\\
            \Pi_{2, E^{\perp}_{\Psi}}\mathcal{R}(S_k, A_k)z_k
        \end{bmatrix}.$$

To consider the stationary behavior of $Y_k$, let $\{\tilde{S}_k, A_k\}_{k\geq 0}$ denote the stationary process. Then, $\tilde{z}_k:=\sum_{\ell=-\infty}^k\lambda^{k-\ell}\psi(\tilde{S}_\ell)$ and $\tilde{Y}_k=(\tilde{S}_k, A_k, \tilde{S}_{k+1}, \tilde{z}_k)$ are the stationary analogs of $z_k$ and $Y_k$, respectively. Let the stationary expectation of the matrices $T(\tilde{Y}_k)$ and $b(\tilde{Y}_k)$ be denoted by $\bar{T}$ and $\bar{b}$. Then, we have the following lemma for $\bar{T}$ and $\bar{b}$ whose proof is given in \ref{appendix:stationary_exp_Tb}. 
\begin{lemma}\label{lem:stationary_exp_Tb}
    Under Assumption \ref{assump:dist_td} and \ref{assump:feature_mat}, the stationary expectations $\bar{T}$ and $\bar{b}$ are finite and given by
    \begin{align*}
        \bar{T}&=\begin{bmatrix}
        -c_\alpha & 0\\
        -\frac{1}{(1-\lambda)}\Pi_{2, E^{\perp}_{\Psi}}\Psi^T\mu & \Pi_{2, E^{\perp}_{\Psi}}\Psi^T\Lambda(P^{(\lambda)}-I)\Psi
        \end{bmatrix};~~ \bar{b}=\begin{bmatrix}
            c_\alpha\bar{r}\\
            \Pi_{2, E^{\perp}_{\Psi}}\Psi^T\Lambda \mathcal{R}^{(\lambda)}
        \end{bmatrix},
    \end{align*}
    where $P^{(\lambda)}=(1-\lambda)\sum_{m=0}^\infty \lambda^mP^{m+1}$ and $\mathcal{R}^{(\lambda)}=(1-\lambda)\sum_{m=0}^\infty\lambda^m\sum_{l=0}^mP^l\mathcal{R}_{\pi}$. 
\end{lemma}

Now using Lemma \ref{lem:stationary_exp_Tb} and Assumptions \ref{assump:dist_td}-\ref{assump:stable_markov}, we have the following Proposition for the verification of the Assumptions \ref{assump:linear_F}-\ref{assump:Poisson_eq}.   
\begin{proposition}\label{prop:td_prop}
    Under Assumptions \ref{assump:dist_td} and \ref{assump:feature_mat}, the TD($\lambda$) algorithm satisfies the following:
    \begin{enumerate}[(a)]
        \item The operator $F(x_k, Y_k)$ defined in Eq. \eqref{eq:linear_SA} has the following properties:
        \begin{enumerate}[(1)]
            \item $\|F(x_k, Y_k)\|_2\leq \|T(Y_k)\|_2\|x-x^*\|_2+\|b(Y_k)\|_2+\|T(Y_k)\|_2\|x^*\|_2$, where $\E_{y_0}[\|T(Y_k)\|^2_2]\leq \hat{T}(s_0, s_1)$ and $\E_{y_0}[\|b(Y_k)\|^2_2]\leq \hat{b}(s_0, s_1)$.
            \item Define $\bar{F}(x)= \E_{\tilde{Y}\sim \mu}[F(x, \tilde{Y})]$. Then, $\bar{F}(x)$ exists and is given by $\bar{F}(x)=\bar{T}x+\bar{b}$.
            \item There exists a unique $\theta^*\in E_{\Psi}^{\perp}$ such that $x^*=(\bar{r}, \theta^{*T})^T$ solves $\bar{T}x+\bar{b}=0$. Furthermore, it is also one of the solutions to the Projected-Bellman equation $\mathcal{B}_{\pi}^{(\lambda)}(\Psi\theta)=\Psi\theta$. \end{enumerate}
        \item There exists a solution to the Poisson equation \eqref{eq:Poisson_eq} for the Markov chain $\mathcal{M}_Y$ which satisfies Assumption \ref{assump:Poisson_eq} with $\hat{A}_2^2(y_0)=9\hat{g}(s_0, s_1)$ and $\hat{B}_2^2(y_0)=2(\|x^*\|_2^2+1)\hat{g}(s_0, s_1)+8c^2_\alpha\sqrt{f_2(s_0)+f_2(s_1)}/(1-\rho)^2$.
    \end{enumerate}
\end{proposition}
Refer to Appendix \ref{appendix:td_prop} for proof and the constants $\hat{T}(s_0, s_1)$, $\hat{b}(s_0, s_1)$, and $\hat{g}(s_0, s_1)$.

\subsubsection{Finite Sample Bounds for TD\texorpdfstring{$(\lambda)$}\ }
We pick $\Phi(x_k-x^*)=(\bar{r}_k-\bar{r})^2+\|\theta_k-\theta^*\|_2^2$ as our Lyapunov function. A key insight in \cite{zhang2021} that established the negative drift was the observation that for any function in the set $\{V|\sum_{s\in \mathcal{S}}V(s)=0, \sum_{s\in \mathcal{S}}
V^2(s)=1\}$, there exists a $\Delta>0$ such that 
\begin{align*}
    V^T\Lambda(I-P^{(\lambda)})V\geq \Delta.
\end{align*}
However, in general such an inequality is not true for infinite state space, as explained in Appendix \ref{appendix:challenge}.

To establish a similar drift condition in our setting, we leverage the fact that the matrix $\Psi$ has a finite number of linearly independent columns. This effectively restricts the value function to a finite-dimensional subspace, allowing us to prove the following lemma. For proof please refer to Appendix \ref{appendix:contraction_td}.

\begin{lemma}\label{lem:contraction_td}
    Under Assumption \ref{assump:dist_td} and  \ref{assump:feature_mat}, we have
    \begin{align*}
        \Delta:=\min_{\theta\in E^{\perp}_{\Psi},\|\theta\|_2=1}\theta^T\Pi_{2, E^{\perp}_{\Psi}}\Psi^T\Lambda(I-P^{(\lambda)})\Psi\theta>0.
    \end{align*}
    Furthermore, when $c_\alpha\geq \Delta+\sqrt{\frac{d^2\hat{\psi}^4}{\Delta^2(1-\lambda)^4}-\frac{d\hat{\psi}^2}{(1-\lambda)^2}}$, we have $-x^T\bar{T}x\geq \frac{\Delta}{2}\|x\|_2^2$ for all $x\in \mathbb{R}\times E^{\perp}_{\Psi}$.
\end{lemma}

Since $\Phi(x)$ is $\ell_2$-norm squared and has a negative drift, Assumption \ref{assump:Lyapunov_fn} is also verified. With all the assumptions now satisfied, we can apply Theorem \ref{thm:main_thm;bounded} to get the following sample complexity for TD($\lambda$).

\begin{theorem}\label{thm:td}
    Consider the iterates $\{\theta_k, \bar{r}_k\}_{k\geq 0}$ generated by Algorithm \ref{algo:TD} under Assumption \ref{assump:dist_td}-\ref{assump:stable_markov} and $c_\alpha\geq \Delta+\sqrt{\frac{d^2\hat{\psi}^4}{\Delta^2(1-\lambda)^4}-\frac{d\hat{\psi}^2}{(1-\lambda)^2}}$.
    \begin{enumerate}[(a)]
        \item When $\alpha_k\equiv\alpha\leq 1$, then for all $k\geq 0$:
       \begin{align*}
           \E[(\bar{r}_{k+1}-\bar{r})^2+\|\theta_{k+1}-\theta^*\|_2^2]\leq \varphi_{V,0}\exp\left(-\frac{\Delta\alpha k}{2}\right)+3\hat{C}_{V}(s_0, s_1)\alpha+\frac{12\hat{C}_{V}(s_0, s_1)\alpha}{\Delta}.
       \end{align*}
       \item 
       When $\xi=1$, $\alpha> 1/\Delta$ and $K\geq \max\{\alpha, 2\}$, then for all $k\geq 0$:
       \begin{align*}
           \E[(\bar{r}_{k+1}-\bar{r})^2+\|\theta_{k+1}-\theta^*\|_2^2]\leq \varphi_{V,0}\left(\frac{K}{k+K}\right)^{\frac{\Delta\alpha}{2}}+\frac{\hat{C}_V(s_0, s_1)\alpha}{k+K}+\frac{4(6+2\Delta)\hat{C}_V(s_0, s_1)e\alpha^2}{\left(\frac{\Delta\alpha}{2}-1\right)(k+K)}.
       \end{align*}
    \end{enumerate}
\end{theorem}
Refer to Appendix \ref{appendix:td_thm} for rate of convergence for the constants $\varphi_{V,0}$ and $\hat{C}_V(s_0, s_1)$.
\begin{remark}
    It is evident from the above bound that to find a pair $(r, \theta)$ such that $\E[|r-\bar{r}|]\leq \epsilon$ and $\E[\|\theta-\theta^*\|_2]\leq \epsilon$, one needs at most $\mathcal{O}\left(1/\epsilon^2\right)$ number of samples.
\end{remark}

Remarkably, one can show the convergence of Algorithm \ref{algo:TD} even without using $\Pi_{\mathcal{X}}(\cdot)$ in the final step. Recall that due to projection of the iterates on $E_{\Psi}^{\perp}$, the fixed point $\theta^*$ is unique. Thus, with an additional assumption on Markov chain, we can apply the general result on SA from \cite{benveniste2012} (Theorem 17, Page 239) to show a.s. convergence.

\begin{theorem}\label{lem:avg_TD_ASconvergence}
    Suppose that in addition to Assumptions \ref{assump:dist_td}-\ref{assump:stable_markov}, we have $\E_{s_0}[f^q_1(S_k)]\leq f_1^q(s_0)$ for all $q> 0$. Then, the Algorithm \ref{algo:TD} a.s. converges to $(\bar{r},\theta^*)$.
\end{theorem}
\begin{remark}
    Previous works showed a.s. convergence when $E_{\Psi}=\{0\}$ since the limit point is unique in this case. However, by utilizing the uniqueness of solution in $E_{\Psi}^{\perp}$ and the ease of projection in such a space, we eliminate any such assumptions on $\Psi$. As highlighted before, one does not need the final projection onto the bounded set for asymptotic convergence.
\end{remark}

\subsection{Approximation Error due to LFA}
Now we will provide a bound on the approximation error of the limiting weights $\theta^*$. In contrast to discounted setting, which admits a unique solution, the average-reward Bellman equation has an infinite number of solutions since any constant function added to one of the solutions also satisfies the Bellman equation. To handle this anomaly, we will slightly modify the definition of error in the following sense. Fix $V^*$ as the unique solution of Eq. \eqref{eq:avg_reward_bellman_eq} such that $\E_{\mu}[V^*(S)]=0$. Then, the approximation error for any function in the feature space $\Psi\theta$ for $\theta\in \mathbb{R}^d$ is obtained by adding a constant $c$ to each element of $V^*$ and choosing the least possible error from this shift. More specifically, with slight abuse of notation, let $e$ be an infinite-dimensional vector with all the entries equal to 1. Then, the approximation error of $\theta$ is defined as the following:
\begin{align*}
    \inf_{c\in\mathbb{R}}\sum_{i=1}^\infty\mu(s_i)(\psi(s_i)^T\theta-c-V^*(s_i))^2=\inf_{c\in\mathbb{R}}\|\Psi\theta-ce-V^*\|^2_{\Lambda}=\|\Psi\theta-(\mu^T\Psi\theta) e-V^*\|^2_{\Lambda}.
\end{align*}
where for the last inequality we used the fact that the minimum is obtained at $c^*=\E_{\mu}[\psi(S)^T\theta-V^*]=\E_{\mu}[\psi(S)^T\theta]$. Note that for any function $x: \mathcal{S}\to \mathbb{R}$, we have $x-\mu^Txe=(I-e\mu^T)x$. Thus, we can define $\bar{\Psi}=(I-e\mu^T)\Psi$ and write the above equation as
\begin{align}\label{eq:error-def}
    p(\Psi\theta-V^*):=\sqrt{\inf_{c\in\mathbb{R}}\sum_{i=1}^\infty\mu(s_i)(\psi(s_i)^T\theta-c-V^*(s_i))^2}=\|\bar{\Psi}\theta-V^*\|_{\Lambda}.
\end{align}
Define $\bar{\Pi}_{\Lambda,\Psi}=\bar{\Psi}(\bar{\Psi}^T\Lambda\bar{\Psi})^{-1}\bar{\Psi}^T\Lambda$. Then, it can be verified that the solution $\theta^*\in E_{\Psi}^{\perp}$ satisfies $\bar{\Pi}_{\Lambda,\Psi}(\mathcal{B}_{\pi}^{(\lambda)}\bar{\Psi}\theta^*)=\bar{\Psi}\theta^*$ \cite[Section 4]{tsitsiklis1999}.
Since the approximation error is now expressed in terms of $\bar{\Psi}$, and we know that $\theta^*$ also satisfies the modified Bellman equation with $\bar{\Psi}$ as the new feature matrix, we will work with $\bar{\Pi}_{\Lambda,\Psi}$ instead of $\Pi_{\Lambda,\Psi}$ for easier analysis. 

Next, recall that in the discounted setting, the contraction property of the Bellman operator plays a crucial role in obtaining the approximation error \cite[Lemma 6]{tsitsiklis1997}. In contrast, no such discount factor exists in the average-reward case and as a result, we need to use the mixing property of the Markov kernel $P$ to get the desired contraction and eventually obtain the upper bound on the approximation error. Recall that the Markov kernel associated with the $\lambda$-weighted projected Bellman equation is given by $P^{(\lambda)}$ in Lemma \ref{lem:stationary_exp_Tb}. Thus, we will assume the following about the mixing property of $P^{(\lambda)}$. 
\begin{assumption}\label{assump:mixing-factor}
    There exists a constant $\nu_{\lambda}$ such that for all $\lambda$  the following relation holds true 
    \begin{align*}
        \nu_{\lambda}=1-\max_{\substack{\mu^Tx=0}}
        \frac{x^T\Lambda P^{(\lambda)}x}{x^T\Lambda x}>0.
    \end{align*}
\end{assumption}

Finally, following the idea in \cite{tsitsiklis1999}, one can obtain a much tighter approximation bound by using the following simple trick. Let $\delta \in (0,1]$ and define a lazy Markov chain as $P_{\delta}=(1-\delta)I+\delta P$. Define $\mathcal{R}^{(\delta)}_{\pi}(s)=\delta \mathcal{R}_{\pi}(s)$ for all $s\in \mathcal{S}$, then note that $V^*$ is also the solution to the following auxiliary Bellman equation:
\begin{align}\label{eq:avg-bellman-op}
    \mathcal{B}^{(\delta)}_{\pi}V(s)=\mathcal{R}^{(\delta)}_{\pi}(s)+\sum_{s'\in \mathcal{S}}P_{\delta}(s'|s)V^*(s')-\delta \bar{r}=V^*(s).
\end{align}
Furthermore, using the linearity of the projection operator, it is easy to show that $\theta^*$ also satisfies the corresponding projected Bellman equation:
\begin{align*}
    \bar{\Pi}_{\Lambda,\Psi}\left(\mathcal{B}^{(\delta)}_{\pi}\bar{\Psi}\theta^*\right)&=\bar{\Pi}_{\Lambda,\Psi}\left(\mathcal{R}^{(\delta)}_{\pi}(s)+\sum_{s'\in \mathcal{S}}P_{\delta}(s'|s)\bar{\Psi}\theta^*-\delta \bar{r}\right)\\
    &=\bar{\Pi}_{\Lambda,\Psi}\left(\delta\mathcal{R}_{\pi}(s)+\delta\sum_{s'\in \mathcal{S}}P(s'|s)\bar{\Psi}\theta^*-\delta \bar{r}\right)+(1-\delta)\bar{\Pi}_{\Lambda,\Psi}\left(\bar{\Psi}\theta^*\right)\\
    &=\delta\bar{\Pi}_{\Lambda,\Psi}\left(\mathcal{B}_{\pi}\bar{\Psi}\theta^*\right)+(1-\delta)\bar{\Psi}\theta^*\\
    &=\bar{\Psi}\theta^*.
\end{align*}
To extend the above idea to the $\lambda$-weighted projected Bellman operator, we define $P_{\delta}^{(\lambda)}=(1-\delta)I+\delta P^{(\lambda)}$ and $\mathcal{B}_{\pi}^{(\lambda, \delta)}=(1-\delta)+\delta \mathcal{B}_{\pi}^{(\lambda)}$, and use similar steps as before to show that $\bar{\Pi}_{\Lambda,\Psi}(\mathcal{B}^{(\lambda, \delta)}_{\pi}\bar{\Psi}\theta^*)=\bar{\Psi}\theta^*$. Clearly, now we can leverage the freedom of choosing $\delta$ and get the tightest possible contraction factor by optimizing over it. Specifically, we define the contraction factor as
\begin{align*}
    \gamma_{\lambda}=\inf_{\delta} \sup_{x} \frac{\|\bar{\Pi}_{\Lambda,\Psi}(P_{\delta}^{(\lambda)}x)\|_{\Lambda}}{\|x\|_{\Lambda}}=\inf_{\delta}\|\bar{\Pi}_{\Lambda,\Psi}P_{\delta}^{(\lambda)}\|_{\Lambda}.
\end{align*}
where $\|\cdot\|_{\Lambda}$ is the operator norm. Now, we establish the approximation error for TD$(\lambda)$-Learning with LFA in the following  whose proof is given in Appendix \ref{sec:LFA_error}.
\begin{theorem}\label{thm:LFA_error}
    Let $\theta^*\in E_{\Psi}^{\perp}$ be the unique solution to the $\lambda$-weighted projected Bellman equation Eq. \eqref{eq:proj_td}. Under Assumptions \ref{assump:dist_td}-\ref{assump:mixing-factor}, we have the following:
    \begin{enumerate}[(a)]
        \item  For each $\lambda$, $\gamma_{\lambda}\leq \sqrt{1-\nu_\lambda/2}<1$.
        \item The approximation error is upper bounded by
        \begin{align*}
            p(\Psi\theta^*-V^*)\leq\frac{1}{\sqrt{1-\gamma_{\lambda}^2}}\inf_{\theta\in \mathbb{R}^d}p(\Psi\theta-V^*).
        \end{align*}
    \end{enumerate}
\end{theorem}
Part (a) provides an upper bound on the contraction factor in terms of the mixing factor of $P^{(\lambda)}$ and ensures that it is strictly less than 1. In part (b), we observe that the upper bound is a multiple of $\inf_{\theta\in \mathbb{R}^d}\|\bar{\Psi}\theta-V^*\|_{\Lambda}$ which is the least error that can be achieved using the feature matrix $\bar{\Psi}$. Additionally, if the solution $V^*$ lies in the range space of $\bar{\Psi}$, or in other words if $V^*=\Psi\theta+ce$ for some $\theta\in \mathbb{R}^d$ and $c\in \mathbb{R}$, then the limiting solution $\theta^*$ to TD$(\lambda)$ algorithm \ref{algo:TD} solves the Bellman equation \eqref{eq:avg_reward_bellman_eq} exactly. 

Note that we assumed mixing on $P^{(\lambda)}$ which has no physical significance and only emerged in the analysis of our Algorithm \ref{algo:TD}. Instead, it is more natural to assume the mixing property on the original underlying Markov kernel $P$ and get bounds in terms of $P$. Unfortunately, to the best of our knowledge, there does not exist any closed form relationship between the mixing factor of $P$ and $P^{(\lambda)}$ even in finite-state setting for a general Markov chain. Nevertheless, for the special case of reversible Markov chains, we can obtain bounds in terms of the absolute spectral gap of $P$.

\subsubsection{Special Case of Reversible Markov Chains}
Many queueing systems, such as birth-death chains, exhibit rich structural properties including reversibility of the underlying Markov operator $P$. Thus, we will now specialize the results in Theorem \ref{thm:LFA_error} for the case of reversible Markov chains to give a more interpretable result. Specifically, we will characterize the contraction factor $\gamma_{\lambda}$ in terms of the absolute spectral gap of the Markov operator $P$. To this end, we will assume the following: 
\begin{altassumption}\label{assump:spec-norm}
    The absolute spectral gap of $P$ satisfies:
     \begin{align*}
         \nu:=1-\max_{\substack{\mu^Tx=0}}
        \left|\frac{x^T\Lambda Px}{x^T\Lambda x}\right|>0.
     \end{align*}
\end{altassumption} 
\begin{remark}
    It is worth noting the relation between $\rho$ in Assumption \ref{assump:stable_markov} and the absolute spectral gap of $P$ for the reversible setting. By \cite[Proposition 22.2.8]{douc2018markov} for reversible Markov chains, the assumption holds by setting $\rho=1-\nu$.
\end{remark}
\newpage
Now we present the following lemma that characterizes $\gamma_{\lambda}$ in terms of $\nu$.
\begin{lemma}\label{lem:rev-td-error}
     Under Assumptions \ref{assump:dist_td}-\ref{assump:stable_markov} and \ref{assump:spec-norm}, for each $\lambda$, $\gamma_{\lambda}\leq ((1-\lambda)(1-\nu))/(1-(1-\nu)\lambda)$ and $\lim_{\lambda \uparrow 1} \gamma_{\lambda}=0$.
\end{lemma}

\begin{remark}
     Our upper bound on $\gamma_{\lambda}$ suggests that larger value of $\lambda$ leads to lower approximation error. However, note that the constants $\phi_{V, 0}$ and $\hat{C}_V(s_0, s_1)$ in Theorem \ref{thm:td} scale inversely with respect to $1-\lambda$ (see Appendix \ref{appendix:td_prop} and \ref{appendix:td_thm}). Thus, when $\lambda$ is very close to 1, these constants will become extremely large which might slow down the convergence. In summary, combining Theorem \ref{thm:td} and Theorem \ref{thm:LFA_error}, our bounds suggest choosing a moderate value for $\lambda$ to balance the trade-off between faster rate of convergence and smaller approximation error.
\end{remark}

\section{Generalized Linear Model with Markovian data}\label{sec:GLR}
In this section, we will study the statistical learning problem of estimating an unknown parameter $x^*\in\mathbb{R}^d$ observed through the sequence of pairs $\{(z_k, \phi_k)\}_{\geq 0}$ governed by the following generalized linear model:
\begin{align}\label{eq:GLM}
    z_k=g(\phi_k^Tx^*)+v_k
\end{align}
 where $\{v_k\}_{k\geq 0}$ is additive zero-mean i.i.d. noise such that $\E[v_k^2]=\sigma_v^{(2)}$, $\phi_k\in \mathbb{R}^d$ denotes the Markovian regressor, and $z_k$ denotes the received signal. 
 Further, $g: \mathbb{R}\to \mathbb{R}$ is a known link/activation function (possibly non-linear) that satisfies the following properties.
\begin{assumption}\label{assump:act_fns}
    For all $z, z'\in \mathbb{R}$, we have:
    \begin{align*}
        |g(z)-g(z')|&\leq L_1|z-z'|\tag{Lipschitz Continuity}\\
        (g(z)-g(z'))(z-z') &\geq \mu|z-z'|^2\tag{Strong Monotonicity}
    \end{align*}
    Furthermore, since $g(\cdot)$ is Lipschitz continuous, it is differentiable almost everywhere by virtue of Rademacher theorem \cite{folland1999real}. We assume that for any $x_1, x_2\in \mathbb{R}^d$ there exists a $L_2>0$ such that the gradient of $g(\cdot)$, wherever defined, with respect to $\phi$ satisfies:
    \begin{align*}
        \|\nabla_{\phi}g(\phi^Tx_1)-\nabla_{\phi}g(\phi^Tx_2)\|_2\leq L_2\|x_1-x_2\|_2.
    \end{align*}
    For notational simplicity, we define $L=\max\{L_1, L_2\}$ as the common Lipschitz constant for both $g(\cdot)$ and $\nabla_{\phi}g(\cdot)$.
\end{assumption}
\begin{remark}
     Some of the well-known examples of non-linear functions which satisfy the above assumption include activations functions used in neural networks such as Leaky ReLU, Maxout or the identity function. Other popular functions include ReLU, GeLU, sigmoid, etc. but they do not satisfy the strong monotonicity condition. Specifically, strong monotonicity implies that gradient of the function $g$ is bounded below by $\mu>0$, however, the aforementioned functions have vanishing gradients. Nevertheless, we remark that if one can ensure, independent of the setup, that the input to these activation functions is from a bounded set, then the strong monotonicity condition holds and our results can easily be carried to these cases. 
\end{remark}

We suppose that the sequence of regressors $\{\phi_k\}_{k\geq 0}$ is generated according to an auto-regressive process which satisfies following assumption.
\begin{assumption}\label{assump:auto-reg}
    Let $M\in \mathbb{R}^{d\times d}$ matrix whose eigenvalues lie strictly inside a unit circle. Then, the Markov process $\{\phi_k\}_{k\geq 0}$ is given by:
    \begin{align*}
        \phi_{k+1}=M\phi_k+w_k
    \end{align*}
    where $\{w_k\}_{k\geq 0}$ is an i.i.d. sequence of zero-mean noise with $\E[\|w_1\|_2^4]=\sigma_w^{(4)}$
    and $\phi_0\in \mathbb{R}^d$ is the initial state. 
\end{assumption}
\begin{remark}
    Note that in contrast to previous work \cite{kotsalis2022simple}, we do not impose the bounded support condition for $w_k$. As a result, $w_k$ can even be a noise sample from unbounded distributions, provided the fourth moment is finite such as in the case of Gaussian or exponential distributions. Another related work in this setting is \cite{nagaraj2020least} that studies linear regression with Markovian data. However, their analysis is restricted to the special case where $w_k$ is Gaussian noise.
\end{remark}
\begin{remark}
    The problem of filter design in signal processing can also be modeled using this framework by setting $g(\cdot)$ as the identity function \cite{benveniste2012, Kushner1997StochasticAA}. In signal processing, $\phi_k$ denotes the received signal sequence where the dimension of $\phi_k$ corresponds to the window length, $z_k$ corresponds to the original signal sent from sender and $x^*$ is the optimal filter at the receiver's end which minimizes the mean-square error between original and processed signal.
\end{remark}
Note that since the eigenvalues of $M$ lie strictly inside the unit circle, there exists constants $D\in [1, \infty)$ and $\rho\in (0, 1)$ such that $\|M^k\|_2\leq D\rho^k$ \cite{horn2012matrix}. Using this property in conjunction with the martingale convergence theorem, it is easy to argue that the process $\{\phi_k\}_{k\geq 0}$ converges a.s. to $\phi_{\infty}$ that has a unique stationary distribution which we denote by $\mu_{\phi}$ \cite[Part 2, Chapter 2]{benveniste2012}. Furthermore, this implies that $\Sigma_\phi(k):=\E[\phi_k\phi_k^T]$ converges to $\Sigma_\phi^*\succ 0$ which is given by the solution to the following Lyapunov equation
\begin{align*}
    \Sigma_\phi^*=M\Sigma_\phi^*M^T+\Sigma_w
\end{align*}
where $\Sigma_w=\E[w_kw_k^T]$. 

Let $\mathcal{X}$ be a sufficiently large ball around the origin such that $x^*\in \mathcal{X}$. Then, given the above assumption, we will run the following algorithm to estimate $x^*$, also called Least Mean Squares (LMS) in the signal processing literature. In particular, we have the following algorithm
\begin{align}\label{eq:LMS}
    x_{k+1}=\Pi_{\mathcal{X}}\left(x_k+\alpha_k\phi_k(z_k-g(\phi_k^Tx_k))\right)
\end{align}
where $\alpha_k$ is chosen such that Assumption \ref{assump:step-size} is satisfied.

\subsection{Properties of LMS algorithm}
To reformulate the Algorithm \eqref{eq:LMS} into the framework of Eq. \eqref{eq:main_rec}, we define the Markovian noise as $Y_k=(\phi_k, z_k)$. Let $\mathcal{M}_Y=\{Y_k\}_{k\geq 0}$ be the Markov process that takes values in $\mathcal{Y}:=\mathbb{R}^d\times \mathbb{R}$. It is easy to argue using Bayes' Rule and independence of $v_k$ that if $\phi_k\sim \mu_{\phi}$, then $\text{Law}(Y_{k+1})=\text{Law}(Y_k)$. In other words, $\{Y_k\}_{k\geq 0}$ has a unique stationary distribution (dependent upon the distribution of $v_k$) which we denote by $\mu_Y$. Define the operator $F(\cdot, \cdot)$ as follows
\begin{align*}
    F(x_k, Y_k)=\phi_k(z_k-g(\phi_k^Tx_k)).
\end{align*}
Then, the algorithm can be written as
\begin{align}\label{eq:LMS_reform}
    x_{k+1}=\Pi_{\mathcal{X}}\left(x_k+\alpha_kF(x_k, Y_k)\right).
\end{align} 
Using Assumption \ref{assump:auto-reg}, we have the following proposition to verify the required assumptions on $F(\cdot, \cdot)$ and the Markov process $\{Y_k\}_{k\geq 0}$. The proof is provided in Appendix \ref{sec:prop_LMS}.
\begin{proposition}\label{prop:LMS}
    Under Assumption \ref{assump:auto-reg}, the LMS algorithm satisfies the following:
    \begin{enumerate}[(a)]
        \item The operator $F(x_k, Y_k)$ defined in Eq. \eqref{eq:LMS_reform} has the following properties:
        \begin{enumerate}[(1)]
            \item $\|F(x_k, Y_k)\|_2\leq A_1(Y_k)\|x_k-x^*\|_2+B_1(Y_k)$, where $\hat{A}_1^2(y_0)= L^2D^4(\|\phi_0\|_2^4+\sigma_w^{(4)})/(1-\rho)^4$ and $\hat{B}_1^2(y_0)= D^2\sigma_v^{(2)}\sqrt{(\|\phi_0\|_2^4+\sigma_w^{(4)})}/(1-\rho)^2$.
            \item Define $\bar{F}(x)= \E_{\tilde{Y}\sim \mu_Y}[F(x, \tilde{Y})]$. Then, $\bar{F}(x)$ exists and is given by $\bar{F}(x)=\E_{\tilde{Y}\sim \mu_Y}\left[\tilde{\phi}(g(\tilde{\phi}^Tx)-g(\tilde{\phi}^Tx^*))\right]$. Furthermore, let $\lambda_{min}^{\phi}$ be the smallest eigenvalue of $\Sigma^*_{\phi}$, then $\langle x-x^*, \bar{F}(x)\rangle<-\mu\lambda_{min}^{\phi}\|x-x^*\|_2^2$.
            \end{enumerate}
            \item There exists a solution to the Poisson equation \eqref{eq:Poisson_eq} for the Markov chain $\mathcal{M}_Y$ which satisfies Assumption \ref{assump:Poisson_eq} with $\hat{A}_2^2(y_0)=32L^2D^8\left(\|\phi_0\|_2^4+\sqrt{\sigma_w^{(4)}}(\|\phi_0\|_2^2+ \sigma_w^{(4)} \right)/(1-\rho)^{10}$ and $\hat{B}_2^2(y_0)=\hat{B}_1^2(y_0)$.
    \end{enumerate}
    
\end{proposition}

\subsection{Finite Sample Bounds for LMS algorithm}
We again choose $\Phi(x-x^*)=\|x-x^*\|_2^2/2$ as our Lyapunov function. Then, using Proposition \ref{prop:LMS}, it is straightforward to see that $\eta=\mu\lambda_{min}^{\phi}$ and $L_s=1$. Now, we apply Theorem \ref{thm:main_thm;finite} to LMS to obtain the following finite-time sample complexity.

\begin{theorem}\label{thm:LMS}
    Consider the iterates $\{x_k\}_{k\geq 0}$ generated by iteration \eqref{eq:LMS} under Assumption \ref{assump:auto-reg}.  
    \begin{enumerate}[(a)]
        \item When $\alpha_k\equiv\alpha\leq  1$, then for all $k\geq 0$:
       \begin{align*}
           \E[\|x_{k+1}&-x^*\|_2^2]\leq \varphi_{L,0}\exp\left(-\mu\lambda_{min}^{\phi}\alpha k\right)+3\hat{C}_L(z_0, \phi_0)\alpha+\frac{6\hat{C}_L(z_0, \phi_0)\alpha}{\mu\lambda_{min}^{\phi}}.
       \end{align*}
       \item When $\xi=1$, $\alpha> 1/(\mu\lambda_{min}^{\phi})$ and $K\geq \max\{\alpha, 2\}$, then for all $k\geq 0$:
       \begin{align*}
           \hspace{-4mm}\E[\|x_{k+1}-x^*\|_2^2]&\leq \varphi_{L,0}\left(\frac{K}{k+K}\right)^{\mu\lambda_{min}^{\phi}\alpha}+\frac{\hat{C}_L(z_0, \phi_0)\alpha}{k+K}+\frac{4(6+4\mu\lambda_{min}^{\phi})\hat{C}_L(z_0, \phi_0)e\alpha^2}{\left(\mu\lambda_{min}^{\phi}\alpha-1\right)(k+K)}.
       \end{align*}
    \end{enumerate}
\end{theorem}
Refer to Appendix \ref{sec:thm_LMS} for the constants $\varphi_{L, 0}$ and $\hat{C}_L(z_0, \phi_0)$.
\begin{remark}
    Our bounds extend the performance analysis of the LMS algorithm in two notable ways compared to \cite{nagaraj2020least}. As highlighted before, the noise sequence $\{w_k\}_{k\geq 0}$ in \cite{nagaraj2020least} is assumed to be Gaussian, whereas our error bounds hold for any noise with finite-fourth moment. Furthermore, unlike their setting where $g(\cdot)$ is restricted to be the identity function, we allow $g(\cdot)$ to be non-linear satisfying some conditions. However, it must be noted that they do not require projecting the iterates onto a bounded set, as their bound is established conditioned on a high-probability event.
\end{remark}

\section{Control Problem in Reinforcement Learning: \texorpdfstring{$Q$} /-learning}\label{sec:Q-learning}
In this section, we will consider the control problem in discounted-reward RL for finite state space MDPs. We will carry the same notation for MDP parameters as introduced in Section \ref{sec:RL} for this section. The goal here is to maximize the expected cumulative discounted-reward. More formally, let $\gamma\in (0,1)$ be the discount factor and $\pi$ be a policy, define the state value function $V_{\pi}:\mathcal{S}\to \mathbb{R}$ as:
\begin{align*}
    V_{\pi}(s)=\E_{\pi}\left[\sum_{k=0}^\infty\gamma^k\mathcal{R}(S_k,A_k)|S_0=s\right]
\end{align*}
Then the objective of the control problem is to directly find an optimal policy $\pi^*$ such that $V_{\pi^*}(s)\geq V_{\pi}(s),~\forall s\in \mathcal{S}$ and any policy $\pi$.  It can be shown that, under mild conditions, such a policy always exists \cite{puterman2014markov}.

$Q$-learning \cite{watkins1992} is one of the most popular algorithms for finding the optimal policy by running the following iteration:
\begin{align}\label{eq:Q-learning}
    &Q_{k+1}(s,a)=Q_k(s,a)+\alpha_k\mathbbm{1}\{S_k=s, A_k=a\}\left(\mathcal{R}(s,a)+\gamma\max_{a'\in \mathcal{A}}Q_k(S_{k+1}, a')-Q_k(s,a)\right).
\end{align}
where $\{(S_k, A_k)\}_{k\geq 0}$ is a sample trajectory collected using a suitable behavior policy $\pi_b$ and $\mathbbm{1}\{\cdot\}$ is the indicator function. 
It can be shown that the Algorithm \eqref{eq:Q-learning} converges to $Q^*$ which is the unique fixed point of the Bellman optimality operator $\mathcal{B}(Q)$ defined by:
\begin{align*}
    \mathcal{B}(Q)=\mathcal{R}(s,a)+\gamma\sum_{s'\in \mathcal{S}}P(s'|s,a)\max_{a'\in \mathcal{A}}Q(s', a').
\end{align*}
Since $Q^*$ and the optimal policy $\pi^*$ satisfy the following relation: $\pi^*(\cdot|s)\in \argmax_{a\in\mathcal{A}} Q^*(s, a)$ \cite{bertsekas1996}, estimation of $Q^*$ directly related to finding the optimal policy.
We make the following standard assumption on the Markov chain generated by $\pi_b$.
\begin{assumption}\label{assump:Irreducible}
    The behavior policy $\pi_b$ satisfies $\pi_b(a|s)>0$ for all $(s,a)$ and the Markov chain $\mathcal{M}_S^{\pi_b}=\{S_k\}$ induced by $\pi_b$ is irreducible.
\end{assumption}
\begin{remark}
    The condition that $\pi(a|s)>0$ for all $(s,a)$ and the irreducibility of the induced Markov chain $\mathcal{M}_S^{\pi_b}$ is a standard assumption which ensures that all state action pairs are visited infinitely often \cite{bertsekas1996}. Moreover, since the MDP is finite, Assumption \ref{assump:Irreducible} implies that there exists a unique stationary distribution, which we denote as $\mu_b\in \Delta^{|\mathcal{S}|}$.
\end{remark}
\begin{remark}
     Recent works on the finite-time analysis of $Q$-learning often leverage the geometric mixing of Markov chain to handle Markovian noise \cite{li2020, qu2020, chen2023, zhang2024constant}. To ensure geometric mixing, these works commonly assume that $\mathcal{M}_S^{\pi_b}$ is also aperiodic, which is crucial to achieving this property. However, in our case, we do not require the aperiodicity assumption, since we utilize the solution to the Poisson equation, which exists under Assumption \ref{assump:Irreducible}. This flexibility is significant; often, one can design more effective behavior policies from a wider class of distributions to balance the trade-off between exploration and exploitation.
\end{remark}

\subsection{Properties of the \texorpdfstring{$Q$} /-learning Algorithm}
To apply Theorem \ref{thm:main_thm;finite} to $Q$-learning we first rearrange the iteration \eqref{eq:Q-learning} in the form of \eqref{eq:main_eq} and verify the assumptions.
\begin{align}
    Q_{k+1}(s,a)&=Q_k(s,a)+\alpha_k\mathbbm{1}\{S_k=s, A_k=a\}\left(\mathcal{R}(s,a)+\gamma\max_{a'\in \mathcal{A}}Q_k(S_{k+1}, a')-Q_k(s,a)\right)\nonumber\\
    &=Q_k(s,a)+\alpha_k\big(F(Q_k, (S_k, A_k))(s,a)+M_k(Q_k)(s,a)\big)
\end{align}
where
\begin{align*}
    F(Q, (S, A))(s,a)=\mathbbm{1}\{S=s, A=a\}\left(\mathcal{R}(s,a)+\gamma\sum_{s'\in \mathcal{S}}P(s'|s,a)\max_{a'\in \mathcal{A}}Q(s', a')-Q(s,a)\right),
\end{align*}

and
\begin{align*}
    M_k(Q)(s,a)=\gamma\mathbbm{1}\{S_k=s, A_k=a\}\left(\max_{a'\in \mathcal{A}}Q(S_{k+1}, a')-\sum_{s'\in \mathcal{S}}P(s'|S_k,A_k)\max_{a'\in \mathcal{A}}Q(s', a')\right).
\end{align*}

Furthermore, denote $Y_k=(S_k, A_k)$. It is easy to verify that the process $\mathcal{M}_Q=\{Y_k\}_{k\geq 0}$ is a Markov chain whose state space $\mathcal{Y}:=\mathcal{S}\times \mathcal{A}$ is finite. Then, $Q$-learning algorithm can be written as
\begin{align*}
    Q_{k+1}=Q_k+\alpha_k(F(Q_k, Y_k)+M_k(Q_k))
\end{align*}
Note that by Assumption \ref{assump:Irreducible}, the Markov chain $\mathcal{M}_Q$ is irreducible and therefore it has a unique stationary distribution given by $\mu_Q(s,a)=\mu_b(s)\pi_b(a|s)$ for all $(s,a)\in \mathcal{S}\times \mathcal{A}$. Let $y_0$ be some arbitrary state in $\mathcal{Y}$. For any $y\in \mathcal{Y}/\{y_0\}$, define $\tau_{y_0}^{y}$ as the expected hitting time of state $y_0$ starting from state $y$. Let $\tau_{y_0}$ denote $\max_{y\in \mathcal{Y}}\tau_{y_0}^{y}$, which is a well-defined quantity in finite state space \cite{meyn2012}. Furthermore, let $\Lambda\in \mathbb{R}^{|\mathcal{S}||\mathcal{A}|\times|\mathcal{S}||\mathcal{A}|}$ be diagonal matrix with $\{\mu_Q(s,a)\}$ as diagonal entries. Then, we have the following proposition whose proof can be found in Appendix \ref{appendix:Q-learning_prop}.
\begin{proposition}\label{prop:Q-learning}
Under Assumption \ref{assump:Irreducible}, the $Q$-learning algorithm satisfies the following:
    \begin{enumerate}[(a)]
        \item For any $Q, Q_1, Q_2\in \mathbb{R}^{|\mathcal{S}||\mathcal{A}|}$ and $y\in \mathcal{Y}$, the operator $F(Q, y)$ has the following properties:
        \begin{enumerate}[(1)]
            \item The operator $F(Q,y)$ satisfies: $\|F(Q, y)\|_{\infty}\leq 2\|Q-Q^*\|_{\infty}+\|Q^*\|_{\infty}$ and $\|F(Q_1, y)-F(Q_2, y)\|_{\infty}\leq 2\|Q-Q^*\|_{\infty}$.
            \item Define $\bar{F}(Q)= \E_{Y\sim \mu_Q}[F(Q, Y)]$. Then, $\bar{F}(Q)=\Lambda(\mathcal{B}(Q)-Q)$, where $\mathcal{B}(Q)$ is the Bellman optimality operator.
            \item The solution to Bellman equation, i.e., $Q^*$ is also the unique root of equation $\bar{F}(Q)=0$.
        \end{enumerate}
        \item There exists a solution to the Poisson equation \eqref{eq:Poisson_eq} for the Markov chain $\mathcal{M}_Q$ which satisfies Assumption \ref{assump:Poisson_eq} with $A_2=4\tau_{y_0}$ and $B_2=0$.
        \item The noise sequence $M_k(Q_k)$ is a martingale difference sequence and satisfies Assumption \ref{assump:linear_M} with constants $A_3=2$ and $B_3=2\|Q^*\|_{\infty}$.
    \end{enumerate}
\end{proposition}
Finally, we highlight the construction of a suitable Lyapunov function to study the convergence properties of the $Q$-learning algorithm.

\subsection{Finite Sample Bounds for \texorpdfstring{$Q$} /-Learning}
The authors in \cite{chen2020finite} showed that the Generalized Moreau Envelope can serve as a Lyapunov function for any operator which has the contraction property under a non-smooth norm. Specifically, consider the function $\Phi(x)=\min_{u\in\mathbb{R}^d}\left\{\frac{1}{2}\|u\|_c^2+\frac{1}{2\omega}\|x-u\|_p^2\right\}$ where $\omega>0$ and $p\geq 2$. The function $\Phi(\cdot)$ is known to be a smooth approximation of the function $\frac{1}{2}\|x\|_c^2$, with the smoothness parameter $\frac{p-1}{\omega}$. Further details on the properties of $\Phi(\cdot)$ can be found in \cite{beck2017first}. 

For $Q$-learning $\|\cdot\|_c=\|\cdot\|_{\infty}$, which by the properties of $\ell_p$ norms implies that $l_{cs}=1$ and $u_{cs}=(|\mathcal{S}||\mathcal{A}|)^{1/p}$. To verify Assumptions \ref{assump:Lyapunov_fn} in the context of $Q$-learning, we will need the following lemma whose proof can be found in \cite{chen2020finite}.
\begin{lemma}\label{lem:moreau_prop}
    Assign $p=2\log (|\mathcal{S}||\mathcal{A}|)$ and $\omega=\left(\frac{1}{2}+\frac{1}{2(1-(1-\gamma)\Lambda_{min})}\right)^2-1$, where $\Lambda_{min}=\min_{(s,a)} \{\mu_b(s)\pi_b(a|s)\}>0$ due to Assumption \ref{assump:Irreducible}. Then, the function $\Phi(x)$ satisfies the following properties:
    \begin{enumerate}[(a)]
        \item For all $Q\in \mathbb{R}^{|\mathcal{S}||\mathcal{A}|}$ we have $\langle \nabla \Phi(Q-Q^*), \bar{F}(Q)\rangle \leq -(1-\gamma)\Lambda_{min} \Phi(Q-Q^*)$.
        \item $\Phi(x)$ is convex, and $\frac{p-1}{\omega}$-smooth with respect to $\|\cdot\|_p$. That is $\Phi(y)\leq \Phi(s)+\langle \nabla \Phi(x), y-x\rangle+\frac{p-1}{2\omega}\|x-y\|_p^2$ for all $x,y\in \mathbb{R}^d$.
        \item Let $l=2(1+\omega/\sqrt{e})$ and $u=2(1+\omega)$. Then, we have $l\Phi(x)\leq \|x\|_c^2\leq u\Phi(x)$.
    \end{enumerate}
\end{lemma}

With all the assumptions satisfied, we can apply Theorem \ref{thm:main_thm;finite} to $Q$-learning. 
\begin{theorem}\label{thm:Q-learning}
    Consider the iterates $\{Q_k\}_{k\geq 0}$ generated by iteration \eqref{eq:Q-learning} under Assumption \ref{assump:Irreducible}.  
    \begin{enumerate}[(a)]
        \item When $\alpha_k\equiv\alpha\leq  \min\left\{1, \frac{\eta_Q}{A_Q(5+2\eta_Q)\varrho_{Q,1}}\right\}$, then for all $k\geq 0$:
       \begin{align*}
           \E[\|Q_{k+1}&-Q^*\|_{\infty}^2]\leq \varrho_{Q,0}\exp\left(\frac{-\eta_Q\alpha k}{2}\right)
           +\frac{58B_Q\varrho_{Q,1}\alpha}{\eta_Q}.
       \end{align*}
       \item When $\xi=1$, $\alpha> \frac{2}{\eta_Q}$ and $K\geq \max\{A_Q\alpha(5\alpha+8)\varrho_{Q,1}, 2\}$, then for all $k\geq 0$:
       \begin{align*}
           \hspace{-4mm}\E[\|Q_{k+1}-Q^*\|_{\infty}^2]&\leq \varrho_{Q,0}\left(\frac{K}{k+K}\right)^{\frac{\eta_Q\alpha}{2}}+\frac{2B_Q\varrho_{Q,1}\alpha}{k+K}+\frac{72B_Q\varrho_{Q,1}e\alpha^2}{\left(\frac{\eta_Q\alpha}{2}-1\right)(k+K)}.
       \end{align*}
    \end{enumerate}
\end{theorem}
The exact characterization of the constants can be found in Appendix \ref{appendix:Q-learning_thm}.
\begin{remark}
    Note that compared to \cite{chen2023}, our bounds are applicable for finite state Markov chains which do not mix, and hence removing the requirement on the behavior policy to induce an aperiodic chain. Due to this flexibility, one can often design more effective behavior policies from a wider class of distributions to balance the trade-off between exploration and exploitation. Furthermore, as proved in Corollary \ref{cor:Q_sample} for the case of constant step size, the above sample complexity immediately implies $\mathcal{O}\left(\epsilon^{-2}\log (1/\epsilon)\right)\mathcal{O}\left((1-\gamma)^{-5}\right)\mathcal{O}\left(\Lambda_{min}^{-3}\right)$, thus removing any poly-logarithmic factors with respect to $\Lambda_{min}$ or $1/(1-\gamma)$. For decreasing step-size, we also eliminate the $\log k$ factors from the upper bound.
\end{remark}

\section{Application in Optimization: Cyclic Block Coordinate Descent}\label{sec:opt}
\vspace{-1mm}

Consider an optimization problem $\min_{x\in\mathbb{R}^d} f(x)$ where the objective function $f(x)$ is $\mu$-strongly convex and $L$-smooth. Denote $x^*$ as the unique minimizer of $f(x)$. We assume that any vector $x$ can be partitioned into $p$ blocks as follows:
\begin{align*}
    x=(x(1), x(2), \hdots, x(p)),
\end{align*}
where $x(i)\in \mathbb{R}^{d_i}$ with $d_i\geq 1$ for all $1\leq i\leq p$ and satisfying $\sum_{i=1}^pd_i=d$. Furthermore, $\nabla_{i} f(x)$ denotes the partial derivatives with respect to the $i$-th block. Suppose that we have access to the partial gradients only through a noisy oracle which for any $x\in\mathbb{R}^d$ and block $i$ returns $\nabla_{i} f(x)+w$. Here $w$ represents the noise with appropriate dimension  which satisfies the following assumption.
\begin{assumption}\label{assump:noise_opt}
    Let $\mathcal{F}_k$ be the $\sigma$-field generated by $\{x_i, w_i\}_{0\leq i\leq k-1}\cup\{x_k\}$. Then, there exists constants $C_1, C_2\geq 0$ such that for all $k\geq 0$: (a) $\E[w_k|\mathcal{F}_k]=0$, (b) $\|w_k\|_2\leq C_1\|x_k-x^*\|_2+C_2$.
\end{assumption}
Assumption \ref{assump:noise_opt} is a standard assumption in optimization and basically implies that $w_k$ is a martingale difference sequence with respect to $\mathcal{F}_k$ and grows linearly with the iterates. Then, we have Algorithm \ref{algo:opt} to estimate $x^*$.
\begin{algorithm2e}[!ht]
		Initialize $x_0\in\mathbb{R}^d$, and step-size $\{\alpha_k\}_{k\geq 0}$.\\
		\For{$k=0,1,\dots$}{
		   Set $i(k)=k \Mod p +1$\\
              $x_{k+1}(j)=\begin{cases}
                  x_k(j)+\alpha_k(-\nabla_{j} f(x_k)+w_k),~\text{if $j=i(k)$}\\
                  x_k(j),~\text{otherwise}
              \end{cases}
                        $
        }
 \caption{Stochastic Cyclic Block Coordinate Descent (SCBCD)} 
\label{algo:opt}
\end{algorithm2e} 

Without loss of generality, we assume that at $k=0$ we update the first block. At each time step $k$, we cyclically update a block, where the block index $i(k)$ is determined through the modulo function. The oracle provides a noise gradient $-\nabla_{j} f(x_k)+w_k$ and the block corresponding to $i(k)$ gets updated while the rest of the blocks remain unchanged.

\vspace{-2mm}
\subsection{Properties of SCBCD}
\vspace{-2mm}
To fit Algorithm \ref{algo:opt} in the framework of \eqref{eq:main_rec}, we will set up some notation. Define the matrices $U_i\in \mathbb{R}^{d\times d_i}, ~1\leq i\leq p$ that satisfy
\begin{align*}
    \left(U_1, U_2,\dots, U_p\right)=I_d.
\end{align*}
Note that $x(i)=U_i^Tx$ for any vector $x\in \mathbb{R}^d$ and similarly the partial derivatives with respect to the $i$-th block can be written as $\nabla_{i} f(x_k)=U_{i}^T\nabla f(x_k)$. 
Thus, we rewrite the update equation as follows:
\begin{align}\label{eq:opt}
    x_{k+1}&=x_k+\alpha_k(-U_{i(k)}\nabla_{i(k)} f(x_k)+U_{i(k)}w_k)\nonumber\\
    &=x_k+\alpha_k(F(x_k, i(k))+M_k)
\end{align}
where $F(x_k, i(k))=-U_{i(k)}\nabla_{i(k)} f(x_k)$ and $M_k=U_{i(k)}w_k$. Observe that $\mathcal{M}_{U}=\{i(k)\}_{k\geq 0}$ can be viewed as a periodic Markov chain defined on the state space $\mathcal{S}=\{1,2,\dots,p\}$ with transition probabilities given as $P(i \Mod p+1|i)=1$, $~\forall i\in\mathcal{S}$. Furthermore, it is easy to verify that $\mu(i)=1/p,~\forall i\in \mathcal{S}$ is the unique stationary distribution for this Markov chain. This implies $\E_{i\sim \mu}\left[U_{i}\nabla_i f(x)\right]=\nabla f(x)/p$. Thus, solving for $\nabla f(x)=0$ is equivalent to finding the root of $\E_{i\sim \mu}\left[U_{i}\nabla_i f(x)\right]=0$. It is now easy to verify from Eq. \eqref{eq:opt} that all the Assumptions \ref{assump:linear_F}-\ref{assump:Poisson_eq} are satisfied as summarized in the following proposition. We provide the proof in Appendix \ref{appendix:opt}
\begin{proposition}\label{prop:opt}
The SCBCD algorithm has the following properties:
    \begin{enumerate}[(a)]
        \item The operator $F(x, i)$, satisfies:  $\|F(x, i)\|_2\leq L\|x-x^*\|_2,~\forall i\in \mathcal{S}$.
        \item There exists a solution to the Poisson equation \eqref{eq:Poisson_eq} for the Markov chain $\mathcal{M}_U$ which satisfies Assumption \ref{assump:Poisson_eq} with $A_2=\max\{L, 1\}$ and $B_2=0$.
        \item The noise sequence $M_k$ is a martingale difference sequence and satisfies: $\|M_k\|_2\leq C_1\|x-x^*\|_2+C_2$.
    \end{enumerate}
\end{proposition}

\subsection{Finite Sample Bounds for SCBCD}
We choose $\Phi(x-x^*)=\|x-x^*\|_2^2/2$ as our Lyapunov function. This immediately implies the properties of $\Phi(x-x^*)$ in Assumption \ref{assump:Lyapunov_fn}. In addition, $\eta=\mu/p$ and $L_s=1$ by smoothness and strong convexity of $f(x)$. We apply Theorem \ref{thm:main_thm;finite} to SCBCD to obtain the following finite-time sample complexity.
\begin{theorem}\label{thm:opt}
    Consider the iterates $\{x_k\}_{k\geq 0}$ generated by Algorithm \ref{algo:opt} under Assumption \ref{assump:noise_opt}. 
    \begin{enumerate}[(a)]
        \item When $\alpha_k\equiv\alpha\leq \min\left\{1, \frac{\mu}{A_G(5p+2\mu)\varrho_{G,1}}\right\}$, then for all $k\geq 0$:
       \begin{align*}
           \hspace{-1mm}\E[\|x_{k+1}-x^*\|_2^2]&\leq \varrho_{G,0}\exp\left(\frac{-\mu\alpha k}{2p}\right)+18B_{G}\varrho_{G,1}\alpha+\frac{40pB_{G}\varrho_{G,1}\alpha}{\mu}.
       \end{align*}
       \item When $\xi=1$, $\alpha> \frac{2p}{\mu}$ and $K\geq \max\{A_G\alpha(5\alpha+8)\varrho_{G,1}, 2\}$, then for all $k\geq 0$:
       \begin{align*}
           \hspace{-1mm}\E[\|x_{k+1}-x^*\|_2^2]&\leq \varrho_{G,0}\left(\frac{K}{k+K}\right)^{\frac{\mu\alpha}{p}}+\frac{2B_{G}\varrho_{G,1}\alpha}{k+K}+\frac{16B_{G}\left(5p+4\mu\right)\varrho_{G,1}e\alpha^2}{(\mu\alpha-2p)(k+K)}.
       \end{align*}
    \end{enumerate}
\end{theorem}
For the constants $\{\varrho_{G,i}\}_i$ $A_G,$ and $B_G$ refer to Appendix \ref{appendix:opt_thm}.
\begin{remark}
    In the noisy case, we obtain the $\mathcal{O}\left(1/k\right)$ rate of convergence similar to the randomized BCD in \cite{lan2020first}. Moreover, setting $B_G=0$ in the noiseless case, one obtains a geometric rate of convergence with a sample complexity of $\mathcal{O}\left(p^2L^3\log\left(1/\epsilon\right)/\mu^2\right)$. In the most general setting, our bound is optimal with respect to $p$ as shown in \cite{sun2021}. However, we remark that the dependence on the condition number $L/\mu$ is sub-optimal due to universal framework of our theorem. Nevertheless, one can improve upon the constants by using $f(x)-f(x^*)$ as the Lyapunov function and refining our analysis with the additional structure.
\end{remark}

\section{Conclusion}
In this work, we established a general-purpose theorem to obtain performance bounds for a general class of non-linear SA corrupted with unbounded Markovian noise. We handled the Markovian noise by using the solution of Poisson's equation to decompose it into a martingale difference term and some other higher order manageable terms. This enabled us to extend the convergence analysis of SA with martingale difference or i.i.d. noise to Markovian noise. To illustrate the power of our theorem, we studied four different settings: TD learning with eligibility traces for policy evaluation in infinite state space MDPs, generalized linear regression with Markovian data, $Q$-Learning for discounted-reward RL in finite state settings, and Stochastic CBCD for stochastic distributed optimization.

Some interesting future directions are proving some concentration results without projecting the iterates for SA with unbounded Markovian noise. In CBCD, potential natural extensions could involve using the periodic Markov chain perspective to study non-smooth functions, exploring SCBCD with block dependent step-size, and reducing $p$ dependence in specialized cases. 

\begin{acknowledgment}
    We thank Zaiwei Chen for pointing to us the connection between CBCD and Markovian noise that is periodic. We also thank Prof. Ashwin Pananjady for insightful discussion and providing useful comments on the counterexample for diverging mean square error.
\end{acknowledgment}

\bibliographystyle{ieeetr}
\bibliography{refs}

\clearpage
\appendix

\section{Proof of the Main Theorems \ref{thm:main_thm;bounded} and \ref{thm:main_thm;finite}}\label{appendix:proof}
First, we set up some notations for characterizing all the constants in the theorems.

\textbf{Common Notation:} Since we are working with finite-dimensional space $\mathbb{R}^d$, there exists positive constants such that $l_{cs}\|x\|_c\leq \|x\|_s\leq u_{cs}\|x\|_c$ and $l_{2s}\|x\|_2\leq \|x\|_s\leq u_{2s}\|x\|_2$. Denote $\|\cdot\|_{s^*}$ as the dual norm of $\|\cdot\|_s$ and $\kappa:=\left(\frac{\xi}{\alpha}+\eta\right)$.

\textbf{Notation for Theorem \ref{thm:main_thm;bounded}:} Let $\max_{x\in\mathcal{X}}\|x\|_c=M/2$, when $\mathcal{X}$ is an $\ell_2$ ball such that $x^*\in \mathcal{X}$. Then, we define the following constants for Theorem \ref{thm:main_thm;bounded}.
$$
    \hat{A}(y_0)=\hat{A}_1^2(y_0)+\hat{A}_2^2(y_0)+A_3^2; ~\hat{B}(y_0)=\hat{B}_1^2(y_0)+\hat{B}_2^2(y_0)+B_3^2;~ \hat{C}(y_0)=\hat{A}(y_0)M^2+\hat{B}(y_0);
$$
$$
    \varphi_1=\frac{uL_su_{2s}u^2_{cs}}{l_{2s}};~\varphi_0=\frac{u}{l}\|x_0-x^*\|_c^2+2\varphi_1\hat{C}(y_0).
$$

\textbf{Notation for Theorem \ref{thm:main_thm;finite}:} When the state space $\mathcal{Y}$ is bounded, the we define the following constants for Theorem \ref{thm:main_thm;finite}.
$$A=(A_1+A_3+1)^2;~B=\left(B_1+B_3+\frac{B_2}{A_2}\right)^2;$$
$$\varrho_1=uL_su_{cs}^2A_2;~ \varrho_0=\frac{2u(1+2A\varrho_1)}{l}\|x_0-x^*\|_c^2+4B\varrho_1.$$

\begin{theorem}
   Suppose that we run the Markov chain with initial state $y_0$. When the state space $\mathcal{Y}$ is unbounded and the set $\mathcal{X}$ is an $\ell_2$-ball of radius $R/2$, then under the Assumptions \ref{assump:linear_F}-\ref{assump:step-size}, $\{x_k\}_{k\geq 0}$ in the iterations \eqref{eq:main_eq} satisfy the following:
   \begin{enumerate}[(a)]
       \item When $\alpha_k\equiv\alpha\leq 1$, then for all $k\geq 0$:
       \begin{align*}
           \E[\|x_{k+1}-x^*\|_c^2]&\leq \varphi_0\exp\left(-\eta\alpha k\right)+3\varphi_1\hat{C}(y_0)\alpha+\frac{6\varphi_1\hat{C}(y_0)\alpha}{\eta}.
       \end{align*}
       \item When $\xi=1$, $\alpha> \frac{1}{\eta}$ and $K\geq \max\{\alpha, 2\}$, then for all $k\geq 0$:
       \begin{align*}
           E[\|x_{k+1}-x^*\|_c^2]&\leq \varphi_0\left(\frac{K}{k+K}\right)^{\eta\alpha}+\frac{\varphi_1\hat{C}(y_0)\alpha}{k+K}+\frac{4(6+4\eta)\varphi_1\hat{C}(y_0)e\alpha^2}{\left(\frac{\eta\alpha}{2}-1\right)(k+K)}.
       \end{align*}
         \item When $\xi<1$, $\alpha>0$ and $K\geq\max\{\alpha^{1/\xi},2\}$, then for all $k\geq 0$:
        \begin{align*}
            &\E[\|x_{k+1}-x^*\|_c^2]\leq\varphi_0\exp\Bigg(\frac{-\eta\alpha}{(1-\xi)}\left[(k+K)^{1-\xi}-K^{1-\xi}\right]\Bigg)+\frac{\varphi_1\hat{C}(y_0)\alpha}{(k+K)^\xi}+\frac{2(6+4\kappa)\varphi_1\hat{C}(y_0)\alpha}{\eta(k+K)^\xi}.
        \end{align*}
   \end{enumerate}
\end{theorem}

\begin{theorem}
When the state space $\mathcal{Y}$ is compact and the set $\mathcal{X}\equiv\mathbb{R}^d$, then under the Assumptions \ref{assump:linear_F}-\ref{assump:step-size}, $\{x_k\}_{k\geq 0}$ in the iterations \eqref{eq:main_eq} satisfy the following:
\begin{enumerate}[(a)]
       \item When $\alpha_k\equiv\alpha\leq \min\left\{1, \frac{\eta}{A(5+2\eta)\varrho_1}\right\}$, then for all $k\geq 0$:
       \begin{align*}
           \E[\|x_{k+1}-x^*\|_c^2]&\leq \varrho_0\exp\left(\frac{-\eta\alpha k}{2}\right)+18B\varrho_1\alpha+\frac{40B\varrho_1\alpha}{\eta}.
       \end{align*}
       \item When $\xi=1$, $\alpha> \frac{2}{\eta}$ and $K\geq \max\{A\alpha(5\alpha+8)\varrho_1, 2\}$, then for all $k\geq 0$:
       \begin{align*}
           E[\|x_{k+1}-x^*\|_c^2]&\leq \varrho_0\left(\frac{K}{k+K}\right)^{\frac{\eta\alpha}{2}}+\frac{2B\varrho_1\alpha}{k+K}+\frac{8B(5+4\eta)\varrho_1e\alpha^2}{\left(\frac{\eta\alpha}{2}-1\right)\left(k+K\right)}.
       \end{align*}
        \item When $\xi<1$, $\alpha>0$ and $K\geq \max\left\{\left(\frac{2A\alpha(5+2\kappa)\varrho_1}{\eta}\right)^{1/\xi}, 2\right\}$, then for all $k\geq 0$:
        \begin{align*}
            \E[\|x_{k+1}-x^*\|_c^2]&\leq\varrho_0\exp\Bigg(\frac{-\eta\alpha}{2(1-\xi)}\Big[(k+K)^{1-\xi}-K^{1-\xi}\Big]\Bigg)+\frac{2B\varrho_1\alpha}{(k+K)^\xi}+\frac{8B(5+2\kappa)\varrho_1\alpha}{\eta(k+K)^\xi}.
        \end{align*}
   \end{enumerate}
\end{theorem}
Before starting the proof of the theorems, we have the following lemma which decomposes the Lyapunov function at time $k+1$ using its properties in Assumption \ref{assump:Lyapunov_fn} and the recursion \eqref{eq:main_rec}, thereby establishing a one-step recursive relation. Define the following terms:
\begin{align*}
    T_{1,1}&=\alpha_k\langle\nabla \Phi(x_{k+1}-x^*)-\nabla \Phi(x_{k}-x^*), V_{x_k}(Y_{k+1}))\rangle,\\
    T_{1,2}&=\alpha_k\langle\nabla \Phi(x_{k+1}-x^*), V_{x_{k+1}}(Y_{k+1})-V_{x_k}(Y_{k+1}))\rangle,\\
    T_{2}&=\frac{L_s\alpha_k^2}{2}\|F(x_k, Y_k)+M_k\|_s^2,\\
    d_k&=\langle\nabla \Phi(x_k-x^*), V_{x_k}(Y_{k})\rangle.
\end{align*}

\begin{lemma}\label{lem:one_step_rec}
Under the Assumptions \ref{assump:linear_F}-\ref{assump:step-size}, we have the following one-step recursive relation:
    \begin{align}\label{eq:main_rec_to_solve}
        \E[\Phi(x_{k+1}-x^*)]&\leq(1-\eta\alpha_k)\E[\Phi(x_k-x^*)]+\alpha_k(\E[d_k]-\E[d_{k+1}])+\E[T_{1,1}]+\E[T_{1,2}]+\E[T_2].
    \end{align}
\end{lemma}
\begin{proof}
    Using the property \eqref{assump:non-expansive} of the Lyapunov function and the iteration \eqref{eq:main_eq}, we have
    \begin{align}
        \Phi(x_{k+1}-x^*)=&\Phi(\Pi_{\mathcal{X}}\left(x_k+\alpha_k(F(x_k, Y_k)+M_k)\right)-x^*)\nonumber\\
        \leq&\Phi(x_k+\alpha_k(F(x_k, Y_k)+M_k)-x^*)\nonumber\\
        \leq& \Phi(x_k-x^*)+\langle\nabla \Phi(x_k-x^*), \alpha_k(F(x_k, Y_k)+M_k)\rangle+\frac{L_s}{2}\|\alpha_k(F(x_k, Y_k)+M_k)\|_s^2\nonumber\\
        =&\Phi(x_k-x^*)+\alpha_k\langle\nabla \Phi(x_k-x^*), \bar{F}(x_k)\rangle+\underbrace{\alpha_k\langle\nabla \Phi(x_k-x^*), F(x_k, Y_k)-\bar{F}(x)+M_k\rangle}_{T_1}\nonumber\\
        &~~~+\underbrace{\frac{L_s\alpha_k^2}{2}\|F(x_k, Y_k)+M_k\|_s^2}_{T_2}\label{eq:phi_rec}.
    \end{align}
    We begin by re-organizing $T_1$ with the help of solution to the Poisson's equation as follows:
    \begin{align*}
        T_1&=\alpha_k\langle\nabla \Phi(x_k-x^*), V_{x_k}(Y_k)-\E_{Y_k}[V_{x_k}(Y_{k+1})]+M_k\rangle\\
        &=\alpha_k\langle\nabla \Phi(x_k-x^*), V_{x_k}(Y_{k+1})-\E_{Y_k}[V_{x_k}(Y_{k+1})]+M_k\rangle+\alpha_k\langle\nabla \Phi(x_k-x^*), V_{x_k}(Y_{k})-V_{x_k}(Y_{k+1})\rangle.
    \end{align*} 
    Observe that the first term is a martingale difference sequence with respect to the $\sigma$-field $\mathcal{F}_k$. We rewrite the second term as follows:
    \begin{align*}
        \alpha_k\langle\nabla \Phi(x_k-x^*), V_{x_k}(Y_{k})-V_{x_k}(Y_{k+1})\rangle&=\alpha_k(d_k-d_{k+1})+\underbrace{\alpha_k\langle\nabla \Phi(x_{k+1}-x^*)-\nabla \Phi(x_{k}-x^*), V_{x_k}(Y_{k+1}))\rangle}_{T_{1,1}}\\
        &+\underbrace{\alpha_k\langle\nabla \Phi(x_{k+1}-x^*), V_{x_{k+1}}(Y_{k+1})-V_{x_k}(Y_{k+1}))\rangle}_{T_{1,2}}.
    \end{align*}
    Taking expectation conditioned on $\mathcal{F}_k$ on both sides of Eq. \eqref{eq:phi_rec}, we get
    \begin{align*}
        \E[\Phi(x_{k+1}-x^*)|\mathcal{F}_k]&\leq\Phi(x_k-x^*)+\alpha_k\langle\nabla \Phi(x_k-x^*), \bar{F}(x_k)\rangle+\alpha_k\E[(d_k-d_{k+1})|\mathcal{F}_k]\\
        &+\E[T_{1,1}|\mathcal{F}_k]+\E[T_{1,2}|\mathcal{F}_k]+\E[T_2|\mathcal{F}_k].
    \end{align*}
    Using Tower property and Eq. \eqref{assump:neg_drift}, we have
    \begin{align*}
        \E[\Phi(x_{k+1}-x^*)]&\leq(1-\eta\alpha_k)\E[\Phi(x_k-x^*)]+\alpha_k(\E[d_k]-\E[d_{k+1}])+\E[T_{1,1}]+\E[T_{1,2}]+\E[T_2].
    \end{align*}
\end{proof}
Now we can proceed by bounding each of the terms in accordance with the specific settings.

\begin{proof}[Proof for Theorem \ref{thm:main_thm;bounded}]
Using Eq. \eqref{lem:grad_markov_noise;bounded} in Lemma \ref{lem:grad_markov_noise} and Eq. \eqref{lem:poisson_markov_noise;bounded} in Lemma \eqref{lem:poisson_markov_noise} for $\E[T_{1,1}]$ and $\E[T_{1,2}]$ respectively, we get
    \begin{align*}
        \E[T_1]&\leq \frac{4\alpha_k^2\varphi_1}{u}.
    \end{align*}
    We use Eq. \eqref{lem:T_2noise;bounded} in Lemma \ref{lem:T_2noise} to get a bound on $\E[T_2]$,
    \begin{align*}
        \E[T_2]\leq \frac{2\alpha_k^2\varphi_1}{u}.
    \end{align*}
    Furthermore, to upper bound the second term in Eq. \eqref{eq:main_rec_to_solve}, we use Eq. \eqref{lem:telscoping_rearrange;finite} in Lemma \ref{lem:telscoping_rearrange}. Combining all the bounds, we get
    \begin{align*}
        \E[\Phi(x_{k+1}-x^*)]&\leq (1-\eta\alpha_k)\E[\Phi(x_k-x^*)]+\left(1-\eta\alpha_k\right)\alpha_{k-1}\E[d_k]-\alpha_k\E[d_{k+1}]+\frac{\alpha_k^2(6+2\kappa)\varphi_1\hat{C}(y_0)}{u}\\
        &\leq \E[\Phi(x_0-x^*)]\prod_{n=0}^k\left(1-\eta\alpha_n\right)+\alpha_{-1}\E[d_0]\prod_{n=0}^k\left(1-\eta\alpha_n\right)-\alpha_k\E[d_{k+1}]\\
        &~~~~+\frac{(6+2\kappa)\varphi_1\hat{C}(y_0)}{u}\sum_{n=0}^k\alpha_n^2\prod_{\ell=n+1}^k\left(1-\eta\alpha_\ell\right).
    \end{align*}
     Since $K\geq 2$, $\alpha_{-1}$ is well-defined and is bounded above as $\alpha_{-1}\leq 2\alpha_0\leq 2$. Furthermore, using Eq. \eqref{lem:telescoping_term-bound;bounded} in Lemma \ref{lem:telescoping_term-bound} for second and third term, we have
     \begin{align*}
        \E[\Phi(x_{k+1}-x^*)]&\leq \left(\E[\Phi(x_0-x^*)]+\frac{2\varphi_1}{u}\right)\prod_{n=0}^k\left(1-\eta\alpha_n\right)+\frac{\alpha_k\varphi_1}{u}+\frac{(6+2\kappa)\varphi_1\hat{C}(y_0)}{u}\sum_{n=0}^k\alpha_n^2\prod_{\ell=n+1}^k\left(1-\eta\alpha_\ell\right).
    \end{align*}
    Using Eq. \eqref{assump:equivalence_rel} in Assumption \ref{assump:Lyapunov_fn}, we get
    \begin{align*}
        \E[\|x_{k+1}-x^*\|_c^2]&\leq \left(\frac{u}{l}\|x_0-x^*\|_c^2+2\varphi_1\right)\prod_{n=0}^k\left(1-\eta\alpha_n\right)+\alpha_k\varphi_1+(6+2\kappa)\varphi_1\hat{C}(y_0)\sum_{n=0}^k\alpha_n^2\prod_{\ell=n+1}^k\left(1-\eta\alpha_\ell\right).
    \end{align*}
    The finite time bounds for all the choices of step sizes can be obtained using the above bound by a straightforward application of Corollary 2.1.1 and Corollary 2.1.2 in \cite{chen2020finite}.
\end{proof}

\begin{proof}[Proof for Theorem \ref{thm:main_thm;finite}]
Using Eq. \eqref{lem:grad_markov_noise;finite} Lemma \ref{lem:grad_markov_noise} and Eq. in \eqref{lem:poisson_markov_noise;finite} in Lemma  \ref{lem:poisson_markov_noise} for $\E[T_{1,1}]$ and $\E[T_{1,2}]$ respectively, we get
    \begin{align*}
        \E[T_1]&\leq \frac{4\alpha_k^2\varrho_1}{u}\left(uA\E[\Phi(x_k-x^*)]+B\right).
    \end{align*}
    For $\E[T_2]$, we use \eqref{lem:T_2noise;bounded} in Lemma \ref{lem:T_2noise} to get,
    \begin{align*}
        \E[T_2]\leq \frac{\alpha_k^2\varrho_1}{u}\left(uA\E[\Phi(x_k-x^*)]+B\right).
    \end{align*}
    
    Furthermore, to upper bound the second term in Eq. \eqref{eq:main_rec_to_solve}, we use Eq. \eqref{lem:telscoping_rearrange;finite} in Lemma \ref{lem:telscoping_rearrange}. Using all the bounds, we get
    \begin{align*}
        \E[\Phi(x_{k+1}-x^*)]&\leq(1-\eta\alpha_k)\E[\Phi(x_k-x^*)]+\left(1-\frac{\eta\alpha_k}{2}\right)\alpha_{k-1}\E[d_k]-\alpha_k\E[d_{k+1}]\\
        &~~~~+\frac{\alpha_k^2(5+2\kappa)\varrho_1}{u}\left(uA\E[\Phi(x_k-x^*)]+B\right).
    \end{align*} 
    Assume that $\alpha_k$ is small enough such that we have
    \begin{align*}
        \frac{\eta\alpha_k}{2}\geq A(5+2\kappa)\varrho_1\alpha_k^2.
    \end{align*}
    Using the above condition, we get
    \begin{align*}
        \E[\Phi(x_{k+1}-x^*)]&\leq\left(1-\frac{\eta\alpha_k}{2}\right)\E[\Phi(x_k-x^*)]+\left(1-\frac{\eta\alpha_k}{2}\right)\alpha_{k-1}\E[d_k]-\alpha_k\E[d_{k+1}]+\alpha_k^2\frac{B(5+2\kappa)\varrho_1}{u}.
    \end{align*}
    Recursively writing the above inequality, we get
    \begin{align*}
        \E[\Phi(x_{k+1}-x^*)]&\leq \E[\Phi(x_0-x^*)]\prod_{n=0}^k\left(1-\frac{\eta\alpha_n}{2}\right)+\alpha_{-1}\E[d_0]\prod_{n=0}^k\left(1-\frac{\eta\alpha_n}{2}\right)-\alpha_k\E[d_{k+1}]\\
        &~~~~+\frac{B(5+2\kappa)\varrho_1}{u}\sum_{n=0}^k\alpha_n^2\prod_{\ell=n+1}^k\left(1-\frac{\eta\alpha_\ell}{2}\right).
    \end{align*}
    Again, since $K\geq 2$, $\alpha_{-1}$ is well-defined and is bounded above as $\alpha_{-1}\leq 2\alpha_0\leq 2$. Furthermore, using Eq. \eqref{lem:telescoping_term-bound;finite} in Lemma \ref{lem:telescoping_term-bound} for the second and the third term, we have
    \begin{align*}
       \E[\Phi(x_{k+1}-x^*)]&\leq \left(\E[\Phi(x_0-x^*)]+\frac{2\varrho_1}{u}(uA\E[\Phi(x_0-x^*)]+B)\right)\prod_{n=0}^k\left(1-\frac{\eta\alpha_n}{2}\right)\\
        &~~~~+\frac{\alpha_k\varrho_1}{u}(uA\E[\Phi(x_{k+1}-x^*)]+B)+\frac{B(5+2\kappa)\varrho_1}{u}\sum_{n=0}^k\alpha_n^2\prod_{\ell=n+1}^k\left(1-\frac{\eta\alpha_\ell}{2}\right).
    \end{align*}
    Note that $5+2\kappa> \eta$. Thus, $\alpha_k\leq \frac{\eta}{2A(5+2\kappa)\varrho_1}$ implies that $\alpha_kA\varrho_1\leq 0.5,~\forall k\geq 0$. Thus, we have
    \begin{align*}
        \E[\Phi(x_{k+1}-x^*)]&\leq \left(2\left(1+2A\varrho_1\right)\E[\Phi(x_0-x^*)]+\frac{4B\varrho_1}{u}\right)\prod_{n=0}^k\left(1-\frac{\eta\alpha_n}{2}\right)\\
        &~~~~+\frac{2\alpha_kB\varrho_1}{u}++\frac{2B(5+2\kappa)\varrho_1}{u}\sum_{n=0}^k\alpha_n^2\prod_{\ell=n+1}^k\left(1-\frac{\eta\alpha_\ell}{2}\right).
    \end{align*}
    
    Using Eq. \eqref{assump:equivalence_rel} in Assumption \ref{assump:Lyapunov_fn}, we get
    \begin{align*}
        \E[\|x_{k+1}-x^*\|_c^2]&\leq \left(\frac{2u(1+2A\varrho_1)}{l}\|x_0-x^*\|_c^2+4B\varrho_1\right)\prod_{n=0}^k\left(1-\frac{\eta\alpha_n}{2}\right)\\
        &~~~~+2\alpha_kB\varrho_1+2B(5+2\kappa)\varrho_1\sum_{n=0}^k\alpha_n^2\prod_{\ell=n+1}^k\left(1-\frac{\eta\alpha_\ell}{2}\right).
    \end{align*}
    
    Again, the above bound immediately implies finite time bounds for various choices of step sizes by applying Corollary 2.1.1 and Corollary 2.1.2 in \cite{chen2020finite}.
    
\end{proof}
\section{Proof of the main lemmas used in theorems \ref{thm:main_thm;bounded} and \ref{thm:main_thm;finite}}\label{appendix:main_lemmas}

\begin{lemma}\label{lem:grad_markov_noise}
    Under the Assumptions \ref{assump:linear_F}-\ref{assump:step-size}, we have the following:
    \begin{enumerate}[(a)]
        \item When the set $\mathcal{X}$ is an $\ell_2$-ball of sufficiently large such that $x^*\in \mathcal{X}$, then
        \begin{align}\label{lem:grad_markov_noise;bounded}
            \E[T_{1,1}]& \leq \frac{2\alpha_k^2\varphi_1\hat{C}(y_0)}{u}.
        \end{align}
        \item When $\mathcal{Y}$ is bounded and $\mathcal{X}\equiv\mathbb{R}^d$, then
        \begin{align}\label{lem:grad_markov_noise;finite}
            \E[T_{1,1}]& \leq \frac{2\alpha_k^2\varrho_1}{u}\left(uA\E[\Phi(x_k-x^*)]+B\right).
        \end{align}
    \end{enumerate}
\end{lemma}
\begin{proof}
To handle $\E[T_{1,1}]$, we use Holder's inequality and the smoothness of $\Phi(\cdot)$ to get
\begin{align*}
    \E[T_{1,1}]&\leq \alpha_k\E[\|\nabla \Phi(x_{k+1}-x^*)-\nabla \Phi(x_{k}-x^*)\|_{s^*}\|V_{x_k}(Y_{k+1})\|_s]\\
    &\leq \alpha_kL_s\E[\|x_{k+1}-x_{k}\|_{s}\|V_{x_k}(Y_{k+1})\|_s]\\
    &\leq \alpha_kL_su_{cs}^2\E[\|x_{k+1}-x_{k}\|_c\|V_{x_k}(Y_{k+1})\|_c]\\
    &\leq \alpha_kL_su_{cs}^2\E[\|x_{k+1}-x_{k}\|_c(A_2(Y_{k+1})\|x_k-x^*\|_c+B_2(Y_{k+1}))].\tag{Eq. \eqref{assump:value_fn-prop} in Assumption \ref{assump:Poisson_eq}}
\end{align*}

\begin{enumerate}[(a)]
    \item Using Eq. \eqref{lem:one_step-lipschitz;bounded} in Lemma \ref{lem:one_step-lipschitz} and $\|x_k-x^*\|_c\leq M$, we get
    \begin{align*}
        \E[T_{1,1}]&\leq \alpha_k^2\frac{L_su_{2s}u^2_{cs}}{l_{2s}}\E[(A_1(Y_k)M+B_1(Y_k)+A_3M+B_3)(A_2(Y_{k+1})M+B_2(Y_{k+1}))]\\
        &\leq \alpha_k^2\frac{\varphi_1}{2u}\E[(A_1(Y_k)M+B_1(Y_k)+A_3M+B_3)^2+(A_2(Y_{k+1})M+B_2(Y_{k+1}))^2]\tag{$ab\leq \frac{a^2+b^2}{2}$}\\
        &\leq \alpha_k^2\frac{2\varphi_1}{u}\left(\E[A_1^2(Y_k)M^2+B_1^2(Y_k)+A_3^2M^2+B_3^2+A_2^2(Y_{k+1})M^2+B_2^2(Y_{k+1})\right).\tag{ $\left(\sum_{i=1}^na_i\right)^2\leq n\left(\sum_{i=1}^na_i^2\right)$}
    \end{align*}
    Finally, using part \eqref{assump:moments_bound} in Assumption \ref{assump:Poisson_eq}, we get
    \begin{align*}
        \E[T_{1,1}]&\leq \frac{2\varphi_1\alpha_k^2}{u}\hat{C}(y_0).
    \end{align*}
    \item From part \eqref{assump:moments_bound} in Assumption \ref{assump:Poisson_eq}, we get
    \begin{align*}
        \E[T_{1,1}]&\leq \alpha_k^2L_su_{cs}^2\E[((A_1+A_3)\|x_k-x^*\|_c+B_1+B_3)(A_2\|x_k-x^*\|_c+B_2)]\\
        &\leq \alpha_k^2L_su_{cs}^2A_2\E\left[((A_1+A_3)\|x_k-x^*\|_c+B_1+B_3)\left(\|x_k-x^*\|_c+\frac{B_2}{A_2}\right)\right]\\
        &\leq \alpha_k^2\frac{\varrho_1}{u}\E\left[\left((A_1+A_3+1)\|x_k-x^*\|_c+B_1+B_3+\frac{B_2}{A_2}\right)^2\right]\tag{$A_1,A_3\geq 0, A_1+A_3+1\geq 1$}\\
        &\leq \frac{2\alpha_k^2\varrho_1}{u}\E\left[(A\|x_k-x^*\|_c^2+B\right]\tag{$\left(a_1+a_2\right)^2\leq 2\left(a_1^2+a_2^2\right)$}\\
        &\leq \frac{2\alpha_k^2\varrho_1}{u}\left(uA\E[\Phi(x_k-x^*)]+B\right)\tag{Eq. \eqref{assump:equivalence_rel} in Assumptions \eqref{assump:Lyapunov_fn}}.
    \end{align*}
\end{enumerate}
\end{proof}

\begin{lemma}\label{lem:poisson_markov_noise}
    Under the Assumptions \ref{assump:linear_F}-\ref{assump:step-size}, we have the following:
    \begin{enumerate}[(a)]
        \item When the set $\mathcal{X}$ is an $\ell_2$-ball of sufficiently large such that $x^*\in \mathcal{X}$, then
        \begin{align}\label{lem:poisson_markov_noise;bounded}
            \E[T_{1,2}]& \leq \frac{2\alpha_k^2\varphi_1\hat{C}(y_0)}{u}.
        \end{align}
        \item When $\mathcal{Y}$ is bounded and $\mathcal{X}\equiv\mathbb{R}^d$, then
        \begin{align}\label{lem:poisson_markov_noise;finite}
            \E[T_{1,2}]& \leq \frac{2\alpha_k^2\varrho_1}{u}\left(uA\E[\Phi(x_k-x^*)]+B\right).
        \end{align}
    \end{enumerate}
\end{lemma}
\begin{proof}
Denote $s^*$ as the dual norm of $s$. To handle $T_{1,2}$ we use Holder's inequality to get
\begin{align*}
    \E[T_{1,2}]&\leq \alpha_k\E\left[\|\nabla \Phi(x_{k+1}-x^*)\|_{s^*}\|V_{x_{k+1}}(Y_{k+1})-V_{x_k}(Y_{k+1}))\|_s\right]\\
    &\leq \alpha_k\E\left[A_2(Y_{k+1})\|\nabla \Phi(x_{k+1}-x^*)\|_{s^*}\|x_{k+1}-x_k\|_s\right].\tag{Eq. \eqref{assump:value_fn-prop} in Assumption \ref{assump:Poisson_eq}}
\end{align*}
Since $\Phi(\cdot)$ is convex-differentiable and achieves its minima at $0$, $\nabla \Phi(0)=0$. Along with smoothness of the function, we get
\begin{align}\label{eq:derivative_bound}
    \|\nabla \Phi(x_{k+1}-x^*)-\nabla \Phi(0)\|_{s^*}\leq L_s\|x_{k+1}-x^*\|_s\nonumber\\
    \implies \|\nabla \Phi(x_{k+1}-x^*)\|_{s^*}\leq L_s\|x_{k+1}-x^*\|_s
\end{align}
\begin{align*}
    \E[T_{1,2}]&\leq  \alpha_kL_s\E\left[A_2(Y_{k+1})\|x_{k+1}-x^*\|_s\|x_{k+1}-x_k\|_s\right]\\
    &\leq \alpha_kL_su_{cs}\E\left[A_2(Y_{k+1})\|x_{k+1}-x^*\|_s\|x_{k+1}-x_k\|_c\right]
\end{align*}

\begin{enumerate}[(a)]
    \item Using Eq. \eqref{lem:one_step-lipschitz;bounded} in Lemma \ref{lem:one_step-lipschitz} and the fact that $\|x-x^*\|_c\leq M,~\forall~x\in\mathcal{X}$, we get
    \begin{align*}
        \E[T_{1,2}]&\leq \alpha_k^2L_s\frac{u_{2s}u_{cs}^2}{l_{2s}}\E[A_2(Y_{k+1})M(A_1(Y_k)M+B_1(Y_k)+A_3M+B_3)]\tag{$\|x_{k+1}-x^*\|_s\leq u_{cs}M$}\\
        &\leq \alpha_k^2\frac{\varphi_1}{2u}\E[A_2^2(Y_{k+1})M^2+(A_1(Y_k)M+B_1(Y_k)+A_3M+B_3)^2]\tag{$ab\leq \frac{a^2+b^2}{2}$}\\
        &\leq \frac{2\alpha_k^2\varphi_1}{u}\left(\E[A_1^2(Y_k)M^2+B_1^2(Y_k)+A_3^2M^2+B_3^2+A_2^2(Y_{k+1})M^2\right).\tag{$\left(\sum_{i=1}^na_i\right)^2\leq n\left(\sum_{i=1}^na_i^2\right)$}
    \end{align*}
    Finally, using part \eqref{assump:moments_bound} in Assumption \ref{assump:Poisson_eq}, we get
    \begin{align*}
        \E[T_{1,2}]&\leq \frac{2\alpha_k^2\varphi_1}{u}\left(\left(\hat{A}_1^2(y_0)+\hat{A}_2^2(y_0)+A_3^2\right)M^2+\hat{B}_1^2(y_0)+B_3^2\right)\\
        &= \frac{2\alpha_k^2\varphi_1\hat{C}(y_0)}{u}.
    \end{align*}
    \item Since in this case $\mathcal{X}\equiv \mathbb{R}^d$, we use Eq. \eqref{eq:main_eq}, to get
    \begin{align*}
        \|x_{k+1}-x^*\|_s&\leq \|x_k-x^*\|_s+\|x_{k+1}-x_k\|_s\\
        &\leq \|x_k-x^*\|_s+\alpha_ku_{cs}((A_1+A_3)\|x_k-x^*\|_c+B_1+B_3)\tag{Eq. \eqref{lem:one_step-lipschitz;finite} in Lemma \ref{lem:one_step-lipschitz}}\\
        &\leq u_{cs}\left((A_1+A_3+1)\|x_k-x^*\|_c+B_1+B_3)\right).\tag{Assuming $\alpha_k\leq 1$}
    \end{align*}
    
    Furthermore, from part \eqref{assump:moments_bound} and Eq. \eqref{lem:one_step-lipschitz;finite} in Lemma \ref{lem:one_step-lipschitz}, we get
    \begin{align*}
        \E[T_{1,2}]&\leq  \alpha_k^2u_{cs}^2L_sA_2\E\left[\left((A_1+A_3+1)\|x_k-x^*\|_c+B_1+B_3)\right)((A_1+A_3)\|x_k-x^*\|_c+B_1+B_3)\right]\\
        &\leq \frac{\varrho_1\alpha_k^2}{u}\E\left[(A_1+A_3+1)\|x_k-x^*\|_c+B_1+B_3)^2\right]\\
        &\leq \frac{2\varrho_1\alpha_k^2}{u}\E\left[(A_1+A_3+1)^2\|x_k-x^*\|^2_c+(B_1+B_3)^2\right]\tag{$\left(a_1+a_2\right)^2\leq 2\left(a_1^2+a_2^2\right)$}\\
        &\leq \frac{2\varrho_1\alpha_k^2}{u}\left(uA\E[\Phi(x_k-x^*)]+B\right).\tag{Eq. \eqref{assump:equivalence_rel} in Assumptions \eqref{assump:Lyapunov_fn}}
    \end{align*}
\end{enumerate}
\end{proof}

\begin{lemma}\label{lem:T_2noise}
    Under the Assumptions \ref{assump:linear_F}-\ref{assump:step-size}, we have the following:
    \begin{enumerate}[(a)]
        \item When the set $\mathcal{X}$ is an $\ell_2$-ball of sufficiently large such that $x^*\in \mathcal{X}$, then
        \begin{align}\label{lem:T_2noise;bounded}
            \E[T_2]& \leq \frac{2\alpha_k^2\varphi_1\hat{C}(y_0)}{u}.
        \end{align}
        \item When $\mathcal{Y}$ is bounded and $\mathcal{X}\equiv\mathbb{R}^d$, then
        \begin{align}\label{lem:T_2noise;finite}
            \E[T_2]& \leq \frac{\alpha_k^2\varrho_1}{u}\left(uA\E[\Phi(x_k-x^*)]+B\right).
        \end{align}
    \end{enumerate}
\end{lemma}
\begin{proof}
    \begin{enumerate}[(a)]
        \item 
        \begin{align*}
            E[T_2]&\leq \frac{\alpha_k^2L_su_{cs}^2}{2}\E[(\|F(x_k, Y_k)\|_c+\|M_k\|_c)^2]\\
            &\leq \frac{\alpha_k^2L_su_{cs}^2}{2}\E[(A_1(Y)\|x_k-x^*\|_c+B_1(Y_k)+A_3\|x_k-x^*\|_c+B_3)^2]\tag{Assumptions \ref{assump:linear_F} and \ref{assump:linear_M}}\\
            &\leq \frac{\alpha_k^2L_su_{cs}^2}{2}\E[(A_1(Y)M+B_1(Y_k)+A_3M+B_3)^2]\tag{$\|x-x^*\|_c\leq M,~~\forall x\in \mathcal{X}$}\\
            &\leq 2\alpha_k^2L_su_{cs}^2\E[A_1^2(Y)M^2+B_1^2(Y_k)+A_3^2M^2+B_3^2]\tag{$\left(\sum_{i=1}^na_i\right)^2\leq n\left(\sum_{i=1}^na_i^2\right)$}\\
            &\leq 2\alpha_k^2L_su_{cs}^2\left((\hat{A}_1^2(y_0)+A_3^2)M^2+\hat{B}_1^2(y_0)+B_3^2\right)\tag{Part \eqref{assump:moments_bound} in Assumption \ref{assump:Poisson_eq}}\\
            &\leq \frac{2\alpha_k^2\varphi_1\hat{C}(y_0)}{u}.\tag{$\frac{u_{2s}}{l_{2s}}\geq 1$ and $\hat{A}_2^2(y_0), \hat{B}_2^2(y_0)\geq 0$}
        \end{align*}
        \item 
        \begin{align*}
           \E[T_2]&\leq \frac{\alpha_k^2L_su_{cs}^2}{2}\E[(\|F(x_k, Y_k)\|_c+\|M_k\|_c)^2]\\
            &\leq \frac{\alpha_k^2L_su_{cs}^2}{2}\E[(A_1(Y)\|x_k-x^*\|_c+B_1(Y_k)+A_3\|x_k-x^*\|_c+B_3)^2]\tag{Assumptions \ref{assump:linear_F} and \ref{assump:linear_M}}\\ 
            &\leq \frac{\alpha_k^2L_su_{cs}^2}{2}\E[(A_1\|x_k-x^*\|_c+B_1+A_3\|x_k-x^*\|_c+B_3)^2]\tag{Part \eqref{assump:moments_bound} in Assumption \ref{assump:Poisson_eq}}\\
            &\leq \alpha_k^2L_su_{cs}^2\E[(A_1+A_3)^2\|x_k-x^*\|_c^2+(B_1+B_3)^2]\\
            &\leq \frac{\alpha_k^2\varrho_1}{u}\left(uA\E[\Phi(x_k-x^*)]+B\right)\tag{$A_2\geq 1$ and Eq. \eqref{assump:equivalence_rel} in Assumptions \eqref{assump:Lyapunov_fn}}.
        \end{align*}  
    \end{enumerate}
\end{proof}

\begin{lemma}\label{lem:telescoping_term-bound}
    Under the Assumptions \ref{assump:linear_F}-\ref{assump:step-size}, we have the following:
    \begin{enumerate}[(a)]
        \item When the set $\mathcal{X}$ is an $\ell_2$-ball of sufficiently large such that $x^*\in \mathcal{X}$, then
        \begin{align}\label{lem:telescoping_term-bound;bounded}
            \E[|d_k|]&\leq  \frac{\varphi_1\hat{C}(y_0)}{u}.
        \end{align}
        \item  When $\mathcal{Y}$ is bounded and $\mathcal{X}\equiv\mathbb{R}^d$, then
        \begin{align}\label{lem:telescoping_term-bound;finite}
            \E[|d_k|]& \leq \frac{\varrho_1}{u}(uA\E[\Phi(x_k-x^*)]+B).
        \end{align}
    \end{enumerate}
\end{lemma}
\begin{proof}
Using Holder's inequality, we have
    \begin{align*}
        \E[|d_k|]&\leq \E[\|\nabla \Phi(x_k-x^*)\|_{s^*}\|V_{x_k}(Y_{k})\|_{s}].
    \end{align*}
    Using the same argument as in Eq. \eqref{eq:derivative_bound}, we get
    \begin{align*}
        \E[|d_k|]&\leq L_s\E[\|x_k-x^*\|_s\|V_{x_k}(Y_{k})\|_s]\\
        &\leq L_su_{cs}^2\E[\|x_k-x^*\|_c\|V_{x_k}(Y_{k})\|_c]\\
        &\leq L_su_{cs}^2\E[\|x_k-x^*\|_c(A_2(Y_k)\|x_k-x^*\|_c+B_2(Y_k))].\tag{Using Eq. \eqref{assump:value_fn-prop} in Assumption \ref{assump:Poisson_eq}}
    \end{align*}
    \begin{enumerate}[(a)]
        \item Since $\|x_k-x^*\|_c\leq M$, we get
        \begin{align*}
            \E[|d_k|]&\leq L_su_{cs}^2\E[M(A_2(Y_k)M+B_2(Y_k))]\\
            &\leq \frac{L_su_{cs}^2}{2}\E[M^2+(A_2(Y_k)M+B_2(Y_k))^2]\\
            &\leq L_su_{cs}^2\E[M^2+A_2^2(Y_k)M^2+B_2^2(Y_k)]\tag{$\left(a_1+a_2\right)^2\leq 2\left(a_1^2+a_2^2\right)$}\\
            &\leq L_su_{cs}^2\hat{C}(y_0)\tag{Part \eqref{assump:moments_bound} in Assumption \ref{assump:Poisson_eq}}\\
            &\leq \frac{\varphi_1\hat{C}(y_0)}{u}.\tag{$\frac{u_{2s}}{l_{2s}}\geq 1$}
        \end{align*}
        \item For this part, we have
        \begin{align*}
            \E[|d_k|]&\leq \frac{L_su_{cs}^2A_2}{2}\left(\E\left[\|x_k-x^*\|_c^2+\left(\|x_k-x^*\|_c+\frac{B_2}{A_2}\right)^2\right]\right)\tag{$ab\leq \frac{a^2+b^2}{2}$}\\
            &\leq L_su_{cs}^2A_2\left(\E[\|x_k-x^*\|_c^2]+\left(\frac{B_2}{A_2}\right)^2\right)\\
            &\leq \frac{\varrho_1}{u}(uA\E[\Phi(x_k-x^*)]+B).\tag{Eq. \eqref{assump:equivalence_rel} in Assumptions \eqref{assump:Lyapunov_fn}}
        \end{align*}
    \end{enumerate}
\end{proof}

\begin{lemma}\label{lem:telscoping_rearrange}
Under the Assumptions \ref{assump:linear_F}-\ref{assump:step-size}, we have the following:
    \begin{enumerate}[(a)]
        \item When the set $\mathcal{X}$ is an $\ell_2$-ball of sufficiently large such that $x^*\in \mathcal{X}$, then
        \begin{align}\label{lem:telscoping_rearrange;bounded}
            \alpha_k(\E[d_k]-\E[d_{k+1}])&\leq  \left(1-\eta\alpha_k\right)\alpha_{k-1}\E[d_k]-\alpha_k\E[d_{k+1}]+\frac{2\alpha_k^2\kappa \varphi_1\hat{C}(y_0)}{u}.
        \end{align}
        \item  When $\mathcal{Y}$ is bounded and $\mathcal{X}\equiv\mathbb{R}^d$, then
        \begin{align}\label{lem:telscoping_rearrange;finite}
            \alpha_k(\E[d_k]-\E[d_{k+1}])&\leq \left(1-\frac{\eta\alpha_k}{2}\right)\alpha_{k-1}\E[d_k]-\alpha_k\E[d_{k+1}]+\frac{2\alpha_k^2\kappa \varrho_1}{u}\left(uA\E[\Phi(x_k-x^*)]+B\right).
        \end{align}
    \end{enumerate}
\end{lemma}
\begin{proof}
    \begin{enumerate}[(a)]
        \item Re-writing the expression, we get
        \begin{align*}
            \alpha_k(\E[d_k]-\E[d_{k+1}])&=\left(1-\eta\alpha_k\right)\alpha_{k-1}\E[d_k]-\alpha_k\E[d_{k+1}]+\alpha_k\E[d_k]-\left(1-\eta\alpha_k\right)\alpha_{k-1}\E[d_k]\\
            &= \left(1-\eta\alpha_k\right)\alpha_{k-1}\E[d_k]-\alpha_k\E[d_{k+1}]+\left(\alpha_k-\alpha_{k-1}+\eta\alpha_k\alpha_{k-1}\right)\E[d_k]\\
            &\leq \left(1-\eta\alpha_k\right)\alpha_{k-1}\E[d_k]-\alpha_k\E[d_{k+1}]+2\alpha_k^2\left(\frac{\xi}{\alpha}+\eta\right)\E[d_k].\tag{Lemma \ref{lem:step-size_prop}}
            \end{align*} 
        Using Eq. \eqref{lem:telescoping_term-bound;bounded} in Lemma \ref{lem:telescoping_term-bound}, we get 
        \begin{align*}
            \alpha_k(\E[d_k]-\E[d_{k+1}])&\leq \left(1-\eta\alpha_k\right)\alpha_{k-1}\E[d_k]-\alpha_k\E[d_{k+1}]+\frac{2\alpha_k^2\kappa \varphi_1\hat{C}(y_0)}{u}.
        \end{align*}
        \item For this part, we re-write the expression as follows:
        \begin{align*}
            \alpha_k(\E[d_k]-\E[d_{k+1}])&=\left(1-\frac{\eta\alpha_k}{2}\right)\alpha_{k-1}\E[d_k]-\alpha_k\E[d_{k+1}]+\alpha_k\E[d_k]-\left(1-\frac{\eta\alpha_k}{2}\right)\alpha_{k-1}\E[d_k]\\
            &= \left(1-\frac{\eta\alpha_k}{2}\right)\alpha_{k-1}\E[d_k]-\alpha_k\E[d_{k+1}]+\left(\alpha_k-\alpha_{k-1}+\frac{\eta\alpha_k\alpha_{k-1}}{2}\right)\E[d_k]\\
            &\leq \left(1-\frac{\eta\alpha_k}{2}\right)\alpha_{k-1}\E[d_k]-\alpha_k\E[d_{k+1}]+\alpha_k^2\left(\frac{2\xi}{\alpha}+\eta\right)\E[d_k].\tag{Lemma \ref{lem:step-size_prop}}
            \end{align*} 
        Using Eq. \eqref{lem:telescoping_term-bound;finite} in Lemma \ref{lem:telescoping_term-bound} to get 
        \begin{align*}
            \alpha_k(\E[d_k]-\E[d_{k+1}])&\leq \left(1-\frac{\eta\alpha_k}{2}\right)\alpha_{k-1}\E[d_k]-\alpha_k\E[d_{k+1}]\\
            &~~~~+\alpha_k^2\left(\frac{2\xi}{\alpha}+\eta\right)\frac{\varrho_1}{u}\left(uA\E[\Phi(x_k-x^*)]+B\right)\\
            &\leq \left(1-\frac{\eta\alpha_k}{2}\right)\alpha_{k-1}\E[d_k]-\alpha_k\E[d_{k+1}]+\frac{2\alpha_k^2\kappa \varrho_1}{u}\left(uA\E[\Phi(x_k-x^*)]+B\right).
        \end{align*}
    \end{enumerate}
\end{proof}

\begin{lemma}\label{lem:one_step-lipschitz}
    Under the Assumptions \ref{assump:linear_F} and \ref{assump:linear_M}, $\forall~k\geq 0$, we have 
    \begin{enumerate}[(a)]
        \item When the set $\mathcal{X}$ is an $\ell_2$-ball of sufficiently large such that $x^*\in \mathcal{X}$, then
        \begin{align}\label{lem:one_step-lipschitz;bounded}
            \|x_{k+1}-x_{k}\|_c\leq \alpha_k\frac{u_{2c}}{l_{2c}}(A_1(Y_k)M+B_1(Y_k)+A_3M+B_3).
        \end{align}
        \item  When $\mathcal{Y}$ is bounded and $\mathcal{X}\equiv\mathbb{R}^d$, then
        \begin{align}\label{lem:one_step-lipschitz;finite}
            \|x_{k+1}-x_{k}\|_c\leq \alpha_k((A_1+A_3)\|x_k-x^*\|_c+B_1+B_3).
        \end{align}
    \end{enumerate}
\end{lemma}
\begin{proof}
\begin{enumerate}[(a)]
    \item Using the iteration \eqref{eq:main_rec} and the fact that $x_k\in \mathcal{X}$, we have
    \begin{align*}
        \|x_{k+1}-x_{k}\|_{c}&\leq u_{2c}\|x_{k+1}-x_{k}\|_{c}\\
        &=u_{2c}\|\Pi_\mathcal{X}\left(x_k+\alpha_k(F(x_k, Y_k)+M_k)\right)-\Pi_{\mathcal{X}}(x_{k})\|_{2}\\
        &\leq\alpha_ku_{2c}\|F(x_k, Y_k)+M_k\|_2 \tag{Non-expansive projection}\\
        &\leq\alpha_k\frac{u_{2c}}{l_{2c}}\|F(x_k, Y_k)+M_k\|_c \\
        &\leq \alpha_k\frac{u_{2c}}{l_{2c}}(\|F(x_k, Y_k)\|_c+\|M_k\|_c) \\
        &\leq \alpha_k\frac{u_{2c}}{l_{2c}}((A_1(Y_k)+A_3)\|x_k-x^*\|_c+B_1(Y_k)+B_3)\tag{Assumptions \eqref{assump:linear_F} and \eqref{assump:linear_M}}\\
        &\leq \alpha_k\frac{u_{2s}}{l_{2s}}((A_1(Y_k)+A_3)M+B_1(Y_k)+B_3).\tag{$\|x_k-x^*\|_c\leq M$}
    \end{align*}
    \item Using the iteration \eqref{eq:main_rec} and the fact that $\mathcal{X}\equiv \mathbb{R}^d$, we have
    \begin{align*}
        \|x_{k+1}-x_{k}\|_{c}&=\alpha_k\|F(x_k, Y_k)+M_k\|_c \\
        &\leq \alpha_k(\|F(x_k, Y_k)\|_c+\|M_k\|_c) \\
        &\leq \alpha_k(\|F(x_k, Y_k)\|_c+\|M_k\|_c)\\
        &\leq \alpha_k((A_1(Y_k)+A_3)\|x_k-x^*\|_c+B_1(Y_k)+B_3)\tag{Assumptions \eqref{assump:linear_F} and \eqref{assump:linear_M}}\\
        &\leq \alpha_k((A_1+A_3)\|x_k-x^*\|_c+B_1+B_3).\tag{Part \eqref{assump:moments_bound} in Assumption \ref{assump:Poisson_eq}}
    \end{align*}
\end{enumerate}

\end{proof}

\section{SA counterexample with almost sure convergence but diverging mean square error}\label{appendix:counterexample}
\begin{theorem}
    The SA recursion defined in Eq. \eqref{eq:counterex_iteration} satisfies the following:
    \begin{enumerate}[(a)]
        \item For any choice of $\lambda$ and $\nu$ such that $\lambda<\nu$, we have $\lim_{k\to \infty} x_k = 0$ a.s.
        \item If $\lambda\in (\exp(-\log(4)\log(3.4)), 0.5)$, then $\lim_{k\to \infty}\E[x_k^2]=\infty$.
    \end{enumerate}
\end{theorem}
\begin{proof}
    The almost sure convergence of Eq. \eqref{eq:counterex_iteration} follows directly from Theorem 1 in \cite{borkar2024}. For part (b), we will assume that $\bar{q}\in \mathbb{N}$ for notational simplicity. We choose $k_0=\bar{q}$ and open up the iteration \eqref{eq:counterex_iteration} from $k_0$ to any time $k>k_0$ as follows
    \begin{align*}
        x_{k+1}=x_{k_0}\prod_{j=k_0}^{k}\left(1+\alpha_j\left(Q_{j+1}-\bar{q}-1\right)\right)+\sum_{j=k_0}^kW_j\prod_{l=j+1}^{k}\left(1+\alpha_l\left(Q_{l+1}-\bar{q}-1\right)\right).
    \end{align*}
    Using independence of $W_k$, it follows that the mean square error is given as
    \begin{align}
        \E[x_k^2]&=\E\left[x_{k_0}^2\prod_{j=k_0}^{k}\left(1+\alpha_j\left(Q_{j+1}-\bar{q}-1\right)\right)^2\right]+\sum_{j=k_0}^k\E\left[\prod_{l=j+1}^{k}\left(1+\alpha_l\left(Q_{l+1}-\bar{q}-1\right)\right)^2\right]\nonumber\\
        &\geq \E\left[x_{k_0}^2\prod_{l=k_0+1}^{k}\left(1+\alpha_l\left(Q_{l+1}-\bar{q}-1\right)\right)^2\right]=\Delta_k\label{eq:phik}.
    \end{align}
    Note that $\Delta_k$ is the expectation of a positive quantity and thus can be lower bounded by accounting only a single sample path of $\{Q_k\}_{k\geq 0}$. In what follows, we will show that this lower bound on $\Delta_k$ is an increasing function of $k$, therefore $\E[x_k^2]$ diverges as $k$ tends to infinity. 
    
    Define $g(l)=Q_{l+1}-\bar{q}-1$ and pick an arbitrary time instant $k\geq k_1$ large enough, where $k_1$ will be specified in Lemma \ref{lem:summ_bound}. Consider a sample path where $D_l=1$ for all $0\leq l\leq k$. Then, for this sample path, we can write $g(l)$ as
    \begin{align*}
        g(l)=
            l-\bar{q}~~\text{for } 0\leq l\leq k.
    \end{align*}
    By the choice of $k_0$, $g(l)\geq 0$ for $k_0\leq l\leq k$. Furthermore, since $D_l\geq -1$, we also have the following almost sure lower bound on $g(l)$
    \begin{align*}
        g(l)\geq 2k-\bar{q}-l~~\text{for } k+1\leq l\leq 2k-k_0.
    \end{align*}
    Next, note that $\alpha_lg(l)\geq g(l)/(l+1)$ since $\alpha\geq 1$ and $\xi\leq 1$. With the proposed sample path and the relations established above, we can lower bound $\Delta_j$  for all $0\leq j\leq 2k-k_0$ as
    \begin{align*}
        \Delta_j\geq \E[x_{k_0}^2]e^{2\log(2)\sum_{l=k_0}^{j}g(l)/(l+1)}P(D_l=1, 0\leq l\leq k),
    \end{align*}
    where we use the fact $0\leq \alpha_lg(l)\leq 1$ for all $k_0\leq j \leq k$ and $(1+x)\geq e^{\log(2)x}$ for all $0\leq x\leq 1$. Our next goal is to show that this lower bound on $\Delta_{k}$ diverges as $k\to \infty$. To this end, we will need the following lemma.
    \begin{lemma}\label{lem:summ_bound}
    Let $k_1$ be a constant defined as
    \begin{align*}
        k_1=\inf\left\{j:\log\left(\frac{2j-k_0+1}{j+2}\right)\geq \log(1.9), \frac{1}{j}-\frac{\bar{q}+1}{j}\log \left(\frac{2j-k_0}{k_0+1}\right)\geq \log(0.95)\right\}.
    \end{align*}
    Then, for all $k\geq k_1$, the following relation holds.
        \begin{align*}
            \sum_{l=k_0}^{2k-k_0}\frac{1}{l+1}g(l)\geq k\log(3.4).
        \end{align*}
    \end{lemma}
    \begin{proof}
        \begin{align*}
            \sum_{l=k_0}^{2k-k_0}\frac{1}{l+1}g(l)&= \sum_{l=k_0}^{k}\frac{1}{l+1}(l-\bar{q})+\sum_{l=k+1}^{2k-k_0}\frac{1}{l+1}(2k-\bar{q}-l)\\
            &=k-k_0-(\bar{q}+1)\left(\sum_{l=k_0+1}^{k}\frac{1}{l+1}\right)+(2k-\bar{q}+1)\left(\sum_{k=k+1}^{2k-k_0-1}\frac{1}{l+1}\right)-(k-k_0-1)\\
            &=1-(\bar{q}+1)\left(\sum_{l=k_0+1}^{2k-k_0-1}\frac{1}{l+1}\right)+2(k+1)\left(\sum_{l=k+1}^{2k-k_0-1}\frac{1}{l+1}\right).
        \end{align*}
        Since the function $h(x)=1/(x+1)$ is non-increasing, we use the inequality $\int_a^{b+1}dx/(x+1)\leq \sum_{l=a}^b1/(k+1)\leq \int_{a-1}^bdx/(x+1)$ for the second and third term, to get
        \begin{align*}
            \sum_{l=k_0}^{2k-k_0}\frac{1}{l+1}g(l)&\geq 1-(\bar{q}+1)\int_{k_0+1}^{2k-k_0-1}\frac{1}{x+1}dx+2(k+1)\int_{k+1}^{2k-k_0}\frac{1}{x+1}dx\\
            &\geq 1-(\bar{q}+1)\log \left(\frac{2k-k_0}{k_0+1}\right)+2k\log\left(\frac{2k-k_0+1}{k+2}\right)\\
            &= k\left(2\log\left(\frac{2k-k_0+1}{k+2}\right)+\frac{1}{k}-\frac{\bar{q}+1}{k}\log \left(\frac{2k-k_0}{k_0+1}\right)\right).
        \end{align*}
        Recall that $k\geq k_1$ which leads us to
        \begin{align*}
            \sum_{l=k_0}^{2k-k_0}\frac{1}{l+1}g(l)&\geq k\log(3.4).
        \end{align*}
\end{proof}
The probability of $D_l=1$ consecutively occurring $k$ times is given by $\lambda^{k}=e^{k\log(\lambda)}$. Thus, it follows that $\Delta_{2k-k_0}$ is lower bounded by
\begin{align*}
    \Delta_{2k-k_0}&\geq \E[x_{k_0}^2]e^{k\log(\lambda)}e^{2k\log(2)\log(3.4)}\\
    &=\E[x_{k_0}^2]e^{k\left(\log(\lambda)+2\log(2)\log(3.4)\right)}.
\end{align*}

Note that $\lambda>e^{-\log(4)\log(3.4)}$. Therefore, the coefficient $\log(4)\log(3.4)+\log(\lambda)>0$. Thus, for any large enough $k$, we have
\begin{align*}
    \E[x_{k}^2]\geq \E[x_{k_0}^2]e^{k\left(\log(\lambda)+2\log(2)\log(3.4)\right)}.
\end{align*}
Hence, it follows that $\E[x_k^2]\xrightarrow{k\to \infty} \infty$.
\end{proof}

\subsection{Expectation vs. Almost Sure Behavior} \label{appendix:gamblers_ruin}
This example is inspired by Exercise 5.6 in \cite{gallager2013stochastic} and St. Petersburg Paradox. Suppose that a gambler has initial wealth $W_0$ and finds a casino where the probability of winning the game is $p=0.6$. The strategy he adopts is as follows: for each round, he bets $\alpha=50\%$ of his wealth, and if he wins, he doubles his bet for the next round. He decides to keep betting until he goes bankrupt. 

Denote $X_i$ as the return in the $i$th round, then the wealth $W_k$ at the end of round $k$ is given by $W_k=W_0\Pi_{i=1}^kX_i$. Note that for all $i\geq 1$
\begin{align*}
    X_i=\begin{cases}
        1+\alpha~~\text{w.p.}~p\\
        1-\alpha~~\text{w.p.}~1-p
    \end{cases}
\end{align*}
Therefore,
\begin{align*}
    \E[X_i]=p(1+\alpha)+(1-\alpha)(1-p)&=0.6\cdot 1.5+0.4\cdot 0.5=1.1\\
    \Rightarrow \E[W_k]&=W_0\Pi_{i=1}^k\E[X_i]\\
    &=W_0(1.1)^k.
\end{align*}
Thus, $\lim_{k\to \infty}\E[W_k]=\infty$.

On the other hand, let us consider $\log W_k$. 
\begin{align*}
    \log \frac{W_k}{W_0}=\sum_{i=1}^k\log X_i\\
    \Rightarrow \frac{1}{k}\log \frac{W_k}{W_0}=\frac{1}{k}\sum_{i=1}^k\log X_i.
\end{align*}
Since $\E[|\log X_i|]<\infty$, using strong law of large numbers, we get
\begin{align*}
    \lim_{k\to \infty}\frac{1}{k}\sum_{i=1}^k\log X_i&\stackrel{a.s.}{=}\E[\log X_1]\\
    &=p\log (1+\alpha)+(1-p)\log (1-\alpha)\\
    &\approx -0.05.
\end{align*}
The above implies $\lim_{k\to \infty}\frac{1}{k}\log \frac{W_k}{W_0}\stackrel{a.s.}{=}0.6\log 1.5+0.4\log 0.5<0$. In other words, $W_k\stackrel{a.s.}{\to}0$. Hence, the gambler's wealth blows up to infinity in expectation, while almost surely he goes bankrupt. This behavior is identical to the counterexample presented in Section \ref{sec:counterexample}. 
\subsubsection{Concentration bound}

Nevertheless, for this simple setting, we can still upper bound the expected wealth conditioned on a high probability set as follows. Let $\mathbbm{1}_{W_i}$ be indicator random variables that denote the event of success in round $i$. Then, we can re-write $X_i=(1+\alpha)^{\mathbbm{1}_{W_i}}(1-\alpha)^{1-\mathbbm{1}_{W_i}}$ and therefore, the wealth $W_k$ is given by
\begin{align*}
    W_k&=W_0(1+\alpha)^{\sum_{i=1}^k\mathbbm{1}_{W_i}}(1-\alpha)^{k-\sum_{i=1}^k\mathbbm{1}_{W_i}}\\
    &=W_0\left((1+\alpha)^p(1-\alpha)^{1-p}\right)^k(1+\alpha)^{\sum_{i=1}^k(\mathbbm{1}_{W_i}-p)}(1-\alpha)^{-\sum_{i=1}^k(\mathbbm{1}_{W_i}-p)}\\
    &=W_0\rho^k(1+\alpha)^{\sum_{i=1}^k(\mathbbm{1}_{W_i}-p)}(1-\alpha)^{-\sum_{i=1}^k(\mathbbm{1}_{W_i}-p)}
\end{align*}
where $\rho=(1+\alpha)^p(1-\alpha)^{1-p}<1$. Note that $\mathbbm{1}_{W_i}$ are i.i.d. random variables. Thus, for any $\delta'>0$, Hoeffding's inequality leads us to
\begin{align*}
    P\left(\sum_{i=1}^k(\mathbbm{1}_{W_i}-p)\geq \delta' k\right)&\leq e^{-2{\delta'}^2k}\\
    \Rightarrow P\left(W_k\geq W_0\rho^k\left(\frac{1+\alpha}{1-\alpha}\right)^{\delta' k}\right)&\leq e^{-2{\delta'}^2k}.
\end{align*}
Let $\delta=e^{-2{\delta'}^2k}$. Then, the above concentration can be translated into
\begin{align*}
     P\left(W_k< W_0\rho^k\left(\frac{1+\alpha}{1-\alpha}\right)^{\sqrt{\frac{k}{2}\log\left(\frac{1}{\delta}\right)}}\right)&\geq 1-\delta.
\end{align*}
Note that $\rho<1$ and $\lim_{\xi\to 0}((1+\alpha)/(1-\alpha))^{\xi}=1$. Thus, there exists a constant $\xi_0>0$ such that $\rho((1+\alpha)/(1-\alpha))^{\xi}<1$ for all $\xi\in(0, \xi_0)$. If we set $\delta_0(k)=\exp(-2k\xi_0^2)$, then for $\delta_1(k)=\omega(\delta_0(k))$, we get
\begin{align*}
    P\left(W_k< W_0\rho^k\left(\frac{1+\alpha}{1-\alpha}\right)^{\sqrt{\frac{k}{2}\log\left(\frac{1}{\delta_1(k)}\right)}}\right)&\geq 1-\delta_1(k).
\end{align*}
By the definition of $\omega(\cdot)$, there exists a function $g_1(k)$ such that $\delta_1(k)=g_1(k)\delta_0(k)$ with $g_1(k)\to \infty$. Using the upper bound $\sqrt{1+x}\leq 1+x/2$ for all $x\geq -1$, we get
\begin{align*}
    \rho^k\left(\frac{1+\alpha}{1-\alpha}\right)^{\sqrt{\frac{k}{2}\log\left(\frac{1}{\delta_1(k)}\right)}}
    &=\rho^k\left(\frac{1+\alpha}{1-\alpha}\right)^{\sqrt{k^2\xi_0^2+\frac{k}{2}\log\left(\frac{1}{g_1(k)}\right)}}\\
    &= \rho^k\left(\frac{1+\alpha}{1-\alpha}\right)^{k\xi_0\sqrt{1-\frac{1}{2k\xi_0^2}\log\left(g_1(k)\right)}}\\
    &\leq \rho^k\left(\frac{1+\alpha}{1-\alpha}\right)^{k\xi_0-\frac{1}{4\xi_0}\log\left(g_1(k)\right)}\stackrel{k\to \infty}{\to}  0.
\end{align*}
Similarly, for $\delta_2(k)=o(\delta_0(k))$, we get
\begin{align*}
    P\left(W_k\geq W_0\rho^k\left(\frac{1+\alpha}{1-\alpha}\right)^{\sqrt{\frac{k}{2}\log\left(\frac{1}{\delta_2(k)}\right)}}\right)&\leq \delta_2(k).
\end{align*}
By the definition of $o(\cdot)$, there exists a function $g_2(k)$ such that $\delta_2(k)=g_2(k)\delta_0(k)$ with $g_2(k)\to 0$. Using the lower bound $\sqrt{1+x}\geq 1+x/4$ for all $x\in[0, 1]$, we get
\begin{align*}
    \rho^k\left(\frac{1+\alpha}{1-\alpha}\right)^{\sqrt{\frac{k}{2}\log\left(\frac{1}{\delta_2(k)}\right)}}
    &=\rho^k\left(\frac{1+\alpha}{1-\alpha}\right)^{\sqrt{k^2\xi_0^2+\frac{k}{2}\log\left(\frac{1}{g_2(k)}\right)}}\\
    &= \rho^k\left(\frac{1+\alpha}{1-\alpha}\right)^{k\xi_0\sqrt{1-\frac{1}{2k\xi_0^2}\log\left(g_2(k)\right)}}\\
    &\geq \rho^k\left(\frac{1+\alpha}{1-\alpha}\right)^{k\xi_0-\frac{1}{8\xi_0}\log\left(g_2(k)\right)}\stackrel{k\to \infty}{\to} \infty.
\end{align*}

\subsection{Comparison with a simpler approach to show divergence}\label{sec:simpler_method}
Now we discuss a ``potentially simpler method'' for proving divergence and show why it fails in doing so, thus highlighting the significance of our method. Previously, we considered the set $\{D_l=1\}$ for $0\leq l\leq k$ for a large enough $k$ that led to the linear growth of the length of the queue $Q_l$ within the time window. Instead, pick any arbitrary constant $q_0>\bar{q}+1$ and consider the set $\{Q_l\geq q_0\}$ for $0\leq l\leq k$. For convenience, we will assume that the system started in stationary state. Let $\lambda/\nu = p$. Then, we know that the stationary distribution of the $M/M/1$ queue is a geometric distribution with the parameter $p$ \cite{Ross_07}. Also, note that due to stationarity, there exists some constant $p_{q_0}\in (0, 1)$ such that $P(Q_{k+1}\geq q_0| Q_k\geq q_0)=p_{q_{0}}$ for all $k\geq 0$. 

Fix $k_0$ to be some starting time instant. Repeating the steps we performed for the lower bound in Theorem \ref{thm:div_ex}, we get
\begin{align*}
    \E[x_k^2]&\geq \E\left[x_{k_0}^2\prod_{l=k_0}^{k}\left(1+\alpha_l\left(Q_{l+1}-\bar{q}-1\right)\right)^2\right]\\
    &\geq \E\left[x_{k_0}^2\mathbbm{1}\{Q_{l+1}\geq q_0, k_0\leq l\leq k \}\prod_{l=k_0}^{k}\left(1+\alpha_l\left(Q_{l+1}-\bar{q}-1\right)\right)^2\right]\\
    &\geq \E\left[x_{k_0}^2\mathbbm{1}\{Q_{l+1}\geq q_0, k_0\leq l\leq k \}\prod_{l=k_0}^{k}\left(1+\alpha_l\left(q_0-\bar{q}-1\right)\right)^2\right]\\
    &= \E[x_{k_0}^2\mathbbm{1}\{Q_{l+1}\geq q_0, k_0\leq l\leq k \}]e^{\mathcal{O}(\sum_{l=k_0}^k\alpha_l)}\\
    &= \E[x_{k_0}^2\E[\mathbbm{1}\{Q_{l+1}\geq q_0, k_0\leq l\leq k \}|\mathcal{F}_{k_0}]]e^{\mathcal{O}(\sum_{l=k_0}^k\alpha_l)}
\end{align*}
where $\mathcal{F}_{k_0}$ is the field $\sigma$ generated by the iterations and the queue process up to time $k_0$. Note that the queue process evolves independently of the iterations $x_k$. Thus, the conditional expectation is given as
\begin{align*}
    \E[\mathbbm{1}\{Q_{l+1}\geq q_0, k_0\leq l\leq k \}|\mathcal{F}_{k_0}] &= \prod_{l=k_0}^kP(Q_{l+1}\geq q_0|Q_l\geq q_0)\\
    &= p_{q_0}^{k-k_0+1}= \mathcal{O}(p_{q_0}^k).
\end{align*}
Thus, for any $k$, we get
\begin{align*}
    \E[x_{k}^2]&\geq \E[x_{k_0}^2]e^{\mathcal{O}(\sum_{l=k_0}^{k}\alpha_l)}\mathcal{O}(p_{q_0}^{k}).
\end{align*}
Recall that $\int_a^{b+1}dx/(x+1)\leq \sum_{l=a}^b1/(k+1)\leq \int_{a-1}^bdx/(x+1)$. Thus, $e^{\mathcal{O}(\sum_{l=k_0}^{k_1}\alpha_l)}=\mathcal{O}((k+1)^\zeta)$ for some $\zeta \in \mathbb{R}$. This gives us the following vacuous bound:
\begin{align*}
    \lim_{k\to \infty}\E[x_{k}^2]&\geq \lim_{k\to \infty}\E[x_{k_0}^2]e^{\mathcal{O}(\sum_{l=k_0}^{k}\alpha_l)}\mathcal{O}(p_{q_0}^{k}) = 0.
\end{align*}

\begin{figure}
    \centering
        \begin{tikzpicture}
    \def\kone{4}
    \def\kzero{1}
    \draw[->,thick, black] (6,3.3) -- (6.6,3.3)
    node[pos=1, anchor=west, color=black] {\small{Upper Bound on $Q_l$}};

    \draw[->,thick, black] (6,2.15) -- (6.6,2.15)
    node[pos=1, anchor=west, color=black] {\small{A sample path for $Q_l\in \mathcal{S}_1$}};

    \draw[->,thick, black] (6,1.3) -- (6.6,1.3)
    node[pos=1, anchor=west, color=black] {\small{Lower Bound on $Q_l$}};

    \begin{axis}[
        legend style={
        at={(0.55,1)},         
        anchor=north,            
        legend columns=2,         
        cells={anchor=west},
        font=\small,                 
        column sep=2pt,              
    },
        xlabel={$l$},
        ylabel={$Q_l$},
        xmin=-1, xmax=8,
        ymin=-1, ymax=5,
        axis lines=middle,
        xtick={0, \kzero, \kone, 2*\kone-\kzero+0.1},
        xticklabels={$0$, $k_0$, $k_1$, $2k_1-k_0$},
        ytick={0, \kzero/2, \kzero/2+0.5},
        yticklabels={$0$, $\bar{q}$, $q_{0}$},
        xlabel style={anchor=north east, at={(axis description cs:1.05,0.1)}},
        ylabel style={anchor=north east, at={(axis description cs:0.1,1.1)}},
        axis equal image,
    ]
        \addlegendimage{area legend, fill=blue, draw=blue, fill opacity=0.3, mark=none}
        \addlegendentry{$\mathcal{S}_1$}
        \addlegendimage{area legend, fill=red, draw=red, fill opacity=0.3, mark=none}
        \addlegendentry{$\mathcal{S}_2$}
        
        \addplot[
            color=blue,      
            fill=blue,       
            fill opacity=0.3,
            draw=none       
        ] coordinates {
            (\kone, \kone/2)          
            (2*\kone-0.5, \kone-0.25)        
            (2*\kone-0.5, \kzero/2-0.2)
            (\kone, \kone/2)
        };

        \addplot[
            color=red,      
            fill=red,       
            fill opacity=0.3,
            draw=none       
        ] coordinates {
            (0, 4)          
            (2*\kone-0.5, 4)        
            (2*\kone-0.5, \kzero/2+0.5)
            (0, \kzero/2+0.5)
            (0, 3)
        };

        \addplot[
            color=black,  
            thick,       
            mark size=2pt
        ] coordinates {
            (0, 0) 
            (\kone, \kone/2) 
            (\kone+0.4, \kone/2-0.2)
            (\kone+0.8, \kone/2)
            (\kone+1.2, \kone/2+0.2)
            (\kone+1.6, \kone/2+0.4)
            (\kone+2, \kone/2+0.2)
            (\kone+2.4, \kone/2)
            (\kone+2.8, \kone/2-0.2)
            (\kone+3.2, \kone/2)
            (\kone+3.6, \kone/2+0.2)
        };

        \addplot[
            color=black,  
            thick,  
            dashed,
            mark size=2pt
        ] coordinates {
            (\kone, \kone/2) 
            (2*\kone - 0.5, \kzero/2-0.2)
        };
        \node[rotate=-28] at (5.4,0.95) {\small{Slope $=-1$}};

        \addplot[
            color=black,  
            dashed,       
            thick,
            mark size=2pt
        ] coordinates {
            (0, \kzero/2) 
            (2*\kone-0.5, \kzero/2) 
        };

        \addplot[
            color=black,  
            dashed,       
            thick,
            mark size=2pt
        ] coordinates {
            (0, 0) 
            (2*\kone-0.5, \kone-0.25) 
        };
        \node[rotate=28] at (5.4,3) {\small{Slope $=1$}};

        \addplot[
            color=black,  
            dashed,       
            thick,
            mark size=2pt
        ] coordinates {
            (0, \kzero/2+0.5) 
            (2*\kone-0.5, \kzero/2+0.5) 
        };

    \end{axis}
\end{tikzpicture}
    \caption{$\mathcal{S}_1=\{\{Q_l\}_{l\geq 0}; Q_0=0, D_l=1~\forall 0\leq l\leq k_1\}$, $\mathcal{S}_2=\{\{Q_l\}_{l\geq 0}; Q_l\geq q_0~\forall 0\leq l<\infty \}$. For illustrative purposes, we represent $Q_l$ as a continuous piecewise linear function by linearly interpolating queue lengths between the time instants.}
    \label{fig:bad sets 2}
\end{figure}

The above analysis illustrates the significance of considering on the set where $\{Q_l\}_{l\geq 0}$ grows linearly with time and why it helps to capture the divergence of $x_k$ accurately. More precisely, to offset the exponentially small probability of the bad set, we need the random variable to grow exponentially. This can only be achieved if $\{Q_l\}_{l\geq 0}$ has linear growth with time. Although one can choose the slope of this growth in a variety of fashion, the simplest case is to consider ${D_l=1}$ for all $k_0\leq l\leq k$ where the slope takes the value 1. Figure \ref{fig:bad sets 2} gives a visual representation of the bad sets considered in both approaches.

\section{Proof of technical results in Section \ref{sec:RL}}\label{appendix:B}
Before beginning the proofs of the lemmas and propositions in this section, we need Lemma 7 from \cite{tsitsiklis1997} in order to prove Lemma \ref{lem:stationary_exp_Tb}. We state it here for completeness, but we omit the proof as it is essentially repeating the same arguments with the contraction factor being 1.

\begin{lemma}\label{lem:benvan}
         Under Assumptions \ref{assump:dist_td} and \ref{assump:feature_mat}, the following relations hold in the steady state of the Markov process $Y_k$.
        \begin{enumerate}[(a)]
            \item $\E_{\mu}[\psi(\tilde{S}_k)\psi(\tilde{S}_{k+m})^T]=\Psi^T\Lambda P^m\Psi$, for all $m\geq 0$.
            \item $\|E_{\mu}[\psi(\tilde{S}_k)\psi(\tilde{S}_{k+m})^T]\|_2=\psi'<\infty$, for all $m\geq 0$.
            \item $E_{\mu}[\tilde{z}_k\psi(\tilde{S}_k)^T]=\Psi^T\Lambda\left(\sum_{m=0}^\infty\lambda^m P^m\right)\Psi$.
            \item $E_{\mu}[\tilde{z}_k\psi(\tilde{S}_{k+1})^T]=\Psi^T\Lambda\left(\sum_{m=0}^\infty\lambda^mP^{m+1}\right)\Psi$.
            \item $E_{\mu}[\tilde{z}_k\mathcal{R}(\tilde{S}_k, A_k)]=\Psi^T\Lambda\left(\sum_{m=0}^\infty\lambda^m P^m\right)\mathcal{R}_{\pi}$.
        \end{enumerate}
        
    \end{lemma}

\subsection{Proof of Lemma \ref{lem:stationary_exp_Tb}}\label{appendix:stationary_exp_Tb}
\begin{proof}
    Using Lemma \ref{lem:benvan}, we have
    \begin{align*}
        \E_{\mu}[T(\tilde{Y}_k)]&=\begin{bmatrix}
        -c_\alpha & 0\\
        -\Pi_{2, E^{\perp}_{\Psi}}\E_{\mu}[\tilde{z}_k] & \Pi_{2, E^{\perp}_{\Psi}}\E_{\mu}\left[\tilde{z}_k\left(\psi(\tilde{S}_{k+1})^T\theta_k-\psi(\tilde{S}_k)^T\right)\right]\\
    \end{bmatrix}\\
    &=\begin{bmatrix}
        -c_\alpha & 0\\
        -\frac{1}{(1-\lambda)}\Pi_{2, E^{\perp}_{\Psi}}\Psi^T\mu & \Pi_{2, E^{\perp}_{\Psi}}\left(\sum_{m=0}^\infty\lambda^m \Psi^T\Lambda P^{m+1}\Psi-\Psi^T\Lambda P^m\Psi\right)
        \end{bmatrix}.
    \end{align*}
    Note that for any $\lambda\in [0,1)$, we can rewrite $\lambda^m=(1-\lambda)\sum_{l=m}^\infty\lambda^l$. Then, it follows that for any $j\geq 0$, we have
    \begin{align*}
        \sum_{m=0}^\infty\lambda^m P^{m+j}&=(1-\lambda)\sum_{m=0}^\infty P^{m+j}\sum_{l=m}^\infty\lambda^l\\
        &=(1-\lambda)\sum_{l=0}^\infty\lambda^l\sum_{m=j}^{l+j}P^m.
    \end{align*}
    Using the above relation for $j=0$ and $j=1$, we get
    \begin{align*}
        \sum_{m=0}^\infty\lambda^m \Psi^T\Lambda P^{m+1}\Psi-\Psi^T\Lambda P^m\Psi&=(1-\lambda)\sum_{l=0}^\infty\lambda^l\Psi^T\Lambda\left(\sum_{m=1}^{l+1}P^m-\sum_{m=0}^{l}P^m\right)\Psi\\
        &=(1-\lambda)\sum_{l=0}^\infty\lambda^l\left(\Psi^T\Lambda P^{l+1}\Psi-\Psi^T\Lambda\Psi\right)\\
        &=\Psi^T\Lambda P^{(\lambda)}\Psi-\Psi^T\Lambda\Psi.
    \end{align*}
    Thus, we have
    \begin{align*}
        \E_{\mu}[T(\tilde{Y}_k)]=\begin{bmatrix}
        -c_\alpha & 0\\
        \frac{1}{(\lambda-1)}\Pi_{2, E^{\perp}_{\Psi}}\Psi^T\mu & \Pi_{2, E^{\perp}_{\Psi}}\left(\Psi^T\Lambda P^{(\lambda)}\Psi-\Psi^T\Lambda\Psi\right)
        \end{bmatrix}=\bar{T}.
    \end{align*}
    Similarly, using Lemma \ref{lem:benvan} the steady-state expectation of $b(Y_k)$ is given by
    \begin{align*}
        \E_{\mu}[b(\tilde{Y}_k)]&=\begin{bmatrix}
            c_\alpha\E_{\mu}[\mathcal{R}(\tilde{S}_k, A_k)]\\
            \Pi_{2, E^{\perp}_{\Psi}}\E_{\mu}[\mathcal{R}(\tilde{S}_k, A_k)\tilde{z}_k]
        \end{bmatrix}\\
        &=\begin{bmatrix}
            c_\alpha\bar{r}\\
            \Pi_{2, E^{\perp}_{\Psi}}\Psi^T\Lambda\left((1-\lambda)\sum_{l=0}^\infty\lambda^l\sum_{m=0}^lP^m\mathcal{R}_{\pi}\right)
        \end{bmatrix}\\
        &=\begin{bmatrix}
            c_\alpha\bar{r}\\
            \Pi_{2, E^{\perp}_{\Psi}}\Psi^T\Lambda \mathcal{R}^{(\lambda)}
        \end{bmatrix}=\bar{b}.
    \end{align*}
\end{proof}

\subsection{Properties of TD\texorpdfstring{($\lambda$)}\ }\label{appendix:td_prop}
Next, we state the following lemma which will be crucial for proving desired properties of TD$(\lambda)$. Define $Y_k=(S_k, A_k, S_{k+1}, z_k)$ and $g(Y_k)=\frac{f_1(S_{k+1})}{(1-\lambda)(1-\rho)}+\frac{\|z_k\|_2\rho^{-1}(f_1(S_k)+f_1(S_{k+1}))}{(1-\lambda\rho)}+\frac{\|z_k\|_2\left(1+\hat{r}+\hat{\psi}\sqrt{d}\right)}{(1-\lambda)}+\frac{\hat{\psi}+\psi'+\hat{\psi}\hat{r}\sqrt{d}}{(1-\lambda)^2}$. Let $Y_0=y_0=(s_0, a_0, s_1, z_0)$.
\begin{lemma}\label{lem:z_k_bound}
    Assume that the eligibility trace vector $z_k$ was initialized from $z_0$. Then, the following relations hold, for all $y_0\in \mathcal{Y}$:
    \begin{enumerate}[(a)]
        \item $\E_\mu[\tilde{z}_k]=\frac{1}{(1-\lambda)}\Psi^T\mu$. Furthermore, $\sum_{k=0}^\infty\|\E_{y_0}[z_k]-\E_{\mu}[\tilde{z}_k]\|_2\leq  g(y_0)$.
        \item $\sum_{k=0}^\infty\|\E_{y_0}[z_k\psi(S_k)^T]-\E_{\mu}[\tilde{z}_k\psi(\tilde{S}_k)^T]\|_2\leq g(y_0)$.
        \item $\sum_{k=0}^\infty\|\E_{y_0}[z_k\psi(S_{k+1})^T]-\E_{\mu}[\tilde{z}_k\psi(\tilde{S}_{k+1})^T]\|_2\leq g(y_0)$.
        \item $\|\Psi^T\Lambda P^m\mathcal{R}_\pi\|_2\leq \hat{\psi}\hat{r}\sqrt{d}$. Furthermore, $\sum_{k=0}^\infty\|\E_{y_0}[z_k\mathcal{R}(S_k, A_k)]-\E_{\mu}[\tilde{z}_k\mathcal{R}(\tilde{S}_k, A_k)]\|_2\leq g(y_0)$.
        \item $\E_{y_0}[\|z_k\|_2^4]\leq \frac{\|z_0\|_2^4+f_3(s_1)}{(1-\lambda)^4}$.
        \item $\E_{y_0}[g^2(Y_k)]\leq \frac{4\sqrt{f_2(s_1)}}{(1-\lambda)^2(1-\rho)^2}+\frac{16\rho^{-2}\sqrt{\|z_0\|_2^4+f_3(s_1)}\sqrt{f_2(s_0)+f_2(s_1)}}{(1-\lambda\rho)^2(1-\lambda)^2}+\frac{4\left(1+\hat{r}+\hat{\psi}\sqrt{d}\right)^2\sqrt{\|z_0\|_2^4+f_3(s_1)}}{(1-\lambda)^4}+\frac{4(\hat{\psi}+\psi'+\hat{\psi}\hat{r}\sqrt{d})^2}{(1-\lambda)^4}$.
    \end{enumerate}
    
\end{lemma}
\begin{proof}
\begin{enumerate}[(a)]
    \item From the definition of $\tilde{z}_k$, we have
    \begin{align*}
        \E_{\mu}[\tilde{z}_k]&=\E_{\mu}\left[\sum_{m=-\infty}^{k}\lambda^{k-m}\psi(\tilde{S}_k)\right]\\
        &=\sum_{m=-\infty}^{k}\lambda^{k-m}\E_{\mu}[\psi(\tilde{S}_k)]\tag{Assumption \ref{assump:feature_mat} and Dominated Convergence Theorem}\\
        &=\sum_{m=-\infty}^{k}\lambda^{k-m}\left(\sum_{s\in\mathcal{S}}\mu(s)\psi(s)\right)\\
        &=\frac{1}{1-\lambda}\left(\sum_{s\in\mathcal{S}}\mu(s)\psi(s)\right)=\frac{1}{1-\lambda}\Psi^T\mu.
    \end{align*}
     Recall that $z_k=\lambda^kz_0+\sum_{j=1}^k\lambda^{k-j}\psi(S_k)$. Using Assumption \ref{assump:stable_markov} and the above relation, we have
    \begin{align*}
        \E_{y_0}[z_k]-\E_{\mu}[\tilde{z}_k]&=\lambda^kz_0+\E_{y_0}\left[\sum_{j=0}^{k-1}\lambda^j\psi(S_{k-j})\right]-\sum_{j=0}^\infty\lambda^j\sum_{s\in \mathcal{S}}\mu(s)\psi(s)\\
        &=\lambda^kz_0+\E_{y_0}\left[\sum_{j=0}^{k-1}\lambda^j\left(\psi(S_{k-j})-\sum_{s\in \mathcal{S}}\mu(s)\psi(s)\right)\right]-\sum_{j=k}^\infty\lambda^j\sum_{s\in \mathcal{S}}\mu(s)\psi(s)
    \end{align*}
    Taking norm both sides and using triangle inequality, we get
    \begin{align*}
        \|\E_{y_0}[z_k]-\E_{\mu}[\tilde{z}_k]\|_2&\leq\lambda^k\|z_0\|_2+\left\|\E_{y_0}\left[\sum_{j=0}^{k-1}\lambda^j(\psi(S_{k-j})-\sum_{s\in \mathcal{S}}\mu(s)\psi(s)\right]\right\|_2+\|\sum_{s\in \mathcal{S}}\mu(s)\psi(s)\|_2\sum_{j=k}^\infty\lambda^j\\
        &\leq \lambda^k\|z_0\|_2+\sum_{j=0}^{k-1}\lambda^j\left\|\E_{y_0}\left[(\psi(S_{k-j})-\E_{\mu}[\psi(\tilde{S}_k)]\right]\right\|_2+\frac{\lambda^k\hat{\psi}}{1-\lambda}\tag{Jensen's inequality and Assumption \ref{assump:feature_mat}}\\
        &\leq \sum_{j=0}^{k-1}\lambda^jf_1(s_1)\rho^{k-j-1}+\lambda^k\|z_0\|_2+\frac{\lambda^k\hat{\psi}}{1-\lambda}\tag{Assumption \ref{assump:stable_markov}}\\
        &\leq  f_1(s_1)\rho^{-1}\sum_{j=0}^{k-1}\lambda^j\rho^{k-j}+\lambda^k\|z_0\|_2+\frac{\lambda^k\hat{\psi}}{1-\lambda}.
    \end{align*}
    Summing over all $k$, we get
    \begin{align*}
        \sum_{k=0}^\infty\|\E_{y_0}[z_k]-\E_{\mu}[\tilde{z}_k]\|_2&\leq \sum_{k=0}^\infty\left(f_1(s_1)\rho^{-1}\sum_{j=0}^{k-1}\lambda^j\rho^{k-j}+\lambda^k\|z_0\|_2+\frac{\lambda^k\hat{\psi}}{1-\lambda}\right)\\
        &=f_1(s_1)\rho^{-1}\sum_{k=0}^\infty\left(\sum_{j=0}^{k-1}\lambda^j\rho^{k-j}\right)+\frac{\|z_0\|_2}{1-\lambda}+\frac{\hat{\psi}}{(1-\lambda)^2}\\
        &=\frac{f_1(s_1)}{(1-\lambda)(1-\rho)}+\frac{\|z_0\|_2}{1-\lambda}+\frac{\hat{\psi}}{(1-\lambda)^2}\tag{Fubini-Tonelli Theorem}\\
        &\leq g(y_0).
    \end{align*}
    \item Using the formula for $z_k$ and part (c) of  Lemma \ref{lem:benvan}, we have
    \begin{align*}
        \E_{y_0}[z_k\psi(S_k)^T]-\E_{\mu}[\tilde{z}_k\psi(\tilde{S}_k)^T]&=\lambda^kz_0\E_{y_0}[\psi(S_k)^T]+\E_{y_0}\left[\sum_{j=0}^{k-1}\lambda^j\psi(S_{k-j})\psi(S_k)^T\right]-\sum_{j=0}^\infty\lambda^j\E_{\mu}[\psi(\tilde{S}_{k-j})\psi(\tilde{S}_k)^T]\\
        &=\lambda^kz_0\E_{y_0}[\psi(S_k)^T]+\E_{y_0}\left[\sum_{j=0}^{k-1}\lambda^j\left(\psi(S_{k-j})\psi(S_k)^T-\E_{\mu}[\psi(\tilde{S}_{k-j})\psi(\tilde{S}_k)^T]\right)\right]\\
        &~~~~-\sum_{j=k}^\infty\lambda^j\E_{\mu}[\psi(\tilde{S}_{k-j})\psi(\tilde{S}_k)^T]
    \end{align*}
    Taking norm both sides and using triangle inequality, we get
    \begin{align*}
        \|\E_{y_0}[z_k\psi(S_k)^T]-\E_{\mu}[\tilde{z}_k\psi(\tilde{S}_k)^T]\|_2&\leq\lambda^k\|z_0\|_2\|\E_{y_0}[\psi(S_k)^T]\|_2\\
        &~~~~+\left\|\E_{y_0}\left[\sum_{j=0}^{k-1}\lambda^j\left(\psi(S_{k-j})\psi(S_k)^T-\E_{\mu}[\psi(\tilde{S}_{k-j})\psi(\tilde{S}_k)^T]\right)\right]\right\|_2\\
        &~~~~+\sum_{j=k}^\infty\lambda^j\|\E_{\mu}[\psi(\tilde{S}_{k-j})\psi(\tilde{S}_k)^T]\|_{2}
    \end{align*}
    To bound the first term, we use Assumption \ref{assump:stable_markov} to get
    \begin{align*}
        \|\E_{y_0}[\psi(S_k)^T]\|_2&\leq \|\E_{y_0}[\psi(S_k)^T]-\E_{\mu}[\psi(\tilde{S}_k)]\|_2+\|\E_{\mu}[\psi(\tilde{S}_k)]\|_2\\
        &\leq \rho^{k-1}(f_1(s_0)+f_1(s_1))+\|\E_{\mu}[\psi(\tilde{S}_k)]\|_2\\
        &\leq \rho^{k-1}(f_1(s_0)+f_1(s_1))+\hat{\psi}\sqrt{d}\tag{Jensen's inequality and Assumption \ref{assump:feature_mat}}
    \end{align*}
    With the above bound, we have
    \begin{align*}
        \|\E_{y_0}[z_k\psi(S_k)^T]&-\E_{\mu}[\tilde{z}_k\psi(\tilde{S}_k)^T]\|_2\leq \sum_{j=0}^{k-1}\lambda^j\left\|\E_{y_0}[\psi(S_{k-j})\psi(S_k)^T]-\E_{\mu}[\psi(\tilde{S}_{k-j})\psi(\tilde{S}_k)^T]\right\|_2\\
        &~~~~+\lambda^k\|z_0\|_2\left(\rho^{k-1}(f_1(s_0)+f_1(s_1))+\hat{\psi}\sqrt{d}\right)+\sum_{j=k}^\infty\lambda^j\|\E_{\mu}[\psi(\tilde{S}_{k-j})\psi(\tilde{S}_k)^T]\|_{2}\\
        &\leq \sum_{j=0}^{k-1}\lambda^j\left\|\E_{y_0}[\psi(S_{k-j})\psi(S_k)^T]-\E_{\mu}[\psi(\tilde{S}_{k-j})\psi(\tilde{S}_k)^T]\right\|_2\\
        &~~~+\lambda^k\|z_0\|_2\left(\rho^{k-1}(f_1(s_0)+f_1(s_1))+\hat{\psi}\sqrt{d}\right)+\frac{\lambda^k\psi'}{1-\lambda}\tag{Part (b) of Lemma \ref{lem:benvan}}\\
        &\leq \sum_{j=0}^{k-1}\lambda^jf_1(s_1)\rho^{k-j-1}+\lambda^k\|z_0\|_2\left(\rho^{k-1}(f_1(s_0)+f_1(s_1))+\hat{\psi}\sqrt{d}\right)+\frac{\lambda^k\psi'}{1-\lambda}\tag{Assumption \ref{assump:stable_markov}}\\
        &\leq f_1(s_1)\rho^{-1}\sum_{j=0}^{k-1}\lambda^j\rho^{k-j}+\lambda^k\|z_0\|_2\left(\rho^{k-1}(f_1(s_0)+f_1(s_1))+\hat{\psi}\sqrt{d}\right)+\frac{\lambda^k\psi'}{1-\lambda}.
    \end{align*}
    Summing over all $k$, we get
    \begin{align*}
        \sum_{k=0}^\infty\|\E_{y_0}[z_k\psi(S_k)^T]-\E_{\mu}[\tilde{z}_k\psi(\tilde{S}_k)^T]\|_2&\leq \sum_{k=0}^\infty\Bigg(f_1(s_1)\rho^{-1}\sum_{j=0}^{k-1}\lambda^j\rho^{k-j}\\
        &~~~~+\lambda^k\|z_0\|_2\left(\rho^{k-1}(f_1(s_0)+f_1(s_1))+\hat{\psi}\sqrt{d}\right)+\frac{\lambda^k\psi'}{1-\lambda}\Bigg)\\
        &=f_1(s_1)\rho^{-1}\sum_{k=0}^\infty\left(\sum_{j=0}^{k-1}\lambda^j\rho^{k-j}\right)+\frac{\|z_0\|_2\rho^{-1}(f_1(s_0)+f_1(s_1))}{(1-\lambda\rho)}\\
        &~~~~+\frac{\|z_0\|_2\hat{\psi}\sqrt{d}}{(1-\lambda)}+\frac{\psi'}{(1-\lambda)^2}\\
        &=\frac{f_1(s_1)}{(1-\lambda)(1-\rho)}+\frac{\|z_0\|_2\rho^{-1}(f_1(s_0)+f_1(s_1))}{(1-\lambda\rho)}\\
        &~~~~+\frac{\|z_0\|_2\hat{\psi}\sqrt{d}}{(1-\lambda)}+\frac{\psi'}{(1-\lambda)^2}\tag{Fubini-Tonelli Theorem}\\
        &\leq g(y_0).
    \end{align*}

    \item It is easy to verify that an identical argument as in the previous part can be carried out for $\E_{y_0}[z_k\psi(S_{k+1})^T]-\E_{\mu}[\tilde{z}_k\psi(\tilde{S}_{k+1})^T]$. Thus to avoid repetition, we omit the proof for this part.

    \item We bound $j$-th element of the vector $\Psi^T\Lambda P^m\mathcal{R}_\pi$ as follows:
    \begin{align*}
        (\Psi^T\Lambda P^m\mathcal{R}_\pi)^2(j)&=\left(\sum_{s\in\mathcal{S}}\mu(s)\psi_j(s)\sum_{s'\in \mathcal{S}}P^{m}(s'|s)\mathcal{R}_{\pi}(s)\right)^2\\
        &\leq\left(\sum_{s\in\mathcal{S}}\mu(s)\psi_j^2(s)\right)\left(\sum_{s\in\mathcal{S}}\mu(s)\left(\sum_{s'\in \mathcal{S}}P^{m}(s'|s)\mathcal{R}_{\pi}(s)\right)^2\right)\tag{Cauchy-Schwartz inequality}\\
        &\leq \left(\sum_{s\in\mathcal{S}}\mu(s)\psi_j^2(s)\right)\left(\sum_{s\in\mathcal{S}}\mu(s)\sum_{s'\in \mathcal{S}}P^{m}(s'|s)(\mathcal{R}_{\pi}(s))^2\right)\tag{Jensen's inequality}\\
        &\leq \hat{\psi}^2\left(\sum_{s\in\mathcal{S}}\mu(s)(\mathcal{R}_{\pi}(s))^2\right)\tag{Assumption \ref{assump:feature_mat} and Fubini-Tonelli Theorem}\\
        &\leq \hat{\psi}^2\hat{r}^2.\tag{Assumption \ref{assump:dist_td}}
    \end{align*}
    Thus, the norm can be bounded as
    \begin{align*}
        \|\Psi^T\Lambda P^m\mathcal{R}_\pi\|_2\leq \hat{\psi}\hat{r}\sqrt{d}.
    \end{align*}
    Proceeding in a similar fashion as in part (c), we have
    \begin{align*}
        \E_{y_0}[z_k\mathcal{R}(S_k, A_k)]-\E_{\mu}[\tilde{z}_k\mathcal{R}(\tilde{S}_k, A_k)]&=\lambda^kz_0\E_{y_0}[\mathcal{R}(S_k, A_k)]+\E_{y_0}\left[\sum_{j=0}^{k-1}\lambda^j\psi(S_{k-j})\mathcal{R}(S_k, A_k)\right]\\
        &~~~~-\sum_{j=0}^\infty\lambda^j\E_{\mu}[\psi(\tilde{S}_{k-j})\mathcal{R}(\tilde{S}_k, A_k)]\\
        &=\lambda^kz_0\E_{y_0}[\mathcal{R}(S_k, A_k)]\\
        &~~~~+\E_{y_0}\left[\sum_{j=0}^{k-1}\lambda^j\left(\psi(S_{k-j})\mathcal{R}(S_k, A_k)-\E_{\mu}[\psi(\tilde{S}_{k-j})\mathcal{R}(\tilde{S}_k, A_k)]\right)\right]\\
        &~~~~-\sum_{j=k}^\infty\lambda^j\E_{\mu}[\psi(\tilde{S}_{k-j})\mathcal{R}(\tilde{S}_k, A_k)]
    \end{align*}
    Taking norm both sides and using triangle inequality, we get
    \begin{align*}
        \|\E_{y_0}[z_k\mathcal{R}(S_k, A_k)]-\E_{\mu}[\tilde{z}_k\mathcal{R}(\tilde{S}_k, A_k)]\|_2&\leq \lambda^k\|z_0\|_2|\E_{y_0}[\mathcal{R}(S_k, A_k)]|\\
        &~~~~+\left\|\E_{y_0}\left[\sum_{j=0}^{k-1}\lambda^j\left(\psi(S_{k-j})\mathcal{R}(S_k, A_k)-\E_{\mu}[\psi(\tilde{S}_{k-j})\mathcal{R}(\tilde{S}_k, A_k)]\right)\right]\right\|_2\\
        &~~~~+\sum_{j=k}^\infty\lambda^j\|\E_{\mu}[\psi(\tilde{S}_{k-j})\mathcal{R}(\tilde{S}_k, A_k)]\|_{2}
    \end{align*}
    To bound the first term, we use Assumption \ref{assump:stable_markov} to get
    \begin{align*}
        |\E_{y_0}[\mathcal{R}(S_k, A_k)]|&\leq |\E_{y_0}[\mathcal{R}(S_k, A_k)]-\E_{\mu}[\mathcal{R}_{\pi}(\tilde{S}_k)]|+|\E_{\mu}[\mathcal{R}_{\pi}(\tilde{S}_k)]|\\
        &\leq \rho^{k-1}(f_1(s_0)+f_1(s_1))+|\E_{\mu}[\mathcal{R}_{\pi}(\tilde{S}_k)]|\\
        &\leq \rho^{k-1}(f_1(s_0)+f_1(s_1))+\hat{r}\tag{Jensen's inequality and Assumption \ref{assump:dist_td}}
    \end{align*}
    With the above bound, we have

    \begin{align*}
        \|\E_{y_0}[z_k\mathcal{R}(S_k, A_k)]-\E_{\mu}[\tilde{z}_k\mathcal{R}(\tilde{S}_k, A_k)]\|_2&\leq \lambda^k\|z_0\|_2(\rho^{k-1}(f_1(s_0)+f_1(s_1))+\hat{r})\\
        &~~~~+\left\|\E_{y_0}\left[\sum_{j=0}^{k-1}\lambda^j\left(\psi(S_{k-j})\mathcal{R}(S_k, A_k)-\E_{\mu}[\psi(\tilde{S}_{k-j})\mathcal{R}(\tilde{S}_k, A_k)]\right)\right]\right\|_2\\
        &~~~~+\sum_{j=k}^\infty\lambda^j\|\E_{\mu}[\psi(\tilde{S}_{k-j})\mathcal{R}(\tilde{S}_k, A_k)]\|_{2}\\
        &\leq \lambda^k\|z_0\|_2(\rho^{k-1}(f_1(s_0)+f_1(s_1))+\hat{r})\\
        &~~~~+\sum_{j=0}^k\lambda^j\left\|\E_{y_0}[\psi(S_{k-j})\mathcal{R}(S_k, A_k)]-\E_{\mu}[\psi(\tilde{S}_{k-j})\mathcal{R}(\tilde{S}_k, A_k)]\right\|_2\\
        &~~~~+\frac{\lambda^{k}\hat{\psi}\hat{r}\sqrt{d}}{1-\lambda}\\
        &\leq \sum_{j=0}^{k-1}\lambda^jf_1(s_1)\rho^{k-j-1}+\lambda^k\|z_0\|_2(\rho^{k-1}(f_1(s_0)+f_1(s_1))+\hat{r})\\
        &~~~~+\frac{\lambda^{k}\hat{\psi}\hat{r}\sqrt{d}}{1-\lambda}\tag{Assumption \ref{assump:stable_markov}}\\
        &\leq f_1(s_1)\rho^{-1}\sum_{j=0}^{k-1}\lambda^j\rho^{k-j}+\lambda^k\|z_0\|_2(\rho^{k-1}(f_1(s_0)+f_1(s_1))+\hat{r})\\
        &~~~~+\frac{\lambda^{k}\hat{\psi}\hat{r}\sqrt{d}}{1-\lambda}.
    \end{align*}
    Similar to part (d), summing over all $k$, we get
    \begin{align*}
        \sum_{k=0}^\infty\|\E_{y_0}[z_k\mathcal{R}(S_k, A_k)]-\E_{\mu}[\tilde{z}_k\mathcal{R}(\tilde{S}_k, A_k)]\|_2 &\leq \frac{f_1(s_1)}{(1-\lambda)(1-\rho)}+\frac{\|z_0\|_2\rho^{-1}(f_1(s_0)+f_1(s_1))}{(1-\lambda\rho)}\\
        &~~~~+\frac{\|z_0\|_2\hat{r}}{(1-\lambda)}+\frac{\hat{\psi}\hat{r}\sqrt{d}}{(1-\lambda)^2}\\
        \\
        &\leq g(y_0).
    \end{align*}

    \item Using triangle inequality on the formula for $z_k$, we have
    \begin{align*}
        \|z_k\|_2&\leq \lambda^k\|z_0\|_2+\sum_{j=1}^k\lambda^{k-j}\|\psi(S_k)\|_2\\
        &= \frac{1-\lambda^{k+1}}{(1-\lambda)}\left(\frac{(1-\lambda)\lambda^k}{1-\lambda^{k+1}}\|z_0\|_2+\sum_{j=1}^k\frac{(1-\lambda)\lambda^{k-j}}{1-\lambda^{k+1}}\|\psi(S_j)\|_2\right)
    \end{align*}
    By taking fourth power both sides, we get
    \begin{align*}
        \|z_k\|^4_2&\leq \frac{(1-\lambda^{k+1})^4}{(1-\lambda)^4}\left(\frac{(1-\lambda)\lambda^k}{1-\lambda^{k+1}}\|z_0\|_2+\sum_{j=1}^k\frac{(1-\lambda)\lambda^{k-j}}{1-\lambda^{k+1}}\|\psi(S_j)\|_2\right)^4.
    \end{align*}
    Since the weights $\frac{(1-\lambda)\lambda^{k-j}}{1-\lambda^{k+1}}$ form a probability distribution, we can apply Jensen's inequality to get
    \begin{align*}
        \|z_k\|^4_2&\leq \frac{(1-\lambda^{k+1})^4}{(1-\lambda)^4}\left(\frac{(1-\lambda)\lambda^k}{1-\lambda^{k+1}}\|z_0\|^4_2+\sum_{j=1}^k\frac{(1-\lambda)\lambda^{k-j}}{1-\lambda^{k+1}}\|\psi(S_j)\|^4_2\right)\\
        &\leq \frac{1}{(1-\lambda)^3}\left(\lambda^k\|z_0\|_2^4+\sum_{j=1}^k\lambda^{k-j}\|\psi(S_j)\|^4_2\right).
    \end{align*}
    Taking expectation both sides conditioned on the initial state, we have
    \begin{align*}
        \E_{y_0}[\|z_k\|^4_2] &\leq \frac{1}{(1-\lambda)^3}\left(\lambda^k\|z_0\|_2^4+\sum_{j=1}^k\lambda^{k-j}\E_{y_0}[\|\psi(S_j)\|^4_2]\right)\\
        &\leq \frac{1}{(1-\lambda)^3}(\|z_0\|_2^4+f_3(s_1))\left(\sum_{j=0}^k\lambda^{k-j}\right)\tag{Assumption \ref{assump:stable_markov}}\\
        &\leq \frac{\|z_0\|_2^4+f_3(s_1)}{(1-\lambda)^4}.
    \end{align*}
    \item Recall $Y_k=(S_k, A_k, S_{k+1}, z_k)\in \mathcal{Y}$. Then, we have
    \begin{align*}
        g^2(Y_k)&\leq \frac{4f_1^2(S_{k+1})}{(1-\lambda)^2(1-\rho)^2}+\frac{4\|z_k\|^2_2\rho^{-2}(f_1(S_k)+f_1(S_{k+1}))^2}{(1-\lambda\rho)^2}\\
        &~~~+\frac{4\|z_k\|_2^2\left(1+\hat{r}+\hat{\psi}\sqrt{d}\right)^2}{(1-\lambda)^2}+\frac{4(\hat{\psi}+\psi'+\hat{\psi}\hat{r}\sqrt{d})^2}{(1-\lambda)^4}.\tag{Using $\left(\sum_{i=1}^na_i\right)^2\leq n\left(\sum_{i=1}^na_i^2\right)$}
    \end{align*}
    Taking expectation both sides, conditioned on initial state $y_0$,
    \begin{align*}
        \E_{y_0}[g^2(Y_k)]&\leq \frac{4\E_{y_0}[f_1^2(S_{k+1})]}{(1-\lambda)^2(1-\rho)^2}+\frac{4\rho^{-2}\E_{y_0}[\|z_k\|^2_2(f_1(S_k)+f_1(S_{k+1}))^2]}{(1-\lambda\rho)^2}+\frac{4\left(1+\hat{r}+\hat{\psi}\sqrt{d}\right)^2\E_{y_0}[\|z_k\|_2^2]}{(1-\lambda)^2}\\
        &~~~~+\frac{4(\hat{\psi}+\psi'+\hat{\psi}\hat{r}\sqrt{d})^2}{(1-\lambda)^4}.
    \end{align*}
    Using Assumption \ref{assump:stable_markov}, we can bound the first term as
    \begin{align*}
        \frac{4\E_{y_0}[f_1^2(S_{k+1})]}{(1-\lambda)^2(1-\rho)^2}\leq \frac{4\sqrt{f_2(s_1)}}{(1-\lambda)^2(1-\rho)^2}.\tag{Jensen's inequality}
    \end{align*}
    For second term, we use Cauchy-Schwartz inequality for expectations, to get
    \begin{align*}
        \frac{4\rho^{-2}\E_{y_0}[\|z_k\|^2_2(f_1(S_k)+f_1(S_{k+1}))^2]}{(1-\lambda\rho)^2}&\leq \frac{4\rho^{-2}\sqrt{\E_{y_0}[\|z_k\|^4_2]}\sqrt{\E_{y_0}[(f_1(S_k)+f_1(S_{k+1}))^4]}}{(1-\lambda\rho)^2}\\
        &\leq \frac{8\rho^{-2}\sqrt{\E_{y_0}[\|z_k\|^4_2]}\sqrt{\E_{y_0}[f_1^4(S_k)+f_1^4(S_{k+1})]}}{(1-\lambda\rho)^2}\tag{$(a+b)^2\leq 2a^2+2b^2$}\\
        &\leq \frac{16\rho^{-2}\sqrt{\|z_0\|_2^4+f_3(s_1)}\sqrt{f_2(s_0)+f_2(s_1)}}{(1-\lambda\rho)^2(1-\lambda)^2}\tag{Assumption \ref{assump:stable_markov}}
    \end{align*}
    For the third term, we use part (e) and Jensen's inequality to get
    \begin{align*}
        \frac{4\left(1+\hat{r}+\hat{\psi}\sqrt{d}\right)^2\E_{y_0}[\|z_k\|_2^2]}{(1-\lambda)^2}\leq \frac{4\left(1+\hat{r}+\hat{\psi}\sqrt{d}\right)^2\sqrt{\|z_0\|_2^4+f_3(s_1)}}{(1-\lambda)^4}.\tag{Jensen's inequality}
    \end{align*}
    The claim follows by combining all the bounds.
\end{enumerate}
    
\end{proof}
Define $\hat{T}(s_0, s_1)=c_\alpha^2+\frac{\sqrt{f_3(s_0)+f_3(s_1)}}{(1-\lambda)^2}+\frac{4(f_3(s_0)+f_3(s_1))}{(1-\lambda)^2}$, $\hat{b}(s_0, s_1)=c_\alpha^2\sqrt{f_3(s_0)+f_3(s_1)}+\frac{f_3(s_0)+f_3(s_1)}{(1-\lambda)^2}$ and $\hat{g}(s_0, s_1)=\E_{y_0}[g^2(s_0, a_0, s_1, \psi(s_0))]$. Furthermore, for ease of notation, we will denote $I$ as the identity matrix (infinite dimensional).
\begin{proposition}
    The TD($\lambda$) algorithm satisfies the following:
    \begin{enumerate}[(a)]
        \item The operator $F(x_k, Y_k)$ defined in Eq. \eqref{eq:linear_SA} has the following properties:
        \begin{enumerate}[(1)]
            \item $\|F(x_k, Y_k)\|_2\leq \|T(Y_k)\|_2\|x-x^*\|_2+\|b(Y_k)\|_2+\|T(Y_k)\|_2\|x^*\|_2$, where $\E_{y_0}[\|T(Y_k)\|^2_2]\leq \hat{T}(s_0, s_1)$ and $\E_{y_0}[(\|b(Y_k)\|^2_2]\leq \hat{b}(s_0, s_1)$.
            \item Define $\bar{F}(x)= \E_{Y\sim \mu}[F(x, Y)]$. Then, under Assumptions \ref{assump:dist_td} and \ref{assump:feature_mat}, $\bar{F}(x)$ exists and is given by $\bar{F}(x)=\bar{T}x+\bar{b}$.
            \item There exists a unique $\theta^*\in E_{\Psi}^{\perp}$ such that $x^*=(\bar{r}, \theta^{*T})^T$ solves $\bar{T}x+\bar{b}=0$. Furthermore, it is also one of the solutions to the Projected-Bellman equation $\mathcal{B}_{\pi}^{(\lambda)}(\Psi\theta)=\Psi\theta$. \end{enumerate}
        \item There exists a solution to the Poisson equation \eqref{eq:Poisson_eq} for the Markov chain $\mathcal{M}_Y$ which satisfies Assumption \ref{assump:Poisson_eq} with $\hat{A}_2^2(y_0)=9\hat{g}(s_0, s_1)$ and $\hat{B}_2^2(y_0)=2(9\|x^*\|_2^2+1)\hat{g}(s_0, s_1)+\frac{8c^2_\alpha\sqrt{f_2(s_0)+f_2(s_1)}}{(1-\rho)^2}$.
    \end{enumerate}
\end{proposition}
\begin{proof}
    \begin{enumerate}[(a)]
        \item 
        \begin{enumerate}[(1)]
            \item Since $T(Y_k)$ is partitioned in a block form, we use Lemma \ref{lem:block_norm} and the non-expansivity of the projection operator $\Pi_{2, E^{\perp}_{\Psi}}$ to get
        \begin{align*}
            \|T(Y_k)\|_2^2 &\leq c_\alpha^2+\|\Pi_{2, E^{\perp}_{\Psi}}z_k\|_2^2+\|\Pi_{2, E^{\perp}_{\Psi}}z_k(\psi(S_{k+1})^T-\psi(S_k)^T)\|_2^2\\
            &\leq c_\alpha^2+\|z_k\|_2^2+\|z_k(\psi(S_{k+1})^T-\psi(S_k)^T)\|_2^2\\
            &\leq c_\alpha^2+\|z_k\|_2^2+(\|z_k\|_2(\|\psi(S_{k+1})\|_2+\|\psi(S_k)\|_2))^2\\
            &\leq c_\alpha^2+\|z_k\|_2^2+2(\|z_k\|^2_2(\|\psi(S_{k+1})\|^2_2+\|\psi(S_k)\|^2_2))\tag{(Using $(a+b)^2\leq 2(a^2+b^2)$}
        \end{align*}
        Note that $z_{-1}=0$ implies $z_0=\psi(s_0)$. Taking expectation both sides, we get 
        \begin{align*}
            \E_{y_0}[\|T(Y_k)\|_2^2] &\leq c_\alpha^2+\E_{y_0}[\|z_k\|_2^2]+2\E_{y_0}[\|z_k\|_2^2\|\psi(S_{k+1})\|^2]+2\E_{y_0}[\|z_k\|_2^2\|\psi(S_k)\|_2^2]\nonumber\\
            &\leq c_\alpha^2+\E_{y_0}[\|z_k\|_2^2]+2\sqrt{\E_{y_0}[\|z_k\|_2^4]}\sqrt{\E_{y_0}[\|\psi(S_{k+1})\|_2^4]}+2\sqrt{\E_{y_0}[\|z_k\|_2^4]}\sqrt{\E_{y_0}[\|\psi(S_k)\|_4^2]}\tag{Cauchy-Schwartz for expectation}\nonumber\\
            &\leq c_\alpha^2+\frac{\sqrt{\|\psi(s_0)\|_2^4+f_3(s_1)}}{(1-\lambda)^2}+\frac{4\sqrt{\|\psi(s_0)\|_2^4+f_3(s_1)}\sqrt{f_3(s_0)+f_3(s_1)}}{(1-\lambda)^2}\tag{Part (e) Lemma \ref{lem:z_k_bound} and Assumption \ref{assump:stable_markov}}\nonumber\\
            &\leq c_\alpha^2+\frac{\sqrt{f_3(s_0)+f_3(s_1)}}{(1-\lambda)^2}+\frac{4(f_3(s_0)+f_3(s_1))}{(1-\lambda)^2}=\hat{T}(s_0, s_1).
        \end{align*}
        Next, we bound $b(Y_k)$
        \begin{align*}
            \|b(Y_k)\|_2^2&=c_\alpha^2\mathcal{R}^2(S_k, A_k)+\mathcal{R}^2(S_k, A_k)\|\Pi_{2, E^{\perp}_{\Psi}}z_k\|_2^2
        \end{align*}
        Again using the non-expansivity of the projection operator $\Pi_{2, E^{\perp}_{\Psi}}$ and taking expectation, we get
        \begin{align}
            \E_{y_0}[\|b(Y_k)\|_2^2]&=c_\alpha^2\E_{y_0}[\mathcal{R}^2(S_k, A_k)]+\E_{y_0}[\mathcal{R}^2(S_k, A_k)\|z_k\|_2^2]\nonumber\\
            &\leq c_\alpha^2\sqrt{\E_{y_0}[\mathcal{R}^4(S_k, A_k)]}+\sqrt{\E_{y_0}[\mathcal{R}^4(S_k, A_k)]}\sqrt{\E_{y_0}[\|z_k\|_4^2]}\tag{Cauchy-Schwartz for expectation}\nonumber\\
            &\leq c_\alpha^2\sqrt{f_3(s_0)+f_3(s_1)}+\frac{\sqrt{f_3(s_0)+f_3(s_1)}\sqrt{\|\psi(s_0)\|_2^4+f_3(s_1)}}{(1-\lambda)^2}.\tag{Part (e) Lemma \ref{lem:z_k_bound} and Assumption \ref{assump:stable_markov}}\nonumber\\
            &\leq c_\alpha^2\sqrt{f_3(s_0)+f_3(s_1)}+\frac{f_3(s_0)+f_3(s_1)}{(1-\lambda)^2}=\hat{b}(s_0, s_1).
        \end{align}
        Combining both the bounds, we have 
        \begin{align*}
            \|F(x_k, Y_k)\|_2\leq \|T(Y_k)\|_2\|x-x^*\|_2+\|b(Y_k)\|_2+\|T(Y_k)\|_2\|x^*\|_2,
        \end{align*}
        where $\E_{y_0}[\|T(Y_k)\|^2_2]\leq \hat{T}(s_0, s_1)$ and $\E_{y_0}[\|b(Y_k)\|_2^2]\leq \hat{b}(s_0, s_1)$.
        \item From Lemma \ref{lem:stationary_exp_Tb}, the stationary expectations of $T(Y_k)$ and $b(Y_k)$ are finite. Thus, 
        \begin{align*}
            \E_{\mu}[F(\tilde{Y}_k, x)]=\bar{T}x+\bar{b}.
        \end{align*}
        \item 
        
        \begin{itemize}
            \item $\nexists~\theta\in\mathbb{R}^d$ such that $\psi(s)^T\theta=1,~\forall s\in\mathcal{S}$: In this case, $E_{\Psi}^\perp\equiv\mathbb{R}^d$. Lemma \ref{lem:contraction_td} implies that all the eigenvalues of $\Psi^T\Lambda(P^{(\lambda)}-I)\Psi$ are strictly negative, immediately suggesting that $\Psi^T\Lambda(P^{(\lambda)}-I)\Psi$ is invertible. Thus, there exists a unique solution $\theta^*$
            \begin{align*}
                -\frac{\bar{r}}{(1-\lambda)}\Psi^T\mu+\Psi^T\Lambda(P^{(\lambda)}-I)\Psi\theta^*+\Psi^T\Lambda \mathcal{R}^{(\lambda)}=0.
            \end{align*}
            
            \item $\exists~\theta_e\in\mathbb{R}^d$ such that $\psi(s)^T\theta_e=1,~\forall s\in\mathcal{S}$: Note that due to linear independence of columns of  $\Psi$ and the irreducibility of $P$, $\theta_e$ is the unique left and right eigenvector of $\Psi^T\Lambda(P^{(\lambda)}-I)\Psi$ corresponding to eigenvalue $0$. This implies that all the other generalized eigenvectors of $\Psi^T\Lambda(P^{(\lambda)}-I)\Psi$ are perpendicular to $\theta_e$ and hence, they span $E_{\Psi}^\perp$. Furthermore, note that 
        \begin{align*}
            \theta_e^T\left(\Psi^T\Lambda \mathcal{R}^{(\lambda)}-\frac{\bar{r}}{(1-\lambda)}\Psi^T\mu\right)=\mu^T\mathcal{R}^{(\lambda)}-\frac{\bar{r}}{(1-\lambda)}=0.
        \end{align*}
        Thus, the vector $\Psi^T\Lambda \mathcal{R}^{(\lambda)}-\frac{\bar{r}}{(1-\lambda)}\Psi^T\mu$ is perpendicular to $\theta_e$ and therefore lies in $E_{\Psi}^\perp$. By the properties of generalized eigenvectors, it is easy to verify that there exists unique $\theta^*\in E_{\Psi}^\perp$ which satisfies
        \begin{align*}
            -\frac{\bar{r}}{(1-\lambda)}\Psi^T\mu+\Psi^T\Lambda(P^{(\lambda)}-I)\Psi\theta^*+\Psi^T\Lambda \mathcal{R}^{(\lambda)}=0.
        \end{align*}

        \end{itemize}
        Note that $\Pi_{2, E^{\perp}_{\Psi}}\Psi^T\Lambda(P^{(\lambda)}-I)\Psi\theta=\Psi^T\Lambda(P^{(\lambda)}-I)\Psi\theta$, for all $\theta\in E^{\perp}_{\Psi}$. Consider the expression $\bar{T}x^*+\bar{b}$. Expanding the matrix $\bar{T}$, we get
        \begin{align*}
            \bar{T}x^*+\bar{b}=\begin{bmatrix}
                -\bar{r}+\bar{r}\\
                -\frac{\bar{r}}{(1-\lambda)}\Pi_{2, E^{\perp}_{\Psi}}\Psi^T\mu+\Pi_{2, E^{\perp}_{\Psi}}\Psi^T\Lambda(P^{(\lambda)}-I)\Psi\theta^*+\Pi_{2, E^{\perp}_{\Psi}}\Psi^T\Lambda \mathcal{R}^{(\lambda)}
            \end{bmatrix}=0
        \end{align*}
        Thus, $x^*$ is the unique solution to $\bar{T}x^*+\bar{b}=0$ in the subspace $E^{\perp}_{\Psi}$. Furthermore, rearranging the terms in $\bar{T}x^*+\bar{b}=0$, we get
            \begin{align*}
                 \Psi^T\Lambda\Psi\theta^*=-\frac{\bar{r}}{(1-\lambda)}\Psi^T\mu+\Psi^T\Lambda P^{(\lambda)}\Psi\theta^*+\Psi^T\Lambda \mathcal{R}^{(\lambda)}.
            \end{align*}
            Multiplying both sides by $\Psi(\Psi^T\Lambda\Psi)^{-1}$ ($\Psi^T\Lambda\Psi$ is a $d\times d$ invertible matrix), we get
            \begin{align*}
                \Psi\theta^*&=\Psi(\Psi^T\Lambda\Psi)^{-1}\left(\Psi^T\Lambda \mathcal{R}^{(\lambda)}+\Psi^T\Lambda P^{(\lambda)}\Psi\theta^*-\frac{\bar{r}}{(1-\lambda)}\Psi^T\mu\right)\\
                \Psi\theta^*&=\Pi_{\Lambda, \Psi}\mathcal{B}_{\pi}^{(\lambda)}(\Psi\theta^*).
            \end{align*}
        Thus, $\theta^*$ is also one of the solutions for the Projected-Bellman equation for TD($\lambda$).
        \end{enumerate}
        
        \item We will use similar arguments as in Lemma 1 of Chapter 2, Part 2 from \cite{benveniste2012} to show the existence of a solution to the Poisson equation. Define $V_x(y)$ for all $y=(s, a, s', z)\in\mathcal{Y}$ as follows:
        \begin{align*}
            V_x(y)=\left(\sum_{k=0}^\infty(\E_y[T(Y_k)]-\bar{T})\right)x+\sum_{k=0}^\infty(\E_y[b(Y_k)]-\bar{b})
        \end{align*}
        Then, using Lemma \ref{lem:z_k_bound} we can bound  each infinite summation as follows:
        \begin{align*}
            \left\|\sum_{k=0}^\infty(\E_y[T(Y_k)]-\bar{T})\right\|_2&\leq \sum_{k=0}^\infty\left\|\E_y[T(Y_k)]-\bar{T}\right\|_2\\
            &\leq \sum_{k=0}^\infty\Bigg(\|\Pi_{2, E^{\perp}_{\Psi}}\left(\E_y[z_k]-\E_{\mu}[\tilde{z}_k]\right)\|_2+\left\|\Pi_{2, E^{\perp}_{\Psi}}\left(\E_y[z_k\psi(S_k)^T]-\E_{\mu}[\tilde{z}_k\psi(\tilde{S}_k)^T\right)]\right\|_2\\
            &~~~+\left\|\Pi_{2, E^{\perp}_{\Psi}}\left(\E_y[z_k\psi(S_{k+1})^T]-\E_{\mu}[\tilde{z}_k\psi(\tilde{S}_{k+1})^T]\right)\right\|_2\Bigg)\tag{Lemma \ref{lem:block_norm} and triangle inequality}\\
            &\leq 3g(y).\tag{Lemma \ref{lem:z_k_bound}}
        \end{align*}
        Next, for the second summation, we have
        \begin{align*}
            \left\|\sum_{k=0}^\infty(\E_y[b(Y_k)]-\bar{b})\right\|_2&\leq \sum_{k=0}^\infty\left\|\E_y[b(Y_k)]-\bar{b}\right\|_2\\
            &\leq \sum_{k=0}^\infty\Bigg(c_\alpha|\E_{y}[\mathcal{R}(S_k, A_k)]-\E_{\mu}[\mathcal{R}(\tilde{S}_k, A_k)]|\\
            &~~~~+\left\|\Pi_{2, E^{\perp}_{\Psi}}\left(\E_{y}[z_k\mathcal{R}(S_k, A_k)]-\E_{\mu}[\tilde{z}_k\mathcal{R}(\tilde{S}_k, A_k)]\right)\right\|_2\Bigg)\\
            &\leq \sum_{k=0}^\infty c_\alpha\rho^k(f_1(s)+f_1(s'))+g(y)\tag{Assumption \ref{assump:stable_markov} and Lemma \ref{lem:z_k_bound}}\\
            &=\frac{c_\alpha(f_1(s)+f_1(s'))}{1-\rho}+g(y).
        \end{align*}
        Thus, both the series are convergent. Furthermore, note that Assumption \ref{assump:stable_markov} implies that $\E_{y}[f_1(S_k)]\leq \sqrt[4]{f_2(s)+f_2(s')}$, for all $y\in \mathcal{Y}$ and $k\geq 0$. Thus, following dominated convergence theorem, 
        \begin{align*}
            \E_y[V_x(Y_1)]&=\E_y\left[\left(\sum_{k=1}^\infty(\E_{Y_1}[T(Y_k)]-\bar{T})\right)x+\sum_{k=1}^\infty(\E_{Y_1}[b(Y_k)]-\bar{b})\right]\\
            &=\left(\sum_{k=1}^\infty(\E_y[\E_{Y_1}[T(Y_k)]-\bar{T}])\right)x+\sum_{k=1}^\infty(\E_y[\E_{Y_1}[b(Y_k)]-\bar{b}])\\
            &=\left(\sum_{k=1}^\infty(\E_y[T(Y_k)]-\bar{T}])\right)x+\sum_{k=1}^\infty(\E_y[b(Y_k)]-\bar{b}])\\
            &=V_y(x)-((T(y)-\bar{T})x+b(y)-\bar{b}).
        \end{align*}
       
       The claim follows. Now to show bounded expectations, we use part (f) of Lemma \ref{lem:z_k_bound} and the fact that $z_0=\psi(s_0)$, to get
       \begin{align*}
           \hat{A}_2^2(y_0)&=\E_{Y_0=(s_0, a_0, s_1, z_0)}\left[\left\|\sum_{k=0}^\infty(\E_y[T(Y_k)]-\bar{T})\right\|^2_2\right]\leq 9\hat{g}(s_0, s_1),\\
           \hat{B}_2^2(y_0)&=\E_{Y_0=(s_0, a_0, s_1, z_0)}\left[\left\|V_{x^*}(Y_k)\right\|_2^2\right]\leq \E_{Y_0=(s_0, a_0, s_1, z_0)}\left[\left(3g(Y_k)x^*+\frac{c_\alpha(f_1(S_k)+f_1(S_{k+1}))}{1-\rho}+g(Y_k)\right)^2\right],\\
            &\quad\quad\quad\quad\quad\quad\quad\quad\quad\quad\quad\quad\quad\quad\leq 2(9\|x^*\|_2^2+1)\hat{g}(s_0, s_1)+\frac{8c^2_\alpha\sqrt{f_2(s_0)+f_2(s_1)}}{(1-\rho)^2}.
       \end{align*}
    \end{enumerate}
\end{proof}

\subsection{Proof of Lemma \ref{lem:contraction_td}}\label{appendix:contraction_td}
\begin{proof}
    The proof largely involves similar arguments as in Lemma 2 in \cite{zhang2021} but we will also need Lemma \ref{lem:strict_drift} to adapt the infinite state space. Since $P^{(\lambda)}$ is an irreducible and aperiodic Markov kernel, for any non-zero $\theta\in E^{\perp}_{\Psi}$, by Lemma \ref{lem:strict_neg} we have 
    \begin{align*}
        \theta^T(\Psi^T\Lambda\Psi-\Psi^T\Lambda P^{(\lambda)}\Psi)\theta>0.
    \end{align*}
    Consider the set $\{\theta\in E^{\perp}_{\Psi}|\|\theta\|_2=1\}$. Note that this set is compact and closed, thus by the extreme value theorem, we have
    \begin{align*}
        \Delta:=\min_{\theta\in E^{\perp}_{\Psi},\|\theta\|_2=1}\theta^T(\Psi^T\Lambda\Psi-\Psi^T\Lambda P^{(\lambda)}\Psi)\theta>0.
    \end{align*}
    By Lemma \ref{lem:stationary_exp_Tb}, in steady state $\E_{\mu}[T(Y_k)]$ is given by
    \begin{align*}
        \bar{T}=\begin{bmatrix}
        -c_\alpha & 0\\
        -\frac{1}{(1-\lambda)}\Pi_{2, E^{\perp}_{\Psi}}\Psi^T\mu & \Pi_{2, E^{\perp}_{\Psi}}(\Psi^T\Lambda\Psi-\Psi^T\Lambda P^{(\lambda)}\Psi)
        \end{bmatrix}.
    \end{align*}
    Thus, the minimization problem $\min_{x\in \mathbb{R}\times E^{\perp}_{\Psi},\|x\|_2=1}-x^T\bar{T}x$ can be written as
    \begin{align*}
        \min_{\theta\in E^{\perp}_{\Psi},r\in \mathbb{R},r^2+\|\theta\|^2_2=1}c_\alpha r^2+\frac{r}{1-\lambda}\theta^T\Pi_{2, E^{\perp}_{\Psi}}\Psi^T\mu +\theta^T\Pi_{2, E^{\perp}_{\Psi}}(\Psi^T\Lambda\Psi-\Psi^T\Lambda P^{(\lambda)}\Psi)\theta.
    \end{align*}
    Since $\theta\in E^{\perp}_{\Psi}$, $\theta^T\Pi_{2, E^{\perp}_{\Psi}}=(\Pi_{2, E^{\perp}_{\Psi}}\theta)^T=\theta$, we have
    \begin{align*}
        \min_{\theta\in E^{\perp}_{\Psi},r\in \mathbb{R},r^2+\|\theta\|^2_2=1}c_\alpha r^2+\frac{r}{1-\lambda}\theta^T\Psi^T\mu +\theta^T(\Psi^T\Lambda\Psi-\Psi^T\Lambda P^{(\lambda)}\Psi)\theta.
    \end{align*}
    First, we bound the second term.
    \begin{align*}
        \left|\frac{r}{1-\lambda}\theta^T\Psi^T\mu\right|&=\frac{|r|}{1-\lambda}\left|\theta^T\Psi^T\mu\right|\\
        &\leq \frac{|r|}{1-\lambda}\|\theta\|_{2}\|\Psi^T\mu\|_2.\tag{Cauchy-Schwartz Inequality}
    \end{align*}
    Since $\|\Psi^T\mu\|^2_2=\sum_{i=1}^d\left(\sum_{s\in\mathcal{S}}\mu(s)\psi_i(s)\right)^2$ which by Jensen's inequality can be bounded as 
    \begin{align*}
        \|\Psi^T\mu\|^2_2&\leq \sum_{i=1}^d\sum_{s\in\mathcal{S}}\mu(s)\psi_i^2(s)\\
        &\leq d\hat{\psi}^2.
    \end{align*}
    Thus, 
    \begin{align*}
        \left|\frac{r}{1-\lambda}\theta^T\Psi^T\mu\right|
        &\leq \frac{\hat{\psi}|r|\|\theta\|_{2}\sqrt{d}}{1-\lambda},~~\forall r\in \mathbb{R},~\theta\in E^{\perp}_{\Psi}.
    \end{align*}
    and
    \begin{align*}
        \theta^T(\Psi^T\Lambda\Psi-\Psi^T\Lambda P^{(\lambda)}\Psi)\theta\geq \Delta\|\theta\|_2^2,~~\forall\theta\in E^{\perp}_{\Psi}.
    \end{align*}
    Combining all the bounds, we get
    \begin{align*}
        &\min_{\theta\in E^{\perp}_{\Psi},r\in \mathbb{R},r^2+\|\theta\|^2_2=1}c_\alpha r+\frac{r}{1-\lambda}\theta^T\Psi^T\mu +\theta^T(\Psi^T\Lambda\Psi-\Psi^T\Lambda P^{(\lambda)}\Psi)\theta\\
        &\geq \min_{\theta\in E^{\perp}_{\Psi},r\in \mathbb{R},r^2+\|\theta\|^2_2=1}c_\alpha r^2-\frac{\hat{\psi}|r|\|\theta\|_{2}\sqrt{d}}{1-\lambda}+\Delta\|\theta\|_2^2\\
        &=\min_{r\in[-1,1]}c_\alpha|r|^2-\frac{\sqrt{d}\hat{\psi}|r|\sqrt{1-|r|^2}}{1-\lambda}+\Delta(1-r^2)\\
        &=\min_{z\in[0,1]}c_\alpha z-\frac{\sqrt{d}\hat{\psi}\sqrt{z(1-z)}}{1-\lambda}+\Delta(1-z)\\
        &=\Delta+\min_{z\in[0,1]}(c_\alpha-\Delta)z-\frac{\sqrt{d}\hat{\psi}\sqrt{z(1-z)}}{1-\lambda}.
    \end{align*}
    When $c_\alpha\geq \Delta+\sqrt{\frac{d^2\hat{\psi}^4}{\Delta^2(1-\lambda)^4}-\frac{d\hat{\psi}^2}{(1-\lambda)^2}}$, we have
    \begin{align*}
        \min_{z\in[0,1]}(c_\alpha-\Delta)z-\frac{\sqrt{d}\hat{\psi}\sqrt{z(1-z)}}{1-\lambda}=&\frac{1}{2}\left((c_\alpha-\Delta)-\sqrt{(c_\alpha-\Delta)^2+\frac{d^2\hat{\psi}^2}{(1-\lambda)^2}}\right)\\
        &\geq \frac{1}{2}\left(\sqrt{\frac{d^2\hat{\psi}^4}{\Delta^2(1-\lambda)^4}-\frac{d\hat{\psi}^2}{(1-\lambda)^2}}-\frac{d\hat{\psi}^2}{\Delta(1-\lambda)^2}\right).\\
        &\geq -\frac{\Delta}{2}
    \end{align*}
    where for the last inequality we used the following fact
    \begin{align*}
        \sqrt{\frac{x^2}{\Delta^2}-x}-\frac{x}{\Delta}\geq -\Delta~~\forall x.\tag{$x=\Delta^2$ is the minimizer}
    \end{align*}
    Therefore, it follows that
    \begin{align*}
        \min_{x\in \mathbb{R}\times E^{\perp}_{\Psi},\|x\|_2=1}-x^TTx\geq \frac{\Delta}{2}.
    \end{align*}
\end{proof}

\subsection{Proof of Theorem \ref{thm:td}}\label{appendix:td_thm}
\begin{proof}
    Since there is no martingale noise in the algorithm $A_3=B_3=0$. From Proposition \ref{prop:td_prop}, we have $\hat{A}_1^2(y_0)=\hat{T}(s_0, s_1)$ and $\hat{A}_2^2(y_0)=9\hat{g}(s_0, s_1)$ which gives 
    \begin{align*}
        \hat{A}(y_0)=\hat{A}_1^2(y_0)+\hat{A}_2^2(y_0)+A_3^2=\hat{T}(s_0, s_1)+9\hat{g}(s_0, s_1).
    \end{align*}
    Next, $\E_{y_0}[(\|b(Y_k)\|^2_2]\leq \hat{b}(s_0, s_1)$, we have
    \begin{align*}
        \E_{y_0}[(\|b(Y_k)\|_2+\|T(Y_k)\|_2\|x^*\|_2)^2]\leq 2\hat{b}(s_0, s_1)+2\hat{T}(s_0, s_1)\|x^*\|_2^2.
    \end{align*}
    Combining above with $B_2^2(y_0)$, we get
    \begin{align*}
        \hat{B}(y_0)^2&=\hat{B}_1^2(y_0)+\hat{B}_2^2(y_0)+B_3^2\\
        &=2\hat{b}(s_0, s_1)+2\hat{g}(s_0, s_1)+2\|x^*\|_2^2\left(9\hat{g}(s_0, s_1)+\hat{T}(s_0, s_1)\right)+\frac{8c^2_\alpha\sqrt{f_2(s_0)+f_2(s_1)}}{(1-\rho)^2}.
    \end{align*}
    Let $\max_{x\in \mathcal{X}}\|x\|_2\leq M/2$. Then, $\hat{C}_V(s_0, s_1)=\hat{C}(y_0)=\hat{A}(y_0)M^2+\hat{B}(y_0)$. Since $\|\cdot\|_c=\|\cdot\|_s=\|\cdot\|_2$, we have
    \begin{align*}
        \varphi_{1}=\frac{uL_su_{2s}u^2_{cs}}{l_{2s}}=1;~\varphi_{V, 0}=(\bar{r}_0-r^*)^2+\|\theta_0-\theta^*\|_2^2+2\hat{C}_V(s_0, s_1).
    \end{align*}
\end{proof}

\subsection{Proof of Theorem \ref{thm:LFA_error}}\label{sec:LFA_error}

\begin{proof}
    \begin{enumerate}[(a)]
        \item By definition $P_{\delta}^{(\lambda)}e=e$ and $\bar{\Pi}_{\Lambda,\Psi}e=0$. Thus,  $\|\bar{\Pi}_{\Lambda,\Psi}P_{\delta}^{(\lambda)}x\|_{\Lambda}=\|\bar{\Pi}_{\Lambda,\Psi}P_{\delta}^{(\lambda)}(x-(\mu^Tx)e)\|_{\Lambda}$ for any function $x$. Therefore, without loss of generality, we will assume $\mu^Tx=0$ and $\|x\|_{\Lambda}\leq 1$. Furthermore, since $\bar{\Pi}_{\Lambda,\Psi}$ is a projection operator, it is non-expansive under $\|\cdot\|_{\Lambda}$ norm which leads us to
        \begin{align*}
            \|\bar{\Pi}_{\Lambda,\Psi}P_{\delta}^{(\lambda)}x\|^2_{\Lambda}&\leq \|P_{\delta}^{(\lambda)}x\|^2_{\Lambda}\\
            &=(1-\delta)^2\|x\|_\Lambda^2+2\delta(1-\delta) x^T \Lambda P^{(\lambda)}x + \delta^2\|P^{(\lambda)}x\|_{\Lambda}^2\\
            &\leq (1-\delta)^2+2\delta(1-\delta)x^T \Lambda P^{(\lambda)}x + \delta^2\tag{$P^{(\lambda)}$ is non-expansive under $\|\cdot\|_{\Lambda}$}\\
            &\leq (1-\delta)^2+2\delta(1-\delta)(1-\nu_\lambda) + \delta^2\\
            &\leq 1-2\delta(1-\delta)\nu_\lambda.
        \end{align*}
        Note that $\max \delta(1-\delta)=1/4$. Thus, $\gamma_{\lambda}\leq \sqrt{1-\nu_{\lambda}/2}$.
        \item  The proof for this part follows using exact arguments as in Theorem 3 in \cite{tsitsiklis1999} and is therefore omitted.
    \end{enumerate}
\end{proof}

\subsection{Proof of Lemma \ref{lem:rev-td-error}}
\begin{proof}
    Using the results in \cite[Chapter 22]{douc2018markov}, it is easy to verify that for all $m\geq 0$
    \begin{align*}
        \max_{\substack{\mu^Tx=0}}
        \left|\frac{x^T\Lambda P^mx}{x^T\Lambda x}\right|=(1-\nu)^m.
    \end{align*}
    Thus, using the expression for $P^{(\lambda)}$ from Lemma \ref{lem:stationary_exp_Tb}, we have
    \begin{align*}
        1-\nu_{\lambda}=\max_{\substack{\mu^Tx=0}}
        \left|\frac{x^T\Lambda P^{(\lambda)}x}{x^T\Lambda x}\right|&\leq (1-\lambda)\sum_{m=0}^{\infty}\lambda^m\left|\max_{\substack{\mu^Tx=0}}
        \frac{x^T\Lambda P^mx}{x^T\Lambda x}\right|\\
        &= (1-\lambda)(1-\nu)\sum_{m=0}^{\infty}(\lambda(1-\nu))^m=\frac{(1-\lambda)(1-\nu)}{(1-(1-\nu)\lambda)}<1\\
        \implies 0<\nu_{\lambda}&=\frac{\nu}{1-(1-\nu)\lambda}.
    \end{align*}
    Note that if $P$ is reversible then so is $P^m$ due to the self-adjoint property. This implies that $P^{(\lambda)}$ is also a reversible Markov operator. Thus, for any $x$ such that $\|x\|_{\Lambda}=1$ and $\mu^Tx=0$, we get
    \begin{align*}
        \|\bar{\Pi}_{\Lambda,\Psi}P_{\delta}^{(\lambda)}x\|^2_{\Lambda}&\leq \|P_{\delta}^{(\lambda)}x\|^2_{\Lambda}\\
            &=(1-\delta)^2\|x\|_\Lambda^2+2\delta(1-\delta) x^T \Lambda P^{(\lambda)}x + \delta^2\|P^{(\lambda)}x\|_{\Lambda}^2\\
            &\leq (1-\delta)^2+2\delta(1-\delta)(1-\nu_{\lambda}) + (1-\nu_{\lambda})^2\delta^2\tag{\cite[Theorem 22.A.17]{douc2018markov}}\\
            &= ((1-\delta)+\delta(1-\nu_{\lambda}))^2\\
            &\leq \left(1-\delta\nu_{\lambda}\right)^2=\left(1-\frac{\delta\nu}{(1-(1-\nu)\lambda)}\right)^2.
    \end{align*}
    The infimum of the r.h.s. is obtained at $\delta=1$ whereupon the first claim follows. For the second claim, we have
    \begin{align*}
        \lim_{\lambda \uparrow 1}\gamma_{\lambda}=\lim_{\lambda \uparrow 1}\inf_{\delta}\|\bar{\Pi}_{\Lambda,\Psi}P_{\delta}^{(\lambda)}x\|_{\Lambda}\leq \lim_{\lambda\uparrow 1}\frac{(1-\lambda)(1-\nu)}{(1-\lambda(1-\nu))}=0.
    \end{align*}
\end{proof}

\subsection{Challenges in the infinite state space}\label{appendix:challenge}
Consider a birth-death chain. Let the state space be given as $\mathcal{S}=\{s_i\}_{i\geq 0}$ with the transition kernel $P(s_{i+1}|s_i)=p$ and $P(s_{i-1}|s_i)=1-p$, where $p<1/2$. Furthermore, $P(s_0|s_0)=1-p$. It is well known that for $p<1/2$, this chain is positive recurrent and the stationary distribution is given by $\mu(s_i)=\frac{(1-2p)p^i}{(1-p)^{i+1}}$. For simplicity, we will consider the setting when $\lambda=0$, which implies $P^{(0)}=P$. Consider a sequence of functions $\{V_j\}_{j\geq 1}$ in the set $\{V|\sum_{s\in \mathcal{S}}V(s)=0, \sum_{s\in \mathcal{S}}
V^2(s)=1\}$ that satisfy the following: 
\begin{align*}
    V_j(s_i)=\begin{cases}
        \frac{1}{\sqrt{2}},~~i=j\\
        -\frac{1}{\sqrt{2}},~~i=j+1\\
        0,~~\text{otherwise}.
    \end{cases}
\end{align*}
Then, we have the following:
\begin{align*}
    V_j^T\Lambda(I-P)V_j&=\frac{1}{2}\E_{\mu}[(V_j(S_{k+1})-V_j(S_k))^2]\\
    &=\mu(s_{j-1})p\left(\frac{1}{\sqrt{2}}\right)^2+\mu(s_j)p\left(\sqrt{2}\right)^2+\mu(s_{j+1})(1-p)\left(\sqrt{2}\right)^2+\mu(s_{j+2})(1-p)\left(\frac{1}{\sqrt{2}}\right)^2\\
    &=\frac{(1-2p)p^j}{(1-p)^j}\left(\frac{1}{2}+\frac{4(1-2p)p}{(1-p)}+\frac{(1-2p)p^2}{2(1-p)^2}\right).
\end{align*}
Note that $p<1/2$, therefore $\lim_{j\to\infty}V_j^T\Lambda(P-I)V_j\to 0$. Thus, 
\begin{align*}
    \inf_V V^T\Lambda(I-P)V=0.
\end{align*}

Observe that in the above example the vector has infinite dimension, thus the point at which the function value has a variation for the first time can drift to infinity. However, by using linear function approximation with a finite number of columns, we are essentially restricting the function in a finite-dimensional setting. More concretely, in Lemma \ref{lem:strict_drift} we establish that for any function given by a linear combination of columns of $\Psi$, the point of variation in the value of the function cannot drift to infinity.

\subsection{Auxiliary Lemmas for TD\texorpdfstring{$(\lambda)$} \ }\label{appendix:aux_lem_td}
\begin{lemma}\label{lem:strict_drift}
    Let $\mathcal{S}=\{s_1,s_2,s_3,\dots\}$ be an indexing of the state space such that $\psi(s_i)^T$ is the $i$-th row in $\Psi$. Then, under the Assumption \ref{assump:feature_mat}, there exists a finite $N$ such that for all $\theta\in E_{\Psi}^{\perp}$ the following relation holds
    \begin{align}\label{eq:strict_drift}
        \sum_{j=1}^d\theta_j\psi_j(s_M)\neq \sum_{j=1}^d\theta_j\psi_j(s_i),~1\leq i\leq N-1.
    \end{align}
\end{lemma}
\begin{proof}
    From Lemma \ref{lem:rank_lemma}, there exists $N_1$ such that the span of the first $N_1$ rows of $\Psi$ is $\mathbb{R}^d$. Therefore, there exists a set of $d$ vectors that are linearly independent. Denote these row vectors by $\{\psi(s_{i_1}), \psi(s_{i_2}),\dots, \psi(s_{i_d}) \}$ and construct a matrix $\hat{\Psi}_d$ by concatenating these row vectors. Let $e\in \mathbb{R}^d$ be the vector of all ones. Now, we have two cases: $(i)$ $E_{\Psi}$ is non-empty, or $(ii)$ $E_{\Psi}$ is empty. We will consider these two cases separately.
    \begin{itemize}
        \item $E_{\Psi}$ is non-empty:  Then, $\hat{\Psi}_d\theta\neq e$ for any $\theta\in E_{\Psi}^{\perp}$ since $\hat{\Psi}_d$ is full rank. The claim follows immediately.
    
        \item$E_{\Psi}$ is empty: In this case $E_{\Psi}^{\perp}\equiv\mathbb{R}^d$. Let $\theta_e$ be the vector for which we have $\Psi_d\theta_e=e$. Then, from Assumption \ref{assump:feature_mat}, there exists a finite $N_2$ such that $\sum_{j=1}^d\theta_e\psi_j(s_{N_2})\neq 1$. Furthermore, $\Psi_d\theta\neq e$ for any $\theta\neq \theta_e$. Thus, for all $\theta\in \mathbb{R}^d$ , $\sum_{j=1}^d\theta_e\psi_j(s_{N})\neq \sum_{j=1}^d\theta_e\psi_j(s_i)$, where $N=\max\{N_1, N_2\}$. 
        \end{itemize}
\end{proof}

\begin{lemma}\label{lem:strict_neg}
Let $P$ be the transition kernel for an irreducible and aperiodic Markov chain. Then, for any $\theta\in E_{\Psi}^{\perp}$, the following is true:
\begin{align}
    \theta^T\Psi^T\Lambda(I-P)\Psi\theta>0
\end{align}
\end{lemma}
\begin{proof}
    Let $\mathcal{V}(s_i)=\psi(s_i)^T\theta$ be a non-constant function of the states, where $\psi(s_i)^T$ is the $i$-th row of $\Psi$. Note that due to Lemma \ref{lem:strict_drift}, there exists a finite $N$ where $\mathcal{V}(s_{N})\neq \mathcal{V}(s_{N-1})$. Since the Markov chain is irreducible and $\mathcal{V}(\cdot)$ is a non-constant function of time, we have
    \begin{align*}
        0&<\frac{1}{2}\sum_{i=1}^\infty\mu(s_i)\sum_{s\in \mathcal{S}}P(s|s_i)(\mathcal{V}(s_i)-\mathcal{V}(s))^2\tag{$P(s_{i+1}|s_i)>0$ by construction of $\Psi$}\\
        &=\sum_{i=1}^\infty \mu(s_i)\left(\mathcal{V}^2(s_i)-\mathcal{V}(s_i)\sum_{s\in \mathcal{S}}P(s|s_i)\mathcal{V}(s)\right)\\
        &=\theta^T\Psi^T\Lambda(I-P)\Psi\theta.
    \end{align*}
\end{proof}

\section{Proof of technical results in Section \ref{sec:GLR}}\label{appendix:GLR}
Before starting the proof for the results in this section, we first establish some properties of $\phi g(\phi^Tx)$. Let $\phi_1$ and $\phi_2$ be any two regressors. Then, for all $x\in \mathbb{R}^d$, we have
\begin{align}\label{eq:lipschitz_phig}
        \|\phi_1g(\phi_1^Tx)-\phi_2g(\phi_2^Tx)\|_2&=\|\phi_1g(\phi_1^Tx)-\phi_2g(\phi_1^Tx)+\phi_2g(\phi_1^Tx)-\phi_2g(\phi_2^Tx)\|_2\nonumber\\
        &\leq |g(\phi_1^Tx)|\|\phi_1-\phi_2\|_2+\|\phi_2\|_2|g(\phi_1^Tx)-g(\phi_2^Tx)|\nonumber\\
        &\leq \|\phi_1-\phi_2\|_2\left(|g(\phi_1^Tx)|+L\|\phi_2\|\|x\|\right)\nonumber\\
        &\leq \|\phi_1-\phi_2\|_2\left(|g(0)|+L\|x\|_2(\|\phi_1\|_2+\|\phi_2\|_2)\right).
\end{align}

In addition, for any $x_1, x_2\in \mathbb{R}^d$, we have
\begin{align*}
    \|\phi_1g(\phi_1^Tx_1)&-\phi_1g(\phi_1^Tx_2)-(\phi_2g(\phi_2^Tx_1)-\phi_2g(\phi_2^Tx_2))\|_2=\|\phi_1g(\phi_1^Tx_1)-\phi_2g(\phi_1^Tx_1)+\phi_2g(\phi_1^Tx_1)-\phi_1g(\phi_1^Tx_2)\\
     &~~-(\phi_2g(\phi_2^Tx_1)-\phi_2g(\phi_1^Tx_2)+\phi_2g(\phi_1^Tx_2)-\phi_2g(\phi_2^Tx_2))\|_2\\
     &\leq \|(\phi_1-\phi_2)g(\phi_1^Tx_1)-(\phi_1-\phi_2)g(\phi_1^Tx_2)\|_2\\
     &~~\|\phi_2g(\phi_1^Tx_1)-\phi_2g(\phi_2^Tx_1)-\phi_2g(\phi_1^Tx_2)+\phi_2g(\phi_2^Tx_2)\|_2\\
     &\leq \|\phi_1-\phi_2\|_2\|g(\phi_1^Tx_1)-g(\phi_1^Tx_2)\|_2+\|\phi_2\|_2\|g(\phi_1^Tx_1)-g(\phi_2^Tx_1)-g(\phi_1^Tx_2)+g(\phi_2^Tx_2)\|_2\\
     &\leq  L\|\phi_1-\phi_2\|_2\|x_1-x_2\|_2\|\phi_1\|_2+\|\phi_2\|_2\|g(\phi_1^Tx_1)-g(\phi_2^Tx_1)-g(\phi_1^Tx_2)+g(\phi_2^Tx_2)\|_2
\end{align*}
To simplify the last term, we use the almost everywhere differentiability of $g(\cdot)$ and the Fundamental Theorem of Lesbesgue Integral Calculus, to get
\begin{align*}
    g(\phi_1^Tx_1)-g(\phi_2^Tx_1)&-g(\phi_1^Tx_2)+g(\phi_2^Tx_2))=(\phi_1-\phi_2)^T\\
    &~~\times \int_0^1\nabla_{\phi}g((\phi_2+t(\phi_1-\phi_2))^Tx_1)-\nabla_{\phi}g((\phi_2+t(\phi_1-\phi_2))^Tx_2)dt
\end{align*}
Using Lipschitz property of $\nabla_{\phi}g(\cdot)$, we obtain
\begin{align*}
    \|g(\phi_1^Tx_1)-g(\phi_2^Tx_1)-g(\phi_1^Tx_2)+g(\phi_2^Tx_2))\|_2&\leq \|\phi_1-\phi_2\|_2\int_0^1L\|x_1-x_2\|_2dt\\
    &\leq L\|\phi_1-\phi_2\|_2\|x_1-x_2\|_2.
\end{align*}
Combining the above bounds, we have
\begin{align}\label{eq:lipschitz_phig_strong}
    \|\phi_1g(\phi_1^Tx_1)-\phi_1g(\phi_1^Tx_2)-(\phi_2g(\phi_2^Tx_1)&-\phi_2g(\phi_2^Tx_2))\|_2\nonumber\\
    &~~\leq L\|\phi_1-\phi_2\|_2\|x_1-x_2\|_2(\|\phi_1\|_2+\|\phi_2\|_2).
\end{align}
Now we state the following intermediate lemmas that will enable us to verify the properties in Proposition \ref{prop:LMS}.
\begin{lemma}\label{lem:benveniste}\cite[Chapter 2, Part 2]{benveniste2012}
    Denote $\{Y_k\}_{k\geq 0}$ as the Markov chain and $\mathcal{S}$ as its state space. Let $h: \mathcal{S}\to \mathbb{R}^d$ be a vector-valued function. Suppose that there exist some constant $C_1\geq 0$, $\rho\in (0, 1)$ and a norm $\|\cdot\|_p$ such that for any $y_1, y_2\in \mathcal{S}$:
    \begin{align*}
        \|\E_{y_1}[h(Y_k)]-\E_{y_2}[h(Y_k)]\|_p\leq C_1\rho^k(f_1(y_1)+f_2(y_2))
    \end{align*}
    where $f_i(\cdot)$ are non-negative functions that satisfy $\E_{y}[f_i(Y_1)]<\infty$ for all starting state $y\in \mathcal{S}$. Then, there exists a constant $\bar{h}$ such that for all $y\in\mathcal{S}$ and $k\geq 0$, we have
    \begin{align*}
        \|\E_{y}[h(Y_k)]-\bar{h}\|_p\leq \frac{C_1\rho^k}{1-\rho}(f_1(y)+\E_{y}[f_2(Y_1)]).
    \end{align*}
    Moreover, if for all $Y_0=y$, $\E_y[f_1(Y_1)+f_2(Y_2)]<\infty$, then $\bar{V}(y)=\sum_{k=0}^\infty (\E_{y}[h(Y_k)]-\bar{h})$ is a solution of the Poisson's equation $V(y)=h(y)+\E_y[V(Y_1)]-\bar{h}$. 
\end{lemma}

\begin{lemma}\label{lem:phi_bound}
    Let $\phi_0$ be the initial Markov regressor and define $y_0=(\phi_0, z_0)$, where $z_0$ is the initial signal. Then,
    \begin{enumerate}[(a)]
        \item $\phi_k=M^k\phi_0+\sum_{j=0}^{k-1}M^{k-j}w_j.$
        \item $\E_{y_0}[\|\phi_k\|^4_2]\leq D^4(\|\phi_0\|_2^4+\sigma_w^{(4)})/(1-\rho)^4.$
        \item Let $\{(\phi_k', z_k')\}_{k\geq 0}$ be an independent process. Then for any $x\in\mathbb{R}^d$, we have
        \begin{align*}
            \|\E_{y_0}[F(x, Y_k)]-\E_{y_0'}[F(x, Y_k')]\|_2
            &\leq \frac{\rho^kD^2}{1-\rho}\left(p_1(\phi_0)+p_1(\phi_0')\right)
        \end{align*}
        where $p_1(\phi)=2L(\|x\|_2+\|x^*\|_2)\|\phi\|_2^2+\left(|g(0)|+2L(\|x\|_2+\|x^*\|_2)\sqrt[4]{\sigma_w^{(4)}}\right)\|\phi\|_2$. Furthermore, for any $x_1, x_2\in \mathbb{R}^d$, we have
        \begin{align*}
            \|\E_{y_0}[F(x_1, Y_k)]&-\E_{y_0}[F(x_2, Y_k)]-(\E_{y_0'}[F(x_2, Y_k)]-\E_{y_0'}[F(x_2, Y_k'))\|_2\\
            &~~\leq \frac{LD^2\rho^k}{1-\rho}\|x_1-x_2\|_2\left(p_2(\phi_0)+p_2(\phi_0')\right)
        \end{align*}
        where $p_2(\phi)=2\|\phi\|_2^2+2\sqrt[4]{\sigma_w^{(4)}}\|\phi\|_2$.
    \end{enumerate}
\end{lemma}

\subsection{Proof of Proposition \ref{prop:LMS}}\label{sec:prop_LMS}
\begin{proof}
    \begin{enumerate}[(a)]
        \item 
        \begin{enumerate}[(1)]
            \item Recall $Y_k=(\phi_k, z_k)$. Then, we have
            \begin{align*}
                \|F(x_k, Y_k)\|_2&=\|\phi_k(z_k-g(\phi_k^Tx_k))\|_2\\
                &\leq \|\phi_k(g(\phi_k^Tx)-g(\phi_k^Tx^*)))\|_2+\|\phi_k(g(\phi_k^Tx^*)-z_k)\|_2\\
                &=A_1(Y_k)\|x_k-x^*\|_2+B_1(Y_k)
            \end{align*}
            where $A_1(Y_k)=L\|\phi_k\|^2$ and $B_1(Y_k)=\|\phi_k(g(\phi_k^Tx^*)-z_k)\|_2$.

            To show the moment bound on $A_1(Y_k)$ and $B_1(Y_k)$, we use Lemma \ref{lem:phi_bound} to obtain
            \begin{align*}
                \E_{y_0}[A_1^2(Y_k)]\leq \frac{L^2D^4(\|\phi_0\|_2^4+\sigma_w^{(4)})}{(1-\rho)^4}.
            \end{align*}
            Using Eq. \ref{eq:GLM}, we get
            \begin{align*}
                \E_{y_0}[B_1^2(Y_k)]&= \E_{y_0}[\|\phi_kv_k\|_2^2]\\
                &=\E_{y_0}[\|\phi_k\|_2^2]\sigma_v^{(2)}\\
                &\leq \sqrt{\E_{y_0}[\|\phi_k\|_4^2]}\sigma_v^{(2)}\\
                &\leq \frac{D^2\sigma_v^{(2)}\sqrt{(\|\phi_0\|_2^4+\sigma_w^{(4)})}}{(1-\rho)^2}.
            \end{align*}
 
            \item Using part (1), $\|F(x, Y)\|\leq A_1(Y)\|x-x^*\|_2+B_1(Y)$ where $A_1$ and $B_1$ are integral random variables as a consequence of Lemma \ref{lem:phi_bound}. Thus, $\|F(x, Y)\|_2$ is an integrable function of the random variable $Y$ and we have that $\bar{F}(x)=\E_{\tilde{Y}\sim \mu_Y}\left[\tilde{\phi}(g(\tilde{\phi}^Tx^*)-g(\tilde{\phi}^Tx))\right]$ exists and is finite.
            
            For the second part, we have
            \begin{align*}
                \langle x-x^*, \bar{F}(x)\rangle&=\left\langle x-x^*, \E_{\tilde{Y}\sim \mu_Y}\left[\tilde{\phi}(g(\tilde{\phi}^Tx^*)-g(\tilde{\phi}^Tx_1))\right]\right\rangle\\
                &=\E_{\tilde{Y}\sim \mu_Y}\left[(x-x^*)^T\tilde{\phi}(g(\tilde{\phi}^Tx^*)-g(\tilde{\phi}^Tx_1))\right]\\
                &\leq -\mu\E_{\tilde{Y}\sim \mu_Y}\left[\left(\phi^T(x-x^*)\right)^2\right]\tag{$\mu$-strongly monontone}\\
                &\leq -\mu\lambda_{min}^{\phi}\|x-x^*\|_2^2\tag{$\E_{\tilde{Y}\sim \mu_Y}\left[\phi\phi^T\right]=\Sigma_{\phi}^*$}.
            \end{align*}
             Note that the above also implies that $x^*$ is the unique solution to $\bar{F}(x)=0$. 
        \end{enumerate}
        \item From part (c) of Lemma \ref{lem:phi_bound}, the expected value of $F(x, Y_k)$ for two different initial states converges geometrically fast. Therefore, using Lemma \ref{lem:benveniste}, for any initial state $y_0$ and $x\in \mathbb{R}^d$ we obtain
        \begin{align*}
            \|\E_{y_0}[F(x, Y_k)]-\bar{F}(x)\|_2&\leq \frac{\rho^kD^2}{(1-\rho)^2}\left(p_1(\phi_0)+\E_{y_0}[p_1(\phi_1)]\right)
        \end{align*}
        where $\E_{y_0}[p_1(\phi_1)]<\infty$ due to Lemma \ref{lem:phi_bound}. Moreover, define $V_x(y)=\sum_{k=0}^{\infty}\left(\E_{y}[F(x, Y_k)]-\bar{F}(x)\right)$. Then $V_x(y)$ satisfies the Poisson equation \eqref{eq:Poisson_eq} for the Markov chain $\mathcal{M}_Y$.

        Again from part (c) of Lemma \ref{lem:phi_bound}, for any $x_1$ and $x_2$ in $\mathbb{R}^d$, the expected value of $F(x_1, \cdot)-F(x_2, \cdot)$ converges geometrically for two different starting states. Therefore, using Lemma \ref{lem:benveniste}, for any initial state $y$ we obtain
        \begin{align*}
            \|\E_{y}[F(x_1, Y_k)]-\E_{y}[F(x_2, Y_k)]-(\bar{F}(x_1)-\bar{F}(x_2))\|_2\leq \frac{LD^2\rho^k}{(1-\rho)^2}\|x_1-x_2\|_2\left(p_2(\phi)+\E_{y}[p_2(\phi_1)]\right).
        \end{align*}
        Recall that $\phi_1=M\phi+w_0$ and thus for any $l\geq 1$, $\|\phi_1\|_2^l\leq 2^{l-1}(\|M\|_2^l\|\phi\|_2^l+\|w_0\|_2^l)$. Using $\|M\|_2^l\leq D^l\rho^l$, we get
        \begin{align*}
            \E_{y}[p_2(\phi_1)]&\leq 4D^2\rho^2\|\phi\|_2^2+2\sqrt[4]{\sigma_w^{(4)}}D\rho\|\phi\|_2+6\sqrt{\sigma_w^{(4)}}<\infty.
        \end{align*}
        Combining the above relation gives us
        \begin{align*}
            \|V_{x_1}(y)-V_{x_2}(y)\|_2&= \left\|\sum_{k=0}^{\infty}\left(\E_{y}[F(x_1, Y_k)]-\bar{F}(x_1)\right)-\sum_{k=0}^{\infty}\left(\E_{y}[F(x_2, Y_k)]-\bar{F}(x_2)\right)\right\|_2\\
            &\leq \sum_{k=0}^\infty \frac{LD^2\rho^k}{(1-\rho)^2}\|x_1-x_2\|_2\left(p_2(\phi)+\E_{y}[p_2(\phi_1)]\right)\\
            &=A_2(y)\|x_1-x_2\|_2
        \end{align*}
        where $A_2(y)=LD^2(p_2(\phi)+\E_{y}[p_2(\phi_1)])/(1-\rho)^3$. This leads us to the following second moment bound on $A_2(y)$
        \begin{align*}
            \E_{y_0}[A_2^2(Y_k)]&\leq \frac{2L^2D^4}{(1-\rho)^6}\left(\E_{y_0}[p_2^2(\phi_k)]+\E_{y_0}[p_2^2(\phi_{k+1})]\right)\\
            &\leq \frac{2L^2D^4}{(1-\rho)^6}\left(8\E_{y_0}\left[\|\phi_k\|_2^4+\sqrt{\sigma_w^{(4)}}\|\phi_k\|_2^2\right]+8\E_{y_0}\left[\|\phi_{k+1}\|_2^4+\sqrt{\sigma_w^{(4)}}\|\phi_{k+1}\|_2^2\right]\right)\\
            &\leq \frac{32L^2D^4}{(1-\rho)^6}\left(\frac{D^4(\|\phi_0\|_2^4+\sigma_w^{(4)})}{(1-\rho)^4}+\frac{D^2\sqrt{\sigma_w^{(4)}}(\|\phi_0\|_2^2+\sqrt{\sigma_w^{(4)}})}{(1-\rho)^2} \right)\tag{Lemma \ref{lem:phi_bound}}\\
            &\leq \frac{32L^2D^8}{(1-\rho)^{10}}\left(\|\phi_0\|_2^4+\sqrt{\sigma_w^{(4)}}(\|\phi_0\|_2^2+ \sigma_w^{(4)} \right).
        \end{align*}
        Finally, $V_{x^*}(y)=\sum_{k=0}^{\infty}\left(\E_{y}[F(x^*, Y_k)]-\bar{F}(x^*)\right)=\sum_{k=0}^{\infty}\E_y[\phi_k\nu_k]=\phi(g(\phi^Tx^*)-z)$, where the rest of the terms are zero due to independence. Thus, $B_2(y)=B_1(y)=\|\phi(g(\phi^Tx^*)-z)\|_2$ whose second moment is bounded by
        \begin{align*}
            \E_{y_0}[B_2^2(Y_k)]&\leq \frac{D^2\sigma_v^{(2)}\sqrt{(\|\phi_0\|_2^4+\sigma_w^{(4)})}}{(1-\rho)^2}.
        \end{align*}
    \end{enumerate}
\end{proof}

\subsection{Proof of Theorem \ref{thm:LMS}}\label{sec:thm_LMS}
Since there is no martingale noise in the algorithm $A_3=B_3=0$. From Proposition \ref{prop:LMS}, we have $\hat{A}_1^2(y_0)=L^2D^4(\|\phi_0\|_2^4+\sigma_w^{(4)})/(1-\rho)^4$ and $\hat{A}_2^2(y_0)=32L^2D^8\left(\|\phi_0\|_2^4+\sqrt{\sigma_w^{(4)}}(\|\phi_0\|_2^2+ \sigma_w^{(4)} \right)/(1-\rho)^{10}$ which gives 
    \begin{align*}
        \hat{A}(y_0)=\hat{A}_1^2(y_0)+\hat{A}_2^2(y_0)+A_3^2\leq \frac{33L^2D^8}{(1-\rho)^{10}}\left(\|\phi_0\|_2^4+\sqrt{\sigma_w^{(4)}}(\|\phi_0\|_2^2+ \sigma_w^{(4)} \right).
    \end{align*}
    Next, $\hat{B}_2^2(y_0)=\hat{B}_1^2(y_0)=D^2\sigma_v^{(2)}\sqrt{(\|\phi_0\|_2^4+\sigma_w^{(4)})}/(1-\rho)^2$. Thus,
    \begin{align*}
        \hat{B}(y_0)^2&=\hat{B}_1^2(y_0)+\hat{B}_2^2(y_0)+B_3^2\\
        &=\frac{2D^2\sigma_v^{(2)}\sqrt{(\|\phi_0\|_2^4+\sigma_w^{(4)})}}{(1-\rho)^2}.
    \end{align*}
    Let $\max_{x\in \mathcal{X}}\|x\|_2\leq M/2$. Then, $\hat{C}_L(z_0, \phi_0)=\hat{C}(y_0)=\hat{A}(y_0)M^2+\hat{B}(y_0)$. Since $\|\cdot\|_c=\|\cdot\|_s=\|\cdot\|_2$, we have
    \begin{align*}
        \varphi_{1}=\frac{uL_su_{2s}u^2_{cs}}{l_{2s}}=1;~\varphi_{L, 0}=\|x_0-x^*\|_2^2+2\hat{C}_L(z_0, \phi_0).
    \end{align*}

\subsection{Proof of Lemma \ref{lem:phi_bound}}
\begin{proof}
\begin{enumerate}[(a)]
    \item It is straightforward to see that opening the recursive Markov process $\{\phi_k\}_{k\geq 0}$ from time-instant $0$ to $k$, we get
    \begin{align*}
        \phi_k&=M^k\phi_0+\sum_{j=0}^{k-1}M^{k-j}w_j.
    \end{align*}
    \item Taking norm both sides of the expression for $\phi_k$ and using triangle inequality leads to
    \begin{align*}
        \|\phi_k\|_2&\leq \|M^k\|_2\|\phi_0\|_2+\sum_{j=0}^{k-1}\|M^{k-1-j}\|_2\|w_j\|_2\\
        &\leq D\rho^k\|\phi_0\|_2+\sum_{j=0}^{k-1}D\rho^{k-1-j}\|w_j\|_2\\
        &\leq \frac{D(1-\rho^{k+1})}{1-\rho}\left(\frac{(1-\rho)\rho^k}{(1-\rho^{k+1})}\|\phi_0\|_2+\sum_{j=0}^{k-1}\frac{(1-\rho)\rho^{k-1-j}}{(1-\rho^{k+1})}\|w_j\|_2\right).
    \end{align*}
    Note that $(1-\rho)\rho^{k-j}/(1-\rho^{k+1})$ for $0\leq j\leq k$ forms a probability distribution, thus taking fourth power both sides and applying Jensen's inequality, we get
    \begin{align*}
        \|\phi_k\|_2^4&\leq \frac{D^4(1-\rho^{k+1})^4}{(1-\rho)^4}\left(\frac{(1-\rho)\rho^k}{(1-\rho^{k+1})}\|\phi_0\|_2^4+\sum_{j=0}^{k-1}\frac{(1-\rho)\rho^{k-1-j}}{(1-\rho^{k+1})}\|w_j\|_2^4\right)\\
        &\leq \frac{D^4}{(1-\rho)^3}\left(\rho^k\|\phi_0\|_2^4+\sum_{j=0}^{k-1}\rho^{k-1-j}\|w_j\|_2^4\right).
    \end{align*}
    Taking expectation on both sides conditioned on the initial state $y_0=(\phi_0, z_0)$, we obtain
    \begin{align*}
        \E_{y_0}[\|\phi_k\|_2^4]
        &\leq \frac{D^4}{(1-\rho)^3}\left(\rho^k\|\phi_0\|_2^4+\sum_{j=0}^{k-1}\rho^{k-1-j}\E[\|w_j\|_2^4]\right)\\
        &=\frac{D^4}{(1-\rho)^3}\left(\rho^k\|\phi_0\|_2^4+\sigma_4\sum_{j=0}^{k-1}\rho^{k-1-j}\right)\tag{Assumption \ref{assump:auto-reg}}\\
        &\leq \frac{D^4(\|\phi_0\|_2^4+\sigma_4)}{(1-\rho)^4}.
    \end{align*}
    \item From the expression for $F(x, Y)$, we have
        \begin{align*}
            \|\E_{y_0}[F(x, Y_k)]-\E_{y_0'}[F(x, Y_k')]\|_2&=\|\E_{y_0}[\phi_k(g(\phi_k^Tx)-z_k)]-\E_{y_0'}[\phi_k'(g({\phi_k'}^Tx)-z_k')]\|_2
        \end{align*}
        Using part (a), we know that for any two independent auto-regressive processes, we have
        \begin{align}\label{eq:diff_phi}
            \E_{y_0}[\phi_kg(\phi_k^Tx)]-\E_{y_0'}[\phi_k'g({\phi_k'}^Tx)]&=\E\Bigg[\left(M^k\phi_0+\sum_{j=0}^{k-1}M^{k-j}w_j\right)g\left(\left(M^k\phi_0+\sum_{j=0}^{k-1}M^{k-j}w_j\right)^Tx\right)\Bigg]\nonumber\\
            &~~-\E\Bigg[\left(M^k\phi_0'+\sum_{j=0}^{k-1}M^{k-j}w_j'\right)g\left(\left(M^k\phi_0'+\sum_{j=0}^{k-1}M^{k-j}w_j'\right)^Tx\right)\Bigg]
        \end{align}
        where the expectation is now over $w_j$ and $w_j'$. Since both of them are independent, we couple the terms inside the expectation such that $w_j=w_j'$ and use the bound \eqref{eq:lipschitz_phig} which leads to
        \begin{align*}
            \|\E_{y_0}[\phi_kg(\phi_k^Tx)]-\E_{y_0'}[\phi_k'g({\phi_k'}^Tx)]\|_2&\leq 
            \|M^k\|_2\|\phi_0-\phi_0'\|_2(|g(0)|+L\|x\|_2(\E_{y_0}[\|\phi_k\|_2]+\E_{y_0'}[\|\phi_k'\|_2])).
        \end{align*}
        Using $\|M^k\|_2\leq D\rho^k$ and part (b) with Jensen's inequality, we get
        \begin{align*}
            \|\E_{y_0}[\phi_kg(\phi_k^Tx)]-\E_{y_0'}[\phi_k'g({\phi_k'}^Tx)]\|_2&\leq\frac{D^2\rho^k\|\phi_0-\phi_0'\|_2}{1-\rho}\left(|g(0)|+L\|x\|_2\left(\sqrt[4]{(\|\phi_0\|_2^4+\sigma_w^{(4)})}+\sqrt[4]{(\|\phi_0'\|_2^4+\sigma_w^{(4)})}\right)\right)\\
            &\leq \frac{D^2\rho^k\|\phi_0-\phi_0'\|_2}{1-\rho}\left(|g(0)|+L\|x\|_2\left(\|\phi_0\|_2+\|\phi_0'\|_2+2\sqrt[4]{\sigma_w^{(4)}}\right)\right).
        \end{align*}
        Recall that $z_k=g(\phi_k^Tx^*)+\nu_k$ where $\nu_k$ is zero-mean independent noise. Thus,
        \begin{align*}
            \E_{y_0}[\phi_kz_k]-\E_{y_0'}[\phi_k'z_k']&=\E_{y_0}[\phi_k(g(\phi_k^Tx^*)+\nu_k)]-\E_{y_0'}[\phi_k'(g({\phi_k'}^Tx^*)+\nu_k')]\\
            &=\E_{y_0}[\phi_kg(\phi_k^Tx^*)]-\E_{y_0'}[\phi_k'g({\phi_k'}^Tx^*)].
        \end{align*}
        Thus, we get
        \begin{align*}
            \|\E_{y_0}[\phi_kz_k]-\E_{y_0'}[\phi_k'z_k']\|_2\leq \frac{D^2\rho^k\|\phi_0-\phi_0'\|_2}{1-\rho}\left(|g(0)|+L\|x^*\|_2\left(\|\phi_0\|_2+\|\phi_0'\|_2+2\sqrt[4]{\sigma_w^{(4)}}\right)\right).
        \end{align*}
        Combining the above bounds, we get 
        \begin{align*}
            \|\E_{y_0}[F(x, Y_k)]-\E_{y_0'}[&F(x, Y_k')]\|_2\leq \|\E_{y_0}[\phi_kg(\phi_k^Tx)]-\E_{y_0'}[\phi_k'(g({\phi_k'}^Tx)]\|_2+
            \|\E_{y_0}[\phi_kz_k]-\E_{y_0'}[\phi_k'z_k']\|_2\\
            &\leq \frac{\rho^kD^2(\|\phi_0\|_2+\|\phi_0'\|_2)}{1-\rho}\left(|g(0)|+L(\|x\|_2+\|x^*\|_2)\left(\|\phi_0\|_2+\|\phi_0'\|_2+2\sqrt[4]{\sigma_w^{(4)}}\right)\right)\\
            &\leq \frac{\rho^kD^2}{1-\rho}\left(p_1(\phi_0)+p_1(\phi_0')\right).
        \end{align*}
        For the second part of the claim, first note that
        \begin{align*}
            \E_{y_0}[F(x_1, Y_k)]-\E_{y_0}[F(x_2, Y_k)]&=\E_{y_0}[\phi_k(z_k-g(\phi_k^Tx_1)]-\E_{y_0}[\phi_k(z_k-g(\phi_k^Tx_2)]\\
            &=\E_{y_0}[\phi_kg(\phi_k^Tx_2)]-\E_{y_0}[\phi_kg(\phi_k^Tx_1)].
        \end{align*}
        Therefore,
        \begin{align*}
            \|\E_{y_0}[F(x_1, Y_k)]&-\E_{y_0}[F(x_2, Y_k)]-(\E_{y_0'}[F(x_1, Y_k)]-\E_{y_0'}[F(x_2, Y_k'))]\|_2\\
            &~~=\|\E_{y_0}[\phi_kg(\phi_k^Tx_2)]-\E_{y_0}[\phi_kg(\phi_k^Tx_1)]-(\E_{y_0'}[\phi_k'g({\phi_k'}^Tx_2)]-\E_{y_0'}[\phi_k'g({\phi_k'}^Tx_1)])\|_2
        \end{align*}
        Using Eq. \eqref{eq:diff_phi}, we note that for $i=1, 2$
        \begin{align*}
            \E_{y_0}[\phi_kg(\phi_k^Tx_i)]-\E_{y_0'}[\phi_k'g({\phi_k'}^Tx_i)]&=\E\Bigg[\left(M^k\phi_0+\sum_{j=0}^{k-1}M^{k-j}w_j\right)g\left(\left(M^k\phi_0+\sum_{j=0}^{k-1}M^{k-j}w_j\right)^Tx_i\right)\Bigg]\nonumber\\
            &~~-\E\Bigg[\left(M^k\phi_0'+\sum_{j=0}^{k-1}M^{k-j}w_j'\right)g\left(\left(M^k\phi_0'+\sum_{j=0}^{k-1}M^{k-j}w_j'\right)^Tx_i\right)\Bigg]
        \end{align*}
        Again, since both $w_j$ and $w_j'$ are independent, we couple the terms inside the expectation such that $w_j=w_j'$ and use the bound \eqref{eq:lipschitz_phig_strong} which leads to
        \begin{align*}
            \|\E_{y_0}[F(x_1, Y_k)]&-\E_{y_0}[F(x_2, Y_k)]-(\E_{y_0'}[F(x_1, Y_k)]-\E_{y_0'}[F(x_2, Y_k'))]\|_2\\
            &~~\leq L\|M^k\|_2\|\phi_0-\phi_0'\|_2\|x_1-x_2\|_2\left(\E_{y_0}[\|\phi_k\|_2]+\E_{y_0'}[\|\phi_k'\|_2]\right).
        \end{align*}
        Using $\|M^k\|_2\leq D\rho^k$ and part (b) with Jensen's inequality, we get
        \begin{align*}
            \|\E_{y_0}[F(x_1, Y_k)]&-\E_{y_0}[F(x_2, Y_k)]-(\E_{y_0'}[F(x_1, Y_k)]-\E_{y_0'}[F(x_2, Y_k'))]\|_2\\
            &~~\leq  \frac{LD^2\rho^k\|\phi_0-\phi_0'\|_2}{1-\rho}\|x_1-x_2\|_2\left(\sqrt[4]{\|\phi_0\|_2^4+\sigma_w^{(4)}}+\sqrt[4]{\|\phi_0'\|_2^4+\sigma_w^{(4)}}\right)\\
            &~~\leq \frac{D^2\rho^k(\|\phi_0\|_2+\|\phi_0'\|_2)}{1-\rho}\|x_1-x_2\|_2\left(\|\phi_0\|_2+\|\phi_0'\|_2+2\sqrt[4]{\sigma_w^{(4)}}\right)\\
            &~~= \frac{LD^2\rho^k}{1-\rho}\|x_1-x_2\|_2\Bigg((\|\phi_0\|_2+\|\phi_0'\|_2)^2+2\sqrt[4]{\sigma_w^{(4)}}(\|\phi_0\|_2+\|\phi_0'\|_2)\Bigg)
        \end{align*}
        Using $(a+b)^2\leq 2a^2+2b^2$, we finally get
        \begin{align*}
            \|\E_{y_0}[F(x_1, Y_k)]&-\E_{y_0}[F(x_2, Y_k)]-(\E_{y_0'}[F(x_1, Y_k)]-\E_{y_0'}[F(x_2, Y_k'))]\|_2\\
            &~~\leq \frac{LD^2\rho^k}{1-\rho}\|x_1-x_2\|_2\left(p_2(\phi_0)+p_2(\phi_0')\right)
        \end{align*}
\end{enumerate}
\end{proof}

\section{Proof of technical results in Section \ref{sec:Q-learning}}
\subsection{Proof of Proposition \ref{prop:Q-learning}}\label{appendix:Q-learning_prop}
\begin{proof}
    \begin{enumerate}[(a)]
        \item \begin{enumerate}[(1)]
            \item Recall that $Q^*(s,a)=\mathcal{R}(s,a)+\gamma\sum_{s'\in \mathcal{S}}P(s'|s,a)\max_{a'\in \mathcal{A}}Q^*(s', a')$. Thus, we have
            \begin{align*}
                \|F(Q, y)\|_{\infty}&\leq \left\|\gamma\sum_{s'\in \mathcal{S}}P(s'|s,a)\left(\max_{a'\in \mathcal{A}}Q(s', a')-\max_{a'\in \mathcal{A}}Q^*(s', a')\right)-Q(s,a)+Q^*(s,a)\right\|_{\infty}+\|Q^*(s,a)\|_{\infty}\\
                &\leq 2\|Q-Q^*\|_{\infty}+\|Q^*\|_{\infty}.\\
            \end{align*}
            Similarly, for any $Q_1$ and $Q_2$ and $y\in \mathcal{Y}$ we have
            \begin{align*}
                 \|F(Q_1, y)-F(Q_2, y)\|_{\infty}&\leq \left\|\gamma\sum_{s'\in \mathcal{S}}P(s'|s,a)\left(\max_{a'\in \mathcal{A}}Q_1(s', a')-\max_{a'\in \mathcal{A}}Q_2(s', a')\right)-Q_1(s,a)+Q_2(s,a)\right\|_{\infty}\\
                 &\leq 2\|Q_1-Q_2\|_{\infty}.
            \end{align*}
            \item Using the Markov property, we have for any $Q\in \mathbb{R}^{|\mathcal{S}||\mathcal{A}|}$ and $(s,a)$:
            \begin{align*}
                \E_{S_k\sim \mu_b}[F(Q, &(S_k, A_k))(s,a)]\\
                &=\E_{S_k\sim \mu_b}\left[\mathbbm{1}\{S_k=s, A_k=a\}\left(\mathcal{R}(s,a)+\gamma\sum_{s'\in \mathcal{S}}P(s'|s,a)\max_{a'\in \mathcal{A}}Q(s', a')-Q(s,a)\right)\right]\\
                &=\mu_b(s)\pi_b(a|s)(\mathcal{B}(Q)(s,a)-Q(s,a)).
            \end{align*}
            Thus, $\bar{F}(Q)=\Lambda(\mathcal{B}(Q)-Q)$.
            \item Since $Q^*$ is the solution to the Bellman equation, we have
            \begin{align*}
                \bar{F}(Q^*)&=\Lambda(\mathcal{B}(Q^*)-Q^*)\\
                &=0.
            \end{align*}
            Hence, $Q^*$ is a solution to the equation $\bar{F}(Q^*)=0$. The uniqueness of the solution is immediate from the fact that $\mu_b(s)\pi_b(a|s)>0$ and $Q^*$ is the unique solution to $\mathcal{B}(Q)-Q=0$.
        \end{enumerate}
        \item Fix a state $y_0=(s_0, a_0)\in \mathcal{Y}$ and define $\tau=\min\{n>0:Y_n=y_0\}$ and $\E_{y}[\cdot]=\E[\cdot|Y_0=y]$, then for all $y\in \mathcal{Y}$
        \begin{align*}
            V_Q(y)=\E_y\left[\sum_{n=0}^{\tau-1}\left(F(Q, Y_n)-\bar{F}(Q)\right)\right]
        \end{align*}
        is a solution to the Poisson equation (Lemma 4.2 and Theorem 4.2 of Section VI.4, pp. 85-91,
       of \cite{borkar1991topics}). Thus, we have
        \begin{align*}
            \|V_{Q_1}(y)-V_{Q_2}(y)\|_{\infty}&= \left\|\E_y\left[\sum_{n=0}^{\tau-1}\left(F(Q_1, Y_n)-F(Q_2, Y_n)-(\bar{F}(Q_1)-\bar{F}(Q_2))\right)\right]\right\|_{\infty}\\
            &\leq \E_y\left[\sum_{n=0}^{\tau-1}\left(\|F(Q_1, Y_n)-F(Q_2, Y_n)\|_{\infty}+\|\bar{F}(Q_1)-\bar{F}(Q_2)\|_{\infty}\right)\right]\\
            &\leq \E_y\left[\sum_{n=0}^{\tau-1}\left(2\|Q_1-Q_2\|_{\infty}+\|\Lambda(\mathcal{B}(Q_1)-(\mathcal{B}(Q_2))\|_{\infty}+\|\Lambda(Q_1-Q_2)\|_{\infty}\right)\right]\tag{Using property 1 and 2 from part (a)}\\
            &\leq 4\|Q_1-Q_2\|_{\infty}\E_y[\tau]\\
            &\leq 4\tau_{y_0}\|Q_1-Q_2\|_{\infty}.
        \end{align*}
        Furthermore, since $Q^*$ solves the Bellman equation, for all $y\in \mathcal{Y}$ we have
        \begin{align*}
            V_{Q^*}(y)=0.
        \end{align*}
        \item Define $\mathcal{F}_k=\{Q_0, Y_0, \dots, Q_{k-1}, Y_{k-1}, Q_k, Y_k\}$. Then due to the Markov property, we have
        \begin{align*}
            \E[M_k(Q_k)|\mathcal{F}_k]=0.
        \end{align*}
        Furthermore, 
        \begin{align*}
            \|M_k(Q_k)\|_{\infty}&\leq \gamma\max_{s,a}\left(\left|\max_{a'\in \mathcal{A}}Q_k(S_{k+1}, a')-\sum_{s'\in \mathcal{S}}P(s'|S_k,A_k)\max_{a'\in \mathcal{A}}Q_k(s', a')\right|\right)\\
            &\leq 2\|Q_k\|_{\infty}\\
            &\leq 2\left(\|Q_k-Q^*\|_{\infty}+\|Q^*\|_{\infty}\right).
        \end{align*}
    \end{enumerate}
\end{proof}

\subsection{Proof of Theorem \ref{thm:Q-learning}}\label{appendix:Q-learning_thm}
\begin{proof}
    To identify the constants $A_Q$, $B_Q$ and $\eta_Q$, we use Lemma \ref{prop:Q-learning} to get
    \begin{align*}
        A_{Q}=(A_1+A_3+1)^2=25;~~B_{Q}=\left(B_1+B_3+\frac{B_2}{A_2}\right)^2=9\|Q^*\|_{\infty}^2;~~\eta_{Q}=(1-\gamma)\Lambda_{min}.
    \end{align*}
    Since $\eta_Q\leq 1$, we have $\frac{B_Q}{\eta_Q}\geq B_Q$. Furthermore, from Lemma \ref{lem:moreau_prop} we get $\varrho_{Q,1}$ as follows:
    \begin{align*}
        \varrho_{Q,1}&=uL_su_{cs}^2A_2=\frac{2(1+\omega)}{\omega}(p-1)\left(|\mathcal{S}||\mathcal{A}|\right)^{2/p}4\tau_{y_0}\\
        &\leq \frac{(1+\omega)}{\omega}16e\tau_{y_0}\log \left(|\mathcal{S}||\mathcal{A}|\right)\\
        &\leq \frac{32e\tau_{y_0}\log \left(|\mathcal{S}||\mathcal{A}|\right)}{(1-\gamma)\Lambda_{min}}.
    \end{align*}
    \begin{align*}
        \varrho_{Q,0}&=\frac{2(1+\omega)(1+50\varrho_{Q,1})}{(1+\omega/\sqrt{e})}\|Q_0-Q^*\|_c^2+36\|Q^*\|_{\infty}^2\varrho_{Q,1}\\
        &\leq 4(1+50\varrho_{Q,1})\|Q_0-Q^*\|_c^2+36\|Q^*\|_{\infty}^2\varrho_{Q,1}.
    \end{align*}
\end{proof}

\subsection{Sample complexity for Q-Learning }\label{cor:Q_sample}
To find an estimate $Q$, such that $\E[\|Q-Q^*\|_{\infty}]\leq \epsilon$, we need
\begin{align*}
    \frac{58B_Q\varrho_{Q,1}\alpha}{\eta_Q}&\leq \frac{\epsilon^2}{2}\\
    \implies \alpha&\leq \mathcal{O}\left(\frac{\eta_Q}{\|Q^*\|_\infty^2\varrho_{Q,1}}\right).
\end{align*}
Using this bound on $\alpha$, we have
\begin{align*}
    \varrho_{Q,0}\exp\left(\frac{-\eta_Q\alpha k}{2}\right)&\leq \frac{\epsilon^2}{2}\\
    \implies k&\leq \mathcal{O}\left(\frac{1}{\alpha\eta_Q}\log (1/\epsilon)\right).
\end{align*}
Since $\varrho_{Q,1}\leq 32e\tau_{y_0}\log \left(|\mathcal{S}||\mathcal{A}|\right)/((1-\gamma)\Lambda_{min})$ and $\|Q^*\|_\infty^2\leq \mathcal{O}\left(1/(1-\gamma)^2\right)$, we have
\begin{align*}
    k\leq \mathcal{O}\left(\frac{\log (1/\epsilon)}{\epsilon^2}\right)\mathcal{O}\left(\frac{1}{(1-\gamma)^5}\right)\mathcal{O}\left(\Lambda_{min}^{-3}\right).
\end{align*}

\section{Proof of technical results in Section \ref{sec:opt}}\label{appendix:C}
\subsection{Proof of Proposition \ref{prop:opt}}\label{appendix:opt}
\begin{proof}
    \begin{enumerate}[(a)]
        \item Using $L$-smoothness of the $f(x)$ and the fact that $\nabla f(x^*)=0$, we get
        \begin{align*}
           \|-U_{i}\nabla_{i} f(x)\|_2&=\|-U_{i}U_i^T\nabla f(x)-U_{i}U_{i}^T\nabla f(x^*)\|_2\\
           &\leq L\|U_{i}U_{i}^T\|_2\|x-x^*\|_2\\
           &=L\|x-x^*\|_2\tag{$\|U_{i}U_{i}^T\|_2=1$}.
        \end{align*}
        \item Note that $\E_{i}[G(Y)]=G(i\Mod p+1)$ for any $G:\mathcal{S}\to\mathbb{R}^d$. Thus, we can write Poisson equation as
        \begin{align*}
            V_x(i)=-U_{i}U_i^T\nabla f(x)+V_x(i\Mod p+1)+\frac{1}{p}\nabla f(x).
        \end{align*}
        Set $V_x(1)=0$, then we have
        \begin{align*}
            V_x(i)=-\sum_{k=0}^{p-i}\left(U_{k+1}U_{k+1}^T\nabla f(x)-\frac{1}{p}\nabla f(x)\right),~\forall i\in\mathcal{S}/\{1\}
        \end{align*}
        Note that  $U_kU_k^T$ is a also diagonal matrix. Hence, $\sum_{k=0}^{p-i}\left(U_{k+1}U_{k+1}^T-\frac{1}{p}I_d\right)$ is a diagonal matrix with entries $\frac{i-1}{p}$ in the first $i$ places and $-\frac{1}{p}$ in the remaining places. This implies that $\left\|\sum_{k=0}^{p-i}\left(U_{k+1}U_{k+1}^T-\frac{1}{p}I_d\right)\right\|_2\leq \frac{i-1}{p}\leq 1$ . Thus, for all $i\in\mathcal{S}$, we have
        \begin{align*}
            \|V_x(i)-V_y(i)\|_2 &= \left\|\sum_{k=0}^{p-i}U_{k+1}U_{k+1}^T\nabla f(x)-\frac{1}{p}\nabla f(x)-U_{k+1}U_{k+1}^T\nabla f(y)+\frac{1}{p}\nabla f(y)\right\|_2\\
            &\leq \left\|\sum_{k=0}^{p-i}\left(U_{k+1}U_{k+1}^T-\frac{1}{p}I_d\right)\right\|_2\|\nabla f(x)-\nabla f(y)\|_2\\
            &\leq \|\nabla f(x)-\nabla f(y)\|_2\\
            &\leq L\|x-y\|_2\tag{Smoothness of $f(x)$}
        \end{align*}
        Additionally, $\nabla f(x^*)=0$ implies $V_{x^*}(i)=0,~\forall i\in\mathcal{S}$.
        \item Using Assumption \ref{assump:noise_opt}, we have
        \begin{align*}
            \E[M_k|\mathcal{F}_k]&=\E[U_{i(k)}w_k|\mathcal{F}_k]\\
            &=U_{i(k)}\E[w_k|\mathcal{F}_k]\tag{$U_{i(k)}$ is deterministic}\\
            &=0.
        \end{align*}
        Furthermore,
        \begin{align*}
            \|M_k\|_2&\leq \|U_{i(k)}w_k\|_2\\
            &\leq \|w_k\|_2\tag{$\|U_i\|_2=1$}\\
            &\leq C_1\|x_k-x^*\|_2+C_2.
        \end{align*}
    \end{enumerate}
\end{proof}

\subsection{Proof of Theorem \ref{thm:opt}}\label{appendix:opt_thm}
From Proposition \ref{prop:opt}, we have
\begin{align*}
    A_G=(L+C_1+1)^2;&~~B_G=C_2^2;\\
    \varrho_{G,0}=2(1+2(L+C_1+1)^2\max\{L,1\})\|x_0-x^*\|_c^2&+4C_2^2\max\{L,1\};~~\varrho_{G,1}=\max\{L,1\}.
\end{align*}

\section{Auxiliary Lemmas}\label{appendix:D}

\begin{lemma}\label{lem:step-size_prop}
The step-size sequence in Assumption \ref{assump:step-size} has following properties:
    \begin{align}
        \alpha_k\leq \alpha_{k-1};~~\alpha_{k-1}\leq 2\alpha_k;~~\alpha_{k-1}-\alpha_{k}\leq \frac{2\xi}{\alpha}\alpha_k^2
    \end{align}
\end{lemma}
\begin{proof}
    The step-size is non-decreasing by construction. Now consider the ratio 
    \begin{align*}
        \frac{\alpha_{k-1}}{\alpha_k}&=\left(\frac{K+k}{K+k-1}\right)^{\xi}\\
        &=\left(1+\frac{1}{K+k-1}\right)^{\xi}.
    \end{align*}
    Since $k\geq 1$ we have $\frac{1}{K+k-1}\leq \frac{1}{K}\leq 1$. Putting together with the expression above, we get
    \begin{align*}
        \frac{\alpha_{k-1}}{\alpha_k}\leq 2^\xi\leq 2.\tag{$\xi\leq 1$}
    \end{align*}
    For the final part, consider the function $f(x)=\frac{1}{(k+x)^\xi}$ for $x\in [0,1]$ and $k\geq 1$. Using Taylor's series expansion, there exists $z\in [x, 1]$ such that we have
    \begin{align*}
        f(1)&=f(x)+(1-x)f'(z)\\
        &=f(x)-\frac{(1-x)\xi}{(k+z)^{1+\xi}}\\
        f(x)-f(1)&= \frac{(1-x)\xi}{(k+z)^{1+\xi}}\\
        &\leq \frac{(1-x)\xi}{k^{1+\xi}}\tag{$z\geq 0$}\\
        &\leq \frac{(1-x)\xi}{k^{2\xi}}.\tag{$\xi\leq 1$}
    \end{align*}
    Substituting $x=0$, we have $\forall~k\geq 1$
    \begin{align*}
        \frac{1}{k^\xi}-\frac{1}{(k+1)^\xi}&\leq \frac{\xi}{k^{2\xi}}\\
        \implies \alpha_{k-1}-\alpha_k&\leq \frac{\alpha\xi}{(k+K-1)^{2\xi}}\leq \frac{2\xi}{\alpha}\alpha_k^2.\tag{$\alpha_{k-1}\leq 2\alpha_k$}
    \end{align*}
\end{proof}

\begin{lemma}\label{lem:rank_lemma}
    Consider the following infinite set of equations:
    \begin{align}\label{eq:infinite_eq}
        \begin{split}
            a_{11}x_1+a_{12}x_2+...+a_{1d}x_d&=0\\
            a_{21}x_1+a_{22}x_2+...+a_{2d}x_d&=0\\
            \vdots~~~~~~~~~~~~~~~~~~&
        \end{split}
    \end{align}
    Assume that the set of equations has a unique solution, i.e., there exists a unique $x^*$ which satisfies $\sum_{j=1}^da_{ij}x_j^*=0$ for all $i\geq 1$. Define $A(i)$  as a row vector, $A(i)=[a_{i1}, a_{i2},..., a_{id}]$. Then, there exists a finite $N$ such that the span of $\{A(i)\}_{i\leq N}=\mathbb{R}^d$.
\end{lemma}
\begin{proof}
    Denote $A_i\in \mathbb{R}^{i\times d}$ as the concatenation of the row vectors $\{A(j)\}_{j\leq i}$ into a matrix. Let $r_i=rank(A_i)$. Then, $r_i$ is a non-decreasing sequence that is bounded above by $d$. This is because the span of a set increases by adding new vectors to the set. Thus, by Monotone convergence theorem, $\lim_{i\to \infty} r_i$ must exist. Let us denote the limit by $r$. Now, to show that $r=d$, we will use the method of contradiction. Assume that $r<d$. Since $r_i\in \mathbb{Z}$, there exists a finite number $N$ such that $r_i=r,~\forall i\geq N$. However, this implies that the null space of $A_i$ is nonempty for all $i$, further implying that the set of equations \eqref{eq:infinite_eq} has more than one solution. Hence, we have a contradiction and $r$ must be equal to $d$. Since the column rank and the row rank of a finite-dimensional matrix are equal, this implies $dim(\{A(i)\}_{i\leq N})=rank(A_N)=d$. The claim follows.
\end{proof}

\begin{lemma}\label{lem:block_norm}
    Let $P$ be a square matrix with dimension $d_1+d_2$ which partitioned as follows
    \begin{align*}
        P=\begin{bmatrix}
            A & B\\
            C & D
        \end{bmatrix}
    \end{align*}
    where $A\in \mathbb{R}^{d_1\times d_1}$, $B\in \mathbb{R}^{d_1\times d_2}$, $C\in \mathbb{R}^{d_2\times d_1}$, and $D\in \mathbb{R}^{d_2\times d_2}$. Then, $\|P\|_2\leq \|A\|_2+\|B\|_2+\|C\|_2+\|D\|_2$.
\end{lemma}
\begin{proof}
    Using the definition of matrix norm, we have
    \begin{align*}
        \|P\|_2=\max_{\|x\|_2=1}\|Px\|_2
    \end{align*}
    Let $x=\begin{bmatrix}
        y\\
        z
    \end{bmatrix}$, where $y\in \mathbb{R}^{d_1}$ and $z\in \mathbb{R}^{d_2}$. Then, $Px$ can be written as
    \begin{align*}
        Px=\begin{bmatrix}
        Ay+Bz\\
        Cy+Dz
    \end{bmatrix}=\begin{bmatrix}
        Ay\\
        0
    \end{bmatrix}+\begin{bmatrix}
        Bz\\
        0
    \end{bmatrix}+\begin{bmatrix}
        0\\
        Cy
    \end{bmatrix}+\begin{bmatrix}
        0\\
        Dz
    \end{bmatrix}
    \end{align*}
    Then, by triangle inequality, we have
    \begin{align*}
        \|Px\|_2\leq (\|A\|_2+\|C\|_2)\|y\|_2+(\|B\|_2+\|D\|_2)\|z\|_2
    \end{align*}
    Note that since $\max\{\|y\|_2, \|z\|_2\}\leq \|x\|_2\leq 1$, we have
    \begin{align*}
        \max_{\|x\|_2=1}\|Px\|_2\leq \|A\|_2+\|B\|_2+\|C\|_2+\|D\|_2.
    \end{align*}
\end{proof}
\end{document}